\documentclass[10pt]{article} 
\usepackage[utf8]{inputenc}
\usepackage[accepted]{tmlr}
\usepackage{amsmath, amssymb, amsthm}
\usepackage{mathtools}
\usepackage{graphicx}
\usepackage{hyperref}
\usepackage{url}
\usepackage{booktabs}
\usepackage{algorithm}
\usepackage{algpseudocode}
\usepackage{booktabs}
\usepackage{array}
\usepackage[table]{xcolor}
\usepackage{pifont}
\usepackage{multirow}
\usepackage{graphicx}
\usepackage[disable]{todonotes}
\usepackage{enumitem} 
\usepackage{subcaption} 
\usepackage{amsthm}
\newtheorem{theorem}{Theorem}
\newtheorem{lemma}{Lemma}
\newtheorem{corollary}{Corollary}

\usepackage[disable]{todonotes}

\title{Estimating Rare Events in Language Models \\with Proper Evaluation}

\author{\name Nikita Y. Parulekar \email nikitaparulekar2604@gmail.com \\
      \addr Johns Hopkins University
      \AND
      \name Anqi Liu \email aliu.cs@jhu.edu \\
      \addr Johns Hopkins University
}

\def\month{11}
\def\year{2026}
\def\openreview{\url{https://openreview.net/forum?id=4AMA60efWX2}}

\begin{document}

\maketitle

\begin{abstract}

Quantifying the risk of rare failures in language models, such as those triggered by adversarial distribution shifts or very large-scale deployments, requires estimating probabilities far too small for random sampling. While recent work has formalized low-probability estimation, existing pipelines remain fragile in the rarest regimes: standard evaluation losses can become unstable or poorly matched to asymmetric safety costs, and estimators can suffer zero-estimate collapse or systematic bias. In this work, we introduce the Shifted-Power Bregman (SPB) Loss, a proper scoring rule that remains finite for zero-estimates and offers tunable asymmetry between underestimation and overestimation penalties, making all estimators, including those whose estimates collapse to zero, comparable under an explicit cost profile. We also propose Gradient Activation Adaptive Multilevel Splitting (GA-AMLS), which adapts rare-event Monte Carlo methods to the continuous activation space of language models. Specifically, GA-AMLS uses a gradient-based MCMC kernel to navigate activation space, reducing zero-estimate failures of input-space search and replacing the independence assumptions of prior activation-space estimators with conditional sampling under an explicit, heavier-tailed activation prior. Experiments on small transformer models reveal distinct bias and cross-token dispersion profiles: GA-AMLS often outperforms the strongest baseline under symmetric evaluation, while methods with overestimation bias prevail under asymmetric penalties and have lower cross-token dispersion. Our findings highlight that estimator choice should be matched to deployment context. More broadly, our work establishes activation space as a tractable domain for rare-event estimation in language models, circumventing the brittleness of discrete input-space search.

\end{abstract}

\section{Introduction}

Modern large language models (LLMs) are trained to generalize to new inputs by minimizing expected loss during training. However, this averaging objective leaves them susceptible to producing highly undesirable outputs on a small subset of rare inputs. While these events may appear statistically negligible during training, the rarity of a failure is often relative. In deployment, distribution shifts, whether manufactured through adversarial jailbreaks or incidental due to goal misgeneralization \citep{shah2022goalmisgeneralizationcorrectspecifications}, can make these catastrophic failures significantly more common. Furthermore, in safety-critical applications such as autonomous vehicles or medical infrastructure, even a single failure can be unacceptable. Additionally, the large-scale use of language models can make low-probability events eventual certainties. 
Therefore, there is a growing need for accurate low-probability estimation (LPE) methods: the ability to quantify the likelihood of a model producing a specific harmful output, even when that probability is extremely small, such as between $10^{-5}$ and $10^{-9}$, ranges where naive sampling is computationally infeasible. Following the problem setting introduced by \cite{wu2025estimatingprobabilitiesrareoutputs}, we study LPE in the context of argmax sampling, where our goal is to estimate the probability that a specific target token will have the largest output logit within a computationally constrained budget as shown in Figure \ref{fig:lpe-overview}. 

Comparing LPE methods requires a metric aligned with the costs of rare-event estimation: underestimates can miss real failure modes, overestimates can create false alarms or misallocate auditing effort, and the probabilities of interest span many orders of magnitude. A useful metric should therefore allow asymmetric penalties for under- and overestimation, optionally emphasize the deepest tail, and remain well-defined when an estimator returns exactly zero. The metrics used in prior work, Itakura--Saito (IS) loss and squared error loss in log space, only partially satisfy these requirements: they involve logarithms or reciprocals of the estimate and are therefore undefined at zero, which \citet{wu2025estimatingprobabilitiesrareoutputs} patch with an affine calibration that itself requires known ground-truth probabilities. These evaluative issues are tightly coupled to shortcomings of existing estimators. Input-space Importance Sampling searches the discrete token space; for the rarest targets it often fails to find any triggering input, collapsing estimates to zero, and it additionally requires knowledge of exact input-sequence probabilities. The activation-space alternative, Quadratic Logit Decomposition (QLD), avoids zero estimates by recombining whitened activation components as if they were independent, producing \(n^2\) synthetic candidates from \(n\) empirical activations. But whitening guarantees only uncorrelatedness, not independence, and language-model activations carry heavy-tailed, outlier-dominated structure \citep{sun2024massive,dettmers2022llmint88bitmatrixmultiplication} that can survive whitening; empirically, the resulting distortion surfaces mainly as overestimation on rare events.

\begin{figure}[t]
    \centering
    \captionsetup[subfigure]{font=small,skip=2pt}

    \begin{subfigure}[!htbp]{0.49\textwidth}
        \centering
        \includegraphics[
            width=\linewidth,
            trim={10 10 10 10},
            clip
        ]{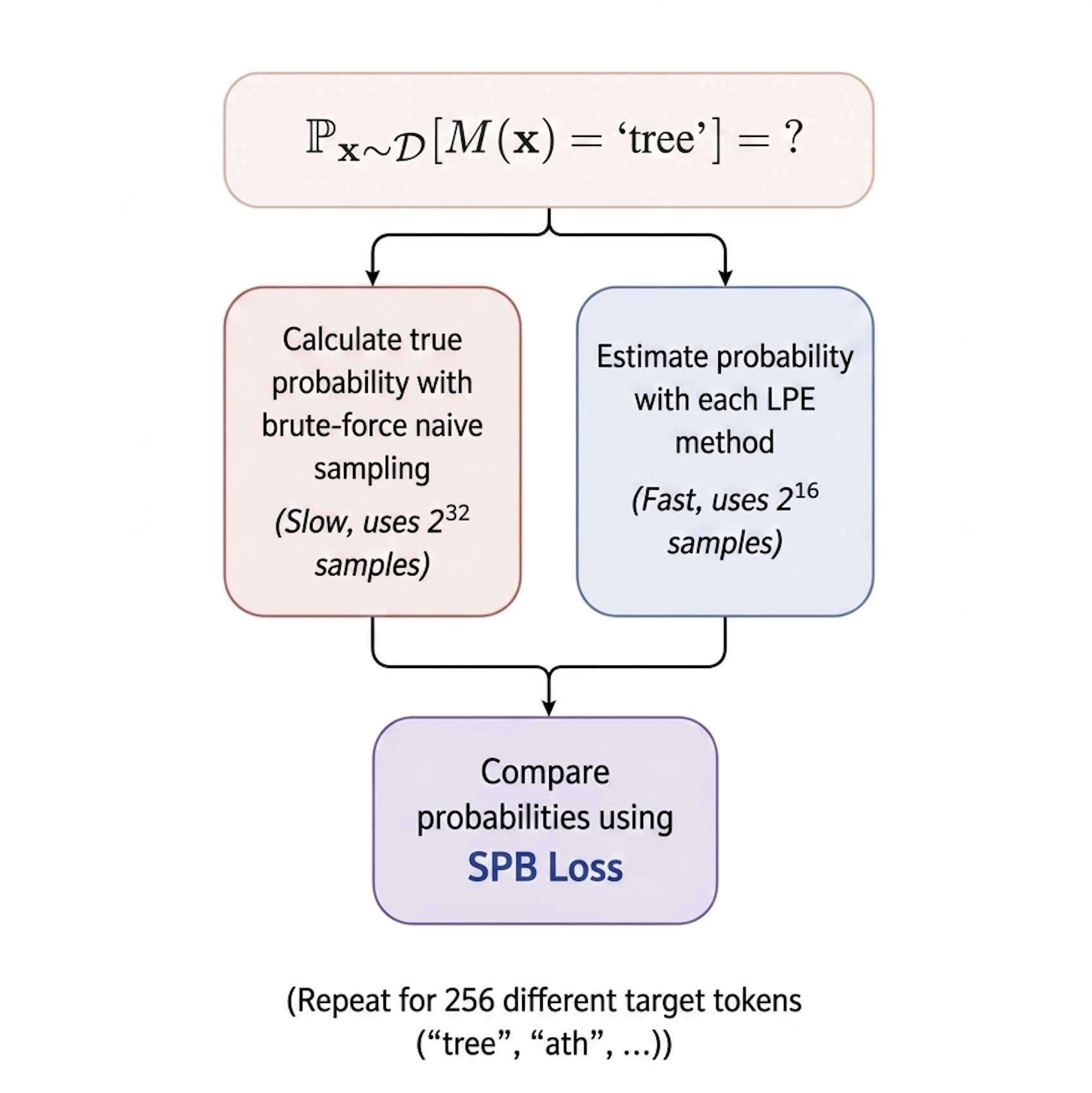}
        \caption{Low-probability estimation setup.}
        \label{fig:lpe-overview}
    \end{subfigure}
    \hfill
    \begin{subfigure}[!htbp]{0.45\textwidth}
        \centering
        \includegraphics[
            width=\linewidth,
        ]{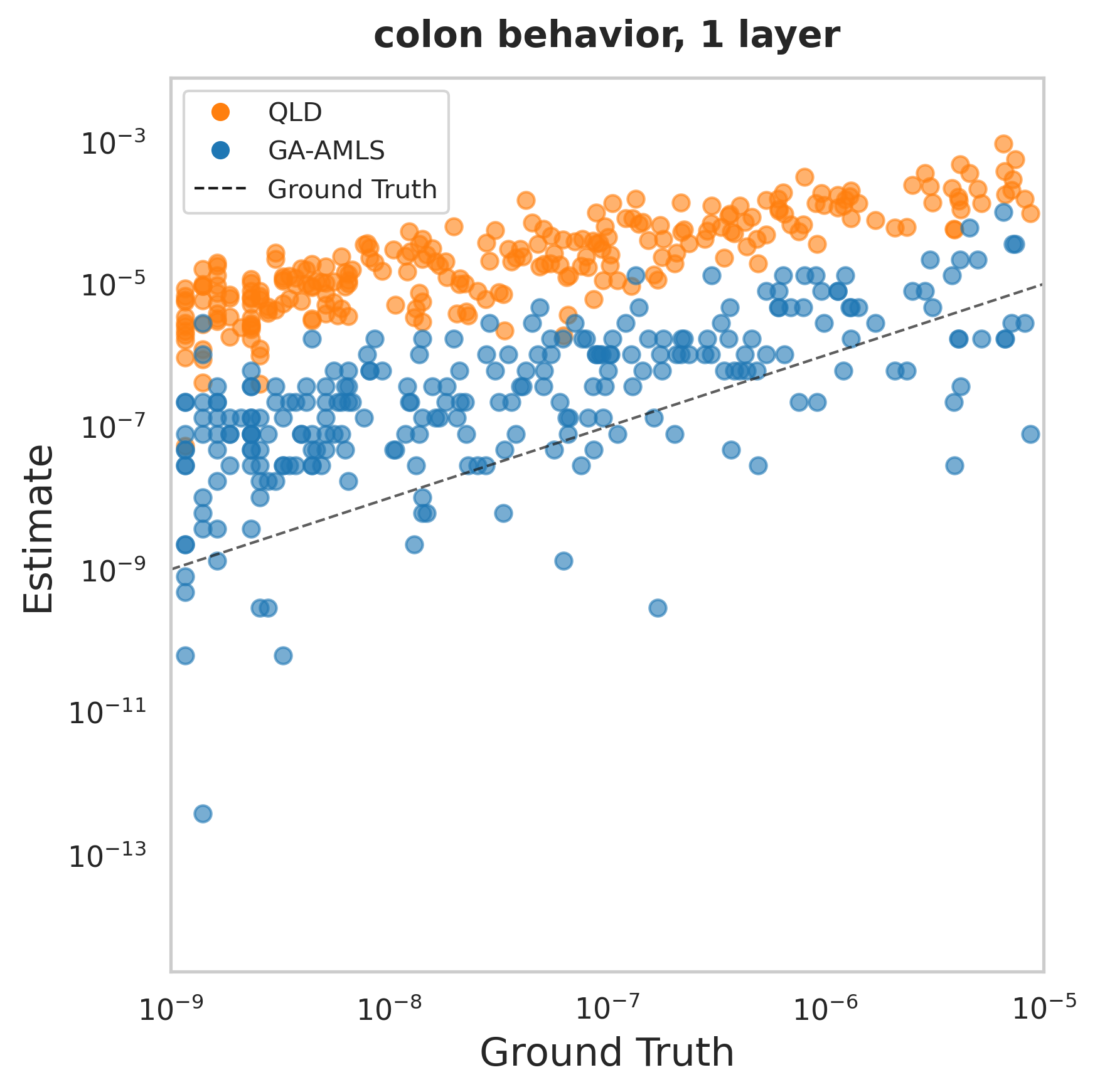}
        \caption{GA-AMLS and QLD estimates on the colon distribution and 1-layer model. Each point represents a different target token.}
        \label{fig:ga-amls-overview}
    \end{subfigure}

    \caption{Overview of the low-probability estimation setup and GA-AMLS estimates.}
    \label{fig:overview}
\end{figure}

In this paper, we address both the evaluative and algorithmic limitations of current LPE approaches through two coupled contributions: a numerically stable evaluation metric that makes comparing all estimators, including those whose estimates collapse to zero, well-posed, and an activation-space estimation algorithm.

First, we introduce \emph{Shifted-Power Bregman} (SPB) Loss, a family of proper
scoring-rule divergences designed for evaluating rare-probability estimates.
SPB remains finite when an estimator returns exactly zero, supports tunable
asymmetry between underestimation and overestimation through an interpretable
asymmetry parameter, and approximately preserves scale invariance across several
orders of magnitude of the ground-truth probability. The unshifted power family
also recovers standard losses as special cases, including squared error and the
Itakura--Saito Loss used by \citet{wu2025estimatingprobabilitiesrareoutputs}
(Appendix \ref{app:standard_loss_connections}). Thus SPB Loss is a strict
generalization of existing evaluation losses, with an explicit finiteness
correction and a tunable asymmetry knob. A comparison with standard losses is
presented in Table \ref{tab:loss_comparison}.

\begin{table}[t]
\centering
\renewcommand{\arraystretch}{1.25}
\setlength{\tabcolsep}{5pt}

\definecolor{navyblue}{RGB}{18,40,160}

\newcommand{\cmark}{\textcolor{green!50!black}{\ding{51}}}
\newcommand{\xmark}{\textcolor{red!75!black}{\ding{55}}}

\begin{tabular}{lcccc}
\toprule
\textbf{Property} &
\textbf{Squared Error} &
\textbf{IS Loss} &
\textbf{Log Squared Error} &
\textcolor{navyblue}{\textbf{SPB Loss}} \\
\midrule

Scale invariance
& \xmark
& \cmark
& \cmark
& \cmark\ \(\approx\) exact \\

Defined when estimate \(=0\)
& \cmark
& \xmark
& \xmark
& \cmark \\

Asymmetric: FN \(>\) FP penalty
& \xmark
& \cmark
& \xmark
& \cmark\ tunable via \(\alpha\) \\

Proper scoring rule
& \cmark
& \cmark
& \xmark
& \cmark \\

\bottomrule
\end{tabular}
\caption{SPB Loss is the only loss satisfying all four LPE-specific desiderata considered here: scale invariance, finiteness at zero estimates, tunable asymmetry between underestimation and overestimation penalties via $\alpha$, and strict propriety.}
\label{tab:loss_comparison}
\end{table}

\begin{figure}[t]
    \centering
    \includegraphics[
        width=0.78\linewidth,
        keepaspectratio,
    ]{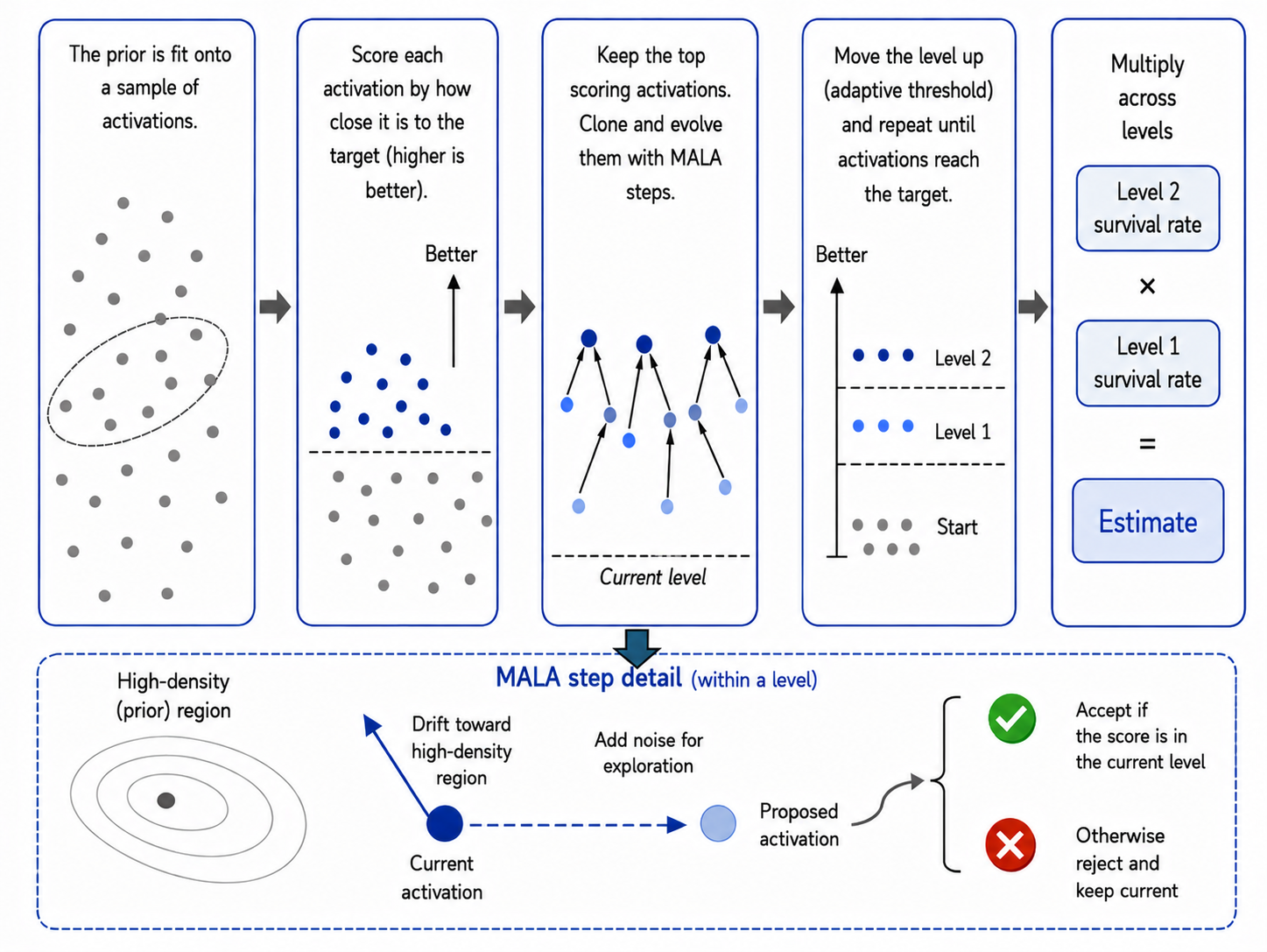}
    \caption{\textbf{GA-AMLS in activation space.} 
The rare-event probability is decomposed into a product of conditional
probabilities over adaptive score levels. At each level, activations below the
threshold are discarded and survivors are resampled. MALA rejuvenation uses
the fitted activation-prior gradient to improve mixing within
the constrained level.}
    \label{fig:ga_amls_algorithm}
\end{figure}

Second, we propose \textit{Gradient Activation Adaptive Multilevel Splitting} (GA-AMLS), which adapts the Adaptive Multilevel Splitting (AMLS) algorithm \citep{article} to operate entirely within the continuous activation space of language models. While \citet{webb_statistical_2019} previously applied AMLS to computer vision models in the input space, ours is, to our knowledge, the first method to apply it in the activation space of language models. Figure~\ref{fig:ga_amls_algorithm} gives a schematic overview of GA-AMLS, while Figure~\ref{fig:ga-amls-overview} previews its comparison with the strongest activation-space baseline, QLD, on a representative distribution.
By working in activation space, we can estimate tail probabilities without the expensive full backward passes required by Importance Sampling. We also avoid needing to compute the input sequence probability under the original distribution, which Importance Sampling methods require for importance weights.
Instead, we decompose the search for activations that trigger the target token into a sequence of intermediate levels, using a gradient-based Metropolis-adjusted Langevin algorithm (MALA) kernel \citep{robertsTweedi1996,srinivasan_high-accuracy_2025} to guide the sampler within each level. The MALA kernel's drift term allows the evolved activations to stay ``close'' to the typical activations of the input distribution, while the adaptive levels progressively steer towards the failure region. This replaces QLD's independence-based recombination with conditional sampling under an explicit prior. Idealized AMLS is an unbiased estimator of the tail probability of its sampling distribution \citep{cerou_adaptive_2019}; because our MCMC kernel targets a fitted activation prior rather than the true activation distribution, GA-AMLS inherits this guarantee only up to prior mismatch (Section~\ref{sec:discussion}).

Empirically, the estimator ranking depends on the evaluation cost model. Under
approximately symmetric penalties, including log-space squared error and
corresponding SPB settings, GA-AMLS has lower bias and outperforms the strongest baseline in the 1- and 4-layer models, with a near tie in the 2-layer model.
Under strongly asymmetric penalties that heavily punish underestimation,
including Itakura--Saito Loss and corresponding SPB settings, QLD can outperform
GA-AMLS because its positive bias acts as a conservative hedge against false
negatives. This rank inversion is visible under pre-existing metrics as well as
under SPB; SPB makes the tradeoff explicit while avoiding the numerical
instability caused by zero estimates. Consequently, the appropriate estimator
depends on the deployment context: conservative catastrophic-risk audits may prefer
overestimation, whereas applications sensitive to excessive false alarms may
prefer the lower-bias GA-AMLS diagnostic.



\section{Background and Problem Setup}

\subsection{Problem Setup}
\textbf{Low-probability estimation.}
We adopt the low-probability estimation (LPE) setting of
\citet{wu2025estimatingprobabilitiesrareoutputs}. Let $M$ be a language model
with vocabulary $\mathcal V$, and let $x \sim \mathcal D$ be an input sequence
drawn from a specified input distribution. We consider deterministic argmax
decoding: $M(x)$ is the token with maximal output logit. For a target token
$t \in \mathcal V$, the rare-event probability of interest is $q_t = \Pr_{x \sim \mathcal D}[M(x)=t].$
In our experiments, $q_t$ lies in the extreme-tail regime
$10^{-9} \le q_t \le 10^{-5}$, where naive Monte Carlo sampling is infeasible.

\textbf{Activation-space formulation.}
GA-AMLS operates on internal activations rather than discrete text. Let
$f:\mathcal V^* \to \mathbb R^d$ map an input sequence to its final-layer
pre-LayerNorm, pre-unembedding activation $a=f(x)$. Given unembedding matrix
$W_U$, the logits are
$z(a) = \mathrm{LayerNorm}(a) W_U.$
The event $M(x)=t$ can therefore be written as an activation-space constraint:
$
    q_t
    = \Pr_{x \sim \mathcal D}
    \left[z_t(f(x)) > \max_{i \neq t} z_i(f(x))\right].
$

\textbf{Existing estimators.}\label{sec:background_methods}
We compare against the methods introduced by
\citet{wu2025estimatingprobabilitiesrareoutputs}. Input-space importance
sampling methods such as ITGIS and MHIS search over discrete token sequences and
reweight samples to estimate $q_t$. These estimators can be unbiased in
principle, but empirically in the extreme-tail regime they often fail to find triggering
inputs, producing zero estimates. QLD avoids this failure
mode by operating in activation space. It whitens pre-unembedding
activations, decomposes each activation as \(u=a+b\) into a target-direction
component and an orthogonal residual, and estimates \(q_t\) by recombining components \(a^{(i)}+b^{(j)}\). This converts \(n\) activations into
\(n^2\) candidate activations, but relies on treating the whitened components as
independent. Since whitening guarantees uncorrelatedness rather than
independence except under stronger assumptions such as joint Gaussianity,
non-Gaussian activation structure can introduce bias.

\subsection{Sampling Methods}
\label{sec:sampling_methods}

The following methods underlie the estimator of Section~\ref{sec:method}.

\textbf{The Adaptive Multilevel Splitting (AMLS) algorithm.}
Adaptive Multilevel Splitting (AMLS) is a rare-event Monte Carlo method that
estimates a small probability by decomposing it into a product of larger
conditional probabilities \citep{article,cerou_adaptive_2019,webb_statistical_2019}. Let $V \sim G$ be a random variable, let
$s:\mathbb R^d \to \mathbb R$ be a score function, and let the rare event be
$\{s(V)\geq \tau\}$. For levels
$-\infty=L_0<L_1<\cdots<L_K=\tau$,
$
    \Pr(s(V)\geq \tau)
    =
    \prod_{k=1}^K
    \Pr(s(V)\geq L_k \mid s(V)\geq L_{k-1}).
$

AMLS estimates these factors using a population of $N$ particles. At each level,
the threshold $L_k$ is chosen adaptively from the empirical score distribution;
particles below $L_k$ are discarded, survivors are resampled, and the resulting
duplicates are rejuvenated with an MCMC kernel targeting
$G(\cdot \mid s(V)\geq L_k)$. The final estimate is
$
    \hat q_{\mathrm{AMLS}}
    =
    \prod_{k=1}^K
    \hat p_k$, where
    $\hat p_k
    =
    \frac{1}{N}
    \sum_{i=1}^N
    \mathbf 1\{s(V_i)\geq L_k\}.
$
Thus, AMLS replaces one extremely rare event with a sequence of more frequent
conditional events.

\textbf{MALA proposal kernel.}
The efficiency of AMLS depends on the MCMC kernel used to rejuvenate particles.
We use the Metropolis-adjusted Langevin algorithm (MALA), which proposes moves
using local gradient information. For a target density $\pi$, MALA proposes \citep{robertsTweedi1996,srinivasan_high-accuracy_2025}
$
    x' = x + h \nabla_x \log \pi(x) + \sqrt{2h}\,\xi$,
    where $\xi \sim \mathcal N(0,I),
$ and accepts or rejects the proposal using the standard Metropolis--Hastings
correction.

\section{Shifted-Power Bregman (SPB) Loss for Evaluation}

We introduce Shifted-Power Bregman (SPB) Loss, a proper scoring rule for evaluating
rare-probability estimates. SPB remains finite when an estimator returns zero
and provides tunable asymmetry between underestimation and overestimation. We define
both the pointwise loss and a dataset-level aggregation that controls scale dependence
across rare-event probabilities.

\textbf{Desiderata for an ideal loss for LPE.} To evaluate an estimator for LPE, an ideal loss would satisfy the following requirements:

\textit{Scale Invariance:} As mentioned in \cite{wu2025estimatingprobabilitiesrareoutputs}, it should penalize relative errors similarly across the scale of ground truth probabilities, as we care about performance across a wide range of rarities (e.g., $10^{-5}$ to $10^{-9}$). 

\textit{Asymmetric Sensitivity:} It should support tunable asymmetry between underestimation and overestimation penalties, as different deployment contexts may prioritize avoiding underestimation over overestimation, or vice versa.

\textit{Numerical Stability:} It must be defined when the estimate is exactly zero, avoiding the need for calibration steps that require ground truth knowledge.

\textit{Statistical Rigor:} It should be the divergence of a proper scoring rule, so that an estimator minimizing expected loss is incentivized to report its unbiased best estimate rather than a systematically distorted one \citep{gneiting_strictly_2007}.

The Itakura--Saito (IS) loss used in \cite{wu2025estimatingprobabilitiesrareoutputs} satisfies most criteria but fails on numerical stability: it is undefined (infinite) when the estimate is $0$. \cite{wu2025estimatingprobabilitiesrareoutputs} overcome this by fitting an affine transformation $x \mapsto ax^c + b$ to the outputs, where parameters are chosen to minimize IS Loss via leave-one-out cross-validation (LOOCV) for each (method, input, distribution) pair. A significant drawback is that this transformation relies on knowing the ground truth probabilities to calculate the IS Loss which is to be minimized during this calibration. In a realistic deployment setting, establishing this ground truth for rare tokens is computationally infeasible using naive sampling, as estimating these probabilities is the very problem we are trying to solve. SPB uses ground truth probabilities only to define the offline evaluation score: to compare each estimate with its ground truth value, to choose the finite zero-estimate shift, and, for the dataset-level version introduced below, to assign fixed evaluation weights. It does not use ground truth to transform an estimator's outputs before scoring. By contrast, the affine calibration used in prior evaluation procedures assumes that ground-truth probabilities are available to \emph{transform each estimator's outputs} before scoring. This is a stronger requirement than ordinary evaluation against ground truth, since such calibration information would generally not be available in deployment.

\textbf{Pointwise SPB Loss.} We propose pointwise Shifted-Power Bregman (SPB) Loss, derived from the framework of proper scoring rules. It overcomes this limitation by accepting a trade-off: if the estimate is $0$, the loss will not be infinite, but rather capped by a large finite constant. We also introduce parameters $\alpha$: to control the degree of asymmetric sensitivity, and $\gamma$: a dataset-level rarity premium.

Following \cite{buja_loss_nodate}, every smooth proper scoring rule for binary probability estimation induces a Bregman divergence determined by a nonnegative weight function $\omega(t) \ge 0$ on $(0,1)$. If $q$ denotes the true probability and $p$ the forecast, the associated divergence can be written as $
    B(q \mid p) = \int_q^p (t - q) \omega(t) \, dt.
$

We define the shifted-power weight function as: 
    $\omega_\varepsilon(t) = (t + \varepsilon)^{-\alpha}$, where $\alpha > 0, \varepsilon > 0$.
The shift parameter $\varepsilon$ ensures the penalty remains finite at $p=0$. 
Geometrically, this integral accumulates cost along the one-dimensional
probability axis between the truth $q$ and the prediction $p$. At an
intermediate probability level $t$, the local contribution is the distance
from the truth, scaled by the weight $\omega_\varepsilon(t)=(t+\varepsilon)^{-\alpha}$.
Since this weight is largest near zero, errors that push the prediction
toward smaller probabilities are penalized more strongly. For an overestimate
$p=mq$, the loss integrates over $t\in[q,mq]$, where the weight decreases as
$t$ increases. For a reciprocal underestimate $p=q/m$, the loss integrates
over $t\in[q/m,q]$, which includes smaller probabilities where
$\omega_\varepsilon(t)$ is larger. Increasing $\alpha$ steepens this growth
near zero and therefore increases the relative penalty for underestimation
compared with overestimation. Thus, $\alpha$ provides a tunable asymmetry
parameter, as formalized in Theorem~\ref{thm:asymmetry}.

\textbf{Dataset-level SPB Loss and the rarity premium $\gamma$.} To evaluate an estimator across a diverse set of rare events, we aggregate the pointwise divergences into a single dataset-level loss. As derived in Appendix~\ref{sec:AppendixA-Derivations}, the \emph{unshifted} pointwise divergence ($\varepsilon=0$) scales exactly proportionally to $q^{2-\alpha}$ for a fixed relative error $m=p/q$. While the shift parameter $\varepsilon > 0$ slightly perturbs this geometry, the shifted divergence $B_\varepsilon(q \mid p)$ remains dominated by this leading-order $q^{2-\alpha}$ scaling. Left uncorrected, an aggregate loss would be dominated by the absolute scale of the ground-truth probabilities rather than the relative accuracy of the estimators. 

To neutralize this and establish controlled scale invariance, we define the dataset-level SPB Loss as the arithmetic mean\footnote{Crucially, we divide by the sample size $n$ rather than the sum of the weights $\sum_i w(q_i)$. Since the weight $w(q)$ is specifically designed to perform a pointwise cancellation of the $q^{2-\alpha}$ scale factor, applying a normalized weighted average would reintroduce a dependence on the empirical distribution of the ground truth probabilities $q$, undoing the cancellation.} of the \emph{weighted} pointwise divergences across $n$ samples: $
    \mathcal{L}_{\text{SPB}} \;=\; \frac{1}{n}\sum_{i=1}^n w(q_i)\, B_\varepsilon(q_i \mid p_i).
    \label{eq:spb_dataset_metric}
$
We define the weight function as:
    $w(q) \;=\; q^{\,\alpha - 2 - \gamma}$, where $\gamma \geq 0$.
The parameter $\gamma$ allows us to toggle between two distinct evaluation modes:

\textit{Scale Neutrality ($\gamma = 0$):} By setting $\gamma = 0$, the weight becomes $q^{\alpha-2}$. This term cancels the leading $q^{2-\alpha}$ scaling of the divergence. Consequently, the final metric approximates scale invariance, depending almost entirely on the relative multiplicative error ($m = p/q$) rather than the absolute rarity of the event. The residual distortion due to $\varepsilon$ is formalized and bounded in the Appendix in Theorem~\ref{thm:scale}.

\textit{The Rarity Premium ($\gamma > 0$):} In specific safety-critical deployments, evaluating performance symmetrically across scales is less important than severely punishing errors on the deepest tail events. Setting $\gamma > 0$ intentionally breaks scale neutrality by introducing an extra $q^{-\gamma}$ factor. This acts as a tunable \emph{rarity premium}, disproportionately upweighting the smallest-$q$ events in the aggregate score.

\paragraph{Pointwise versus dataset-level SPB Loss.}
Strict propriety is inherited pointwise from the shifted-power Bregman divergence
\(B_\varepsilon(q\mid p)\), which is uniquely minimized at \(p=q\). The
dataset-level metric
\[
\mathcal{L}_{\mathrm{SPB}}
=
\frac{1}{n}\sum_{i=1}^n w(q_i) B_\varepsilon(q_i\mid p_i)
\]
is a weighted aggregate of these pointwise divergences. Since \(w(q_i)>0\) and
does not depend on the estimate \(p_i\), the aggregate is still uniquely
minimized by the honest vector of estimates \(p_i=q_i\) for all targets. Thus,
the reference-dependent weighting changes how much each target contributes to
the reported average, but it does not create an incentive to misestimate any
individual probability. We therefore use \(\mathcal{L}_{\mathrm{SPB}}\) as an
offline evaluation metric under a specified profile \((\alpha,\gamma,\varepsilon)\).

\paragraph{Relation to other losses.} The unshifted power family ($\varepsilon=0$) unifies standard losses: $\alpha=0$ recovers squared error and
$\alpha=2$ is \emph{exactly} the Itakura--Saito Loss (Appendix \ref{app:standard_loss_connections}). Therefore SPB comes from a strict generalization of the
metric used by \citet{wu2025estimatingprobabilitiesrareoutputs}, adding
finiteness at zero estimates (with $\varepsilon>0$) and a tunable asymmetry knob.

\subsection{Practical Parameter Selection ($\alpha$, $\gamma$, $\varepsilon$)}\label{sec:param_guidance}

The metric has three user-facing parameters: 1) $\alpha$ sets the FN-vs-FP asymmetry: $\alpha > 1.5$ penalizes underestimation more than overestimation and $\alpha < 1.5$ the reverse (Theorem~\ref{thm:asymmetry}, Eq.~\ref{eq:eps_asym_thm} in Appendix; fine-grained control of $\alpha$ is described in Appendix~\ref{sec:choose_alpha}). 2) $\gamma \ge 0$ is an optional rarity premium: $\gamma = 0$ gives scale invariance up to a tolerance $\eta$ (Theorem~\ref{thm:scale} in Appendix), while $\gamma > 0$ multiplies each example's contribution to the aggregate loss by $q^{-\gamma}$, so at a fixed relative error a $10\times$ rarer event carries $10^{\gamma}\times$ the weight. 3) $\varepsilon$ is the floor that makes $p = 0$ finite at the cost of controlled distortion; a heuristic default is $\varepsilon = 0.01\, q_{\min}/m_{\max}$, and a guaranteed choice is any fraction of the bound given in Appendix by Corollary~\ref{cor:simultaneous} (interpretation in Appendix~\ref{sec:choose_eps}). Table~\ref{tab:alpha_gamma_use_cases} (Appendix~\ref{sec:param_guidance_table}) suggests settings by evaluation use case.

\section{Gradient Activation AMLS (GA-AMLS) for Estimating Rare Events}
\label{sec:method}

With the evaluation loss in place, we turn to the estimator itself. We adapt the general AMLS algorithm to work in the activation space of language models, and use a MALA kernel to guide the search (Section~\ref{sec:sampling_methods}). We call this method Gradient Activation AMLS (GA-AMLS), presented in Appendix~\ref{sec:pseudocode} as Algorithm~\ref{alg:ga-amls} and outlined below. 

\textbf{Notation.} Let \(a \in \mathbb{R}^d\) denote the raw pre-LayerNorm pre-unembedding activation, and let \(u\) denote its whitened coordinate. Given mean \(\mu\) and covariance \(\Sigma\), choose a de-whitening matrix \(A\) such that \(AA^\top=\Sigma\), and define $
u = (a-\mu)A^{-\top},
 a(u) = uA^\top+\mu .
$
GA-AMLS and the MALA kernel operate over \(u\), while model logits are evaluated by inserting the corresponding unwhitened activation \(a(u)\) back into the model.

\textbf{Prepare the activation space.} We use activations from the pre-LayerNorm pre-unembedding layer, and whiten them. 
Raw transformer activations \(a\) are characteristically anisotropic, with a few
outlier dimensions whose scale, in magnitude and in across-distribution
variance, is orders of magnitude larger than the rest. This outlier structure
is documented in large language models \citep{sun2024massive,dettmers2022llmint88bitmatrixmultiplication};
we verify empirically that it also holds for the smaller models we study (Appendix~\ref{app:activation_whitening}, Figure~\ref{fig:anisot}). To improve mixing of MCMC kernels in this high-dimensional space, we operate in whitened coordinates \(u=(a-\mu)A^{-\top}\). This global preconditioning makes isotropic Euclidean steps in \(u\)-space correspond to Mahalanobis-scaled steps in the original activation space. We sample pre-LayerNorm activations so that the evolved activations do not need to satisfy the LayerNorm constraint directly.

\textbf{Define the score function.} We define the score function ${s}$ that directly describes how ``close'' an activation is to an activation producing the target token. 
Given whitened pre-LayerNorm pre-unembedding activation $\textbf{u}$ as input, $A$ as the de-whitening matrix, $\mu$ as the mean, $t$ as the target token index, the score function computes:

\[
s(\textbf{u})
=
\Bigl(\mathrm{LayerNorm}(\textbf{u}A^\top+\mu)\,W_U\Bigr)_t
-
\max_{j\neq t}\Bigl(\mathrm{LayerNorm}(\textbf{u}A^\top+\mu)\,W_U\Bigr)_j.
\]

That is, $s(\textbf{u})$ is the margin by which the target token's logit exceeds the maximum logit among all other tokens.
Using the general AMLS algorithm from Section \ref{sec:sampling_methods}, we define the rare event as $P\bigl(s(\textbf{u}) \ge \tau\bigr)$. In our setting $\tau = 0$, since this is when the logit for the target token $t$ is the largest, and thus $t$ will be the next token predicted under our assumption of argmax sampling.

\textbf{Determine the prior distribution.}
The prior \(\pi\) is defined over whitened activations \(u\) and is used to guide the MCMC kernel. In whitened space, we use a diagonal product Student-\(t\) prior.\footnote{The diagonal prior factorizes across whitened coordinates, so GA-AMLS also carries a coordinate-independence assumption; the difference from QLD lies in how that assumption enters. QLD uses independence operationally, recombining whitened components from different empirical activations as if they were independent, whereas GA-AMLS samples from an explicit fitted prior, whose default diagonal Student-$t$ form can still miss residual dependence and joint-tail structure. Whitening does not close this gap: Figure~\ref{fig:anisot} supports whitening only as global preconditioning and establishes neither coordinate independence nor a joint Student-$t$ fit. GA-AMLS therefore does not eliminate the independence or prior-mismatch concern; rather, it makes the approximation explicit, tunable, and replaceable. Appendix~\ref{app:prior_sensitivity} reports prior-sensitivity diagnostics.} This captures heavier-than-Gaussian marginal tails. This choice is supported by the tuning ablations in Appendix~\ref{sec:hyperparameters}, where matched Gaussian-prior configurations had a median $5.12\times$ higher tuning loss than the corresponding Student-\(t\) configurations. 

\textbf{Choose the MCMC kernel.} To efficiently evolve the clones' activations within the high-dimensional activation space, we use the Metropolis-adjusted Langevin algorithm (MALA) as outlined in Section \ref{sec:sampling_methods}, with the below correction to take into account the constraint induced by the score mechanism of AMLS. 
We must strictly enforce the level constraints $s(\textbf{u}) \ge L_i$. This is addressed through a Metropolis--Hastings filter.

The proposal density $q(\textbf{u}' \mid \textbf{u})$ is Gaussian, while the prior distribution is a Student-$t$ distribution.

The acceptance probability $r$ is computed in two stages to handle the hard constraint:
\begin{enumerate}[leftmargin=*]
    \item Hard score check: If $s(\textbf{u}') < L_i$, the proposal is invalid and immediately rejected ($r = 0$);
    \item Ratio computation: If valid, we compute the standard Metropolis--Hastings (MH) acceptance ratio $r$:
    \begin{equation*}
        r = \min\left(1, \frac{\pi(\textbf{u}') \, q(\textbf{u} \mid \textbf{u}')}{\pi(\textbf{u}) \, q(\textbf{u}' \mid \textbf{u})}\right).
    \end{equation*}
\end{enumerate}
We accept the proposal $\textbf{u}'$ with probability $r$; otherwise, the activation remains at $\textbf{u}$.

Standard Random Walk Metropolis--Hastings (RWMH) proposes isotropic steps blind to the geometry of the distribution, leading to slow mixing in this high-dimensional activation space. MALA improves upon RWMH by using gradient information from the underlying prior distribution to guide the proposal towards the activation space's typical set for the specific input domain $\mathcal{D}$. In GA-AMLS this MALA kernel is applied to the
restricted target $
    \pi_k(u)
    \propto
    \pi(u)\,\mathbf 1\{s(u)\ge L_k\}, $
so proposals below the current level \(L_k\) are rejected by the
MH filter. Thus the Langevin drift improves exploration within
the current constrained level set, while the adaptive thresholds \(L_k\)
move the population toward the rare target-token regions. This choice is also supported by the tuning ablations in Appendix~\ref{sec:hyperparameters}, where the best RWMH configuration in each tuning cell had a median $2.10\times$ higher tuning loss than the best MALA configuration (IQR $1.88$--$2.92$), despite receiving twice as many proposal steps.

To ensure optimal mixing, we adapt the step size $h$ (from MALA under Section~\ref{sec:sampling_methods}) during a burn-in period of $T_{\text{burn}}$ steps. We employ a multiplicative update rule to tune $h$ toward a target acceptance rate of $r^* = 0.57$, which is theoretically optimal for MALA in high dimensions \citep{roberts_optimal_1998}:
    $h_{t+1} = h_{t} \cdot \exp\left( \omega (\bar{r}_t - r^*) \right)$,
where $\omega=0.1$ is the adaptation rate and $\bar{r}_t$ is the empirical acceptance rate at step $t$. After $T_{\text{burn}}$ steps, $h$ is frozen to ensure the Markov chain satisfies the detailed balance condition during the sampling.

\section{Experimental Results}

This section empirically evaluates both components of our proposal. Section~\ref{sec:spb_synth_data_exp} evaluates SPB Loss: a multi-metric evaluation on the real-data LPE benchmark of \cite{wu2025estimatingprobabilitiesrareoutputs} shows that the choice of cost profile determines the apparent ranking of the estimators, and synthetic experiments then confirm the behavior predicted by the theory, including approximate scale invariance, bounded distortion from the \(\varepsilon\)-shift, and tunable asymmetry between underestimation and overestimation. Section~\ref{sec:gaamls_real_data_exp} then analyzes the estimators themselves on the same benchmark, comparing GA-AMLS against input-space Importance Sampling methods and QLD through point-estimate scatterplots and failure-mode diagnostics.

\subsection{SPB Loss: Real-Data and Synthetic Evaluation}
\label{sec:spb_synth_data_exp}

We use SPB Loss as the evaluation loss on the real-data LPE benchmark, where the choice of cost profile determines the apparent ranking of the estimators. The loss geometry these comparisons rely on is established in Section~\ref{sec:param_guidance} and Appendix~\ref{sec:spb_guarantees}; the synthetic experiments that close this subsection verify it in isolation.

\subsubsection{Rank Inversion in Evaluation}

\textbf{Dataset.} We use 8 distributions and the ground-truth probabilities as used in \cite{wu2025estimatingprobabilitiesrareoutputs}. They generated the ground-truth probabilities for each of the 8 distributions by running forward passes on $2^{32}$ random samples. We select a random set of 256 tokens among those with ground-truth probabilities between $10^{-9}$ and $10^{-5}$, and test the methods on these tokens.

\textbf{Methodology.} We follow the experimental setup of \citet{wu2025estimatingprobabilitiesrareoutputs}. To prevent overfitting, the method was only run on the first four distributions during development, and finalized before testing on the last four distributions. We test the methods on 3 models: a 1-layer, a 2-layer, and a 4-layer transformer from \citet{nanda2022transformerlens}. All models have a hidden dimension of $d = 512$, a vocabulary size of $|V| = 48262$, GELU non-linearities \citep{hendrycks2023gaussianerrorlinearunits}, and were trained on the C4 dataset \citep{raffel2023exploringlimitstransferlearning}  and CodeParrot \citep{tunstall2022natural}.  We use the Shifted-Power Bregman (SPB) Loss as our evaluation loss. By varying the asymmetry parameter $\alpha$ and the dataset-level rarity premium
$\gamma$, we demonstrate how different penalization profiles (detailed in Table~\ref{tab:alpha_gamma_use_cases}, Appendix~\ref{sec:param_guidance_table}) alter the apparent ranking of the methods. For hyperparameter tuning details, refer to Appendix \ref{sec:hyperparameters}. For sensitivity to $\varepsilon$, refer to Appendix \ref{app:eps_sensitivity}.

\textbf{Baselines.}
We compare against QLD, ITGIS and MHIS as described in Section \ref{sec:background_methods}.

\begin{figure}[t]
    \centering
    \includegraphics[width=1\linewidth]{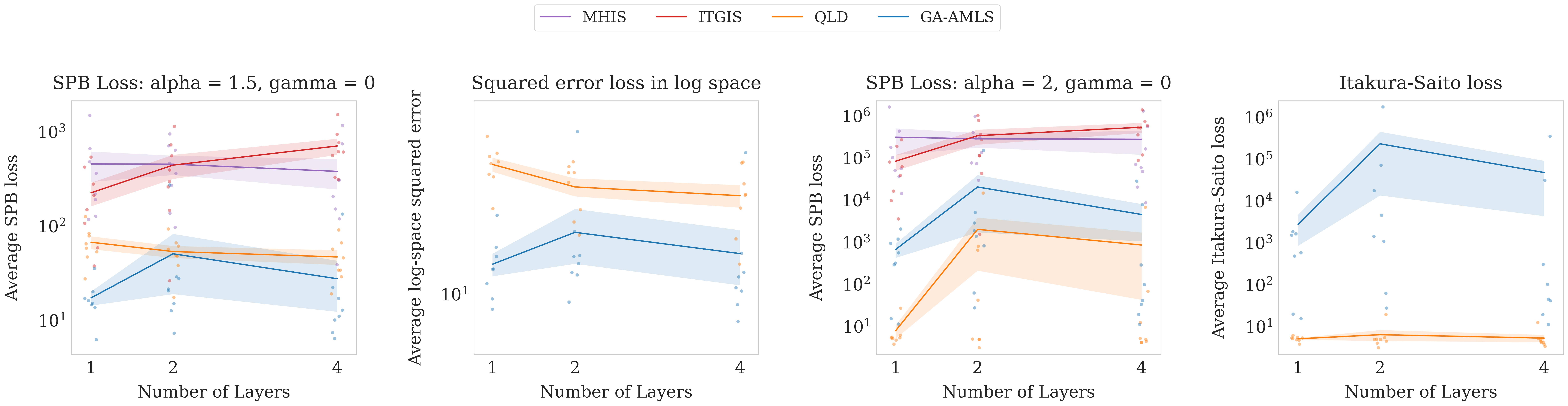}
   \caption{\textbf{Multi-metric evaluation.} Plotted side-by-side, we observe a rank
inversion depending on the metric's sensitivity to underestimation. Under log-space
squared error, GA-AMLS's lower bias gives it lower loss than QLD at all three model
sizes; under the approximately symmetric SPB Loss ($\alpha = 1.5,\varepsilon=\varepsilon_{\mathrm{corr}} = 3.5 \times 10^{-12}$), it outperforms
all baselines on the 1- and 4-layer models and ties with QLD on the 2-layer model.
Under severe asymmetric penalties (SPB $\alpha = 2, \varepsilon=\varepsilon_{\mathrm{corr}}= 3.3 \times 10^{-13}$ and IS Loss), QLD's systematic
overestimation guards it against false-negative penalties, resulting in lower loss.
SPB Loss safely evaluates the input-space methods whose estimates collapse to zero,
overcoming the numerical instability of uncalibrated IS Loss and log-space squared
error. (All panels plotted with log-transformed y-axis.)}
    \label{fig:metrics_compare}
\end{figure}

\textbf{Performance under symmetric penalties.}
\autoref{fig:metrics_compare} compares the overall evaluation losses of all methods across different model sizes grouped by symmetric and asymmetric metrics, while Figure~\ref{fig:SPB-dist} (Appendix~\ref{sec:loss_by_dist}) details the SPB Loss for each of the 8 individual data distributions.
Under metrics that penalize overestimation and underestimation symmetrically, GA-AMLS outperforms QLD on log-space squared error across model sizes; under the approximately symmetric SPB Loss with $\alpha = 1.5$, GA-AMLS outperforms all baseline methods on the 1- and 4-layer models and ties with QLD on the 2-layer model (\autoref{fig:metrics_compare}, first two panels). These comparisons use the full unfiltered benchmark; Appendix~\ref{app:gt_reliability}
repeats them after filtering out the noisiest deep-tail reference probabilities and finds that GA-AMLS's advantage is concentrated in rare-event regimes.

\textbf{Performance under asymmetric penalties.}
The right two panels of \autoref{fig:metrics_compare} show the methods' performance under asymmetric metrics that severely penalize underestimation (SPB Loss with $\alpha = 2$ and IS Loss).
When evaluating with these metrics, we observe a clear inversion in method rankings, where QLD outperforms GA-AMLS under these conditions. This occurs because QLD is biasing its estimates toward overestimation. While less accurate overall, this systematic positive bias acts as a natural safeguard against the severe false-negative penalties imposed by asymmetric losses.

\textbf{The rarity premium $\gamma$.}
Figure~\ref{fig:Layers-SPB} (Appendix~\ref{sec:loss_by_dist}) shows how the aggregate SPB Loss shifts across methods as the rarity premium $\gamma$ increases. The effect of the rarity premium depends on its interaction with $\alpha$. Holding the FN-averse profile $\alpha=2$ fixed,
increasing $\gamma$ from $0$ to $1$ narrows QLD's advantage, because the
reweighting concentrates the average on the deepest tail, where QLD's
overestimation is largest (Section~\ref{sec:gaamls_real_data_exp}); it does not reverse the ordering at the profiles we
evaluate. Under the near-symmetric profile ($\alpha=1.5$), GA-AMLS already
leads at $\gamma=0$. In general, upweighting rarer events favors whichever
method's deep-tail errors are least penalized under the chosen asymmetry, so
$\gamma$ should be read jointly with $\alpha$. Because $\gamma>0$ also
upweights exactly the targets whose ground truth probabilities are noisiest, these
comparisons should be read together with the reliability filtering in Appendix \ref{app:gt_reliability}.

\subsubsection{Synthetic Validation of the Loss Geometry}

We now isolate the behavior of SPB Loss, independently of any estimator. The goal of these synthetic experiments is to confirm the loss geometry discussed in Section~\ref{sec:param_guidance} and formalized in Appendix~\ref{sec:spb_guarantees}: approximate scale invariance, controlled distortion from the finite shift \(\varepsilon\), and tunable asymmetry between underestimation and overestimation.

\textbf{Setup.}
We evaluate rare probabilities \(q\in[10^{-9},10^{-5}]\) on a log-spaced grid and we perturb each probability by fixed multiplicative factors. For scale-invariance experiments, predictions range from \(0.1q\) to \(10q\). For asymmetry experiments, we compare paired reciprocal errors: an overestimate/False-Positive (FP) \(p=mq\) and an underestimate/False-Negative (FN) \(p=q/m\), with \(m\) ranging up to \(1000\). We use the scale-neutral setting \(\gamma=0\).
For finite-shift \(\varepsilon>0\) plots, we compare two choices of \(\varepsilon\): one chosen according to the sufficient bound in Corollary~\ref{cor:simultaneous}, and a simple heuristic choice \(\varepsilon=0.01(q_{\min}/m_{\max})\). The tolerance for the scale-distortion plots is \(\eta=0.05\).

\begin{figure}[t]
    \centering
    \setlength{\tabcolsep}{2pt}

    \begin{tabular}{@{}ccc@{}}
        \begin{subfigure}[!htbp]{0.325\textwidth}
            \centering
            \includegraphics[width=\linewidth]{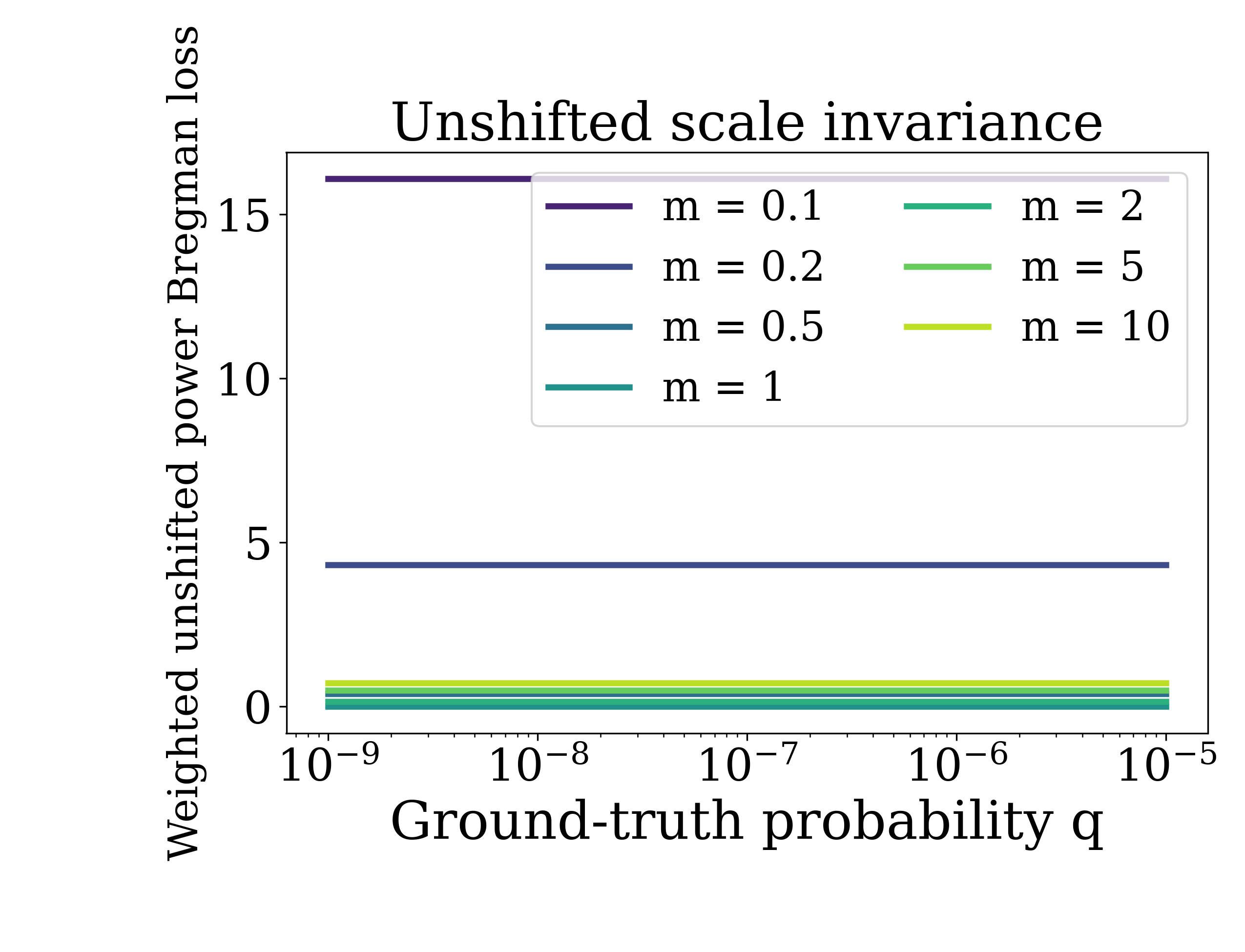}
            \caption{Unshifted scale invariance.}
            \label{fig:scaleinvar-a}
        \end{subfigure}
        &
        \begin{subfigure}[!htbp]{0.325\textwidth}
            \centering
            \includegraphics[width=\linewidth]{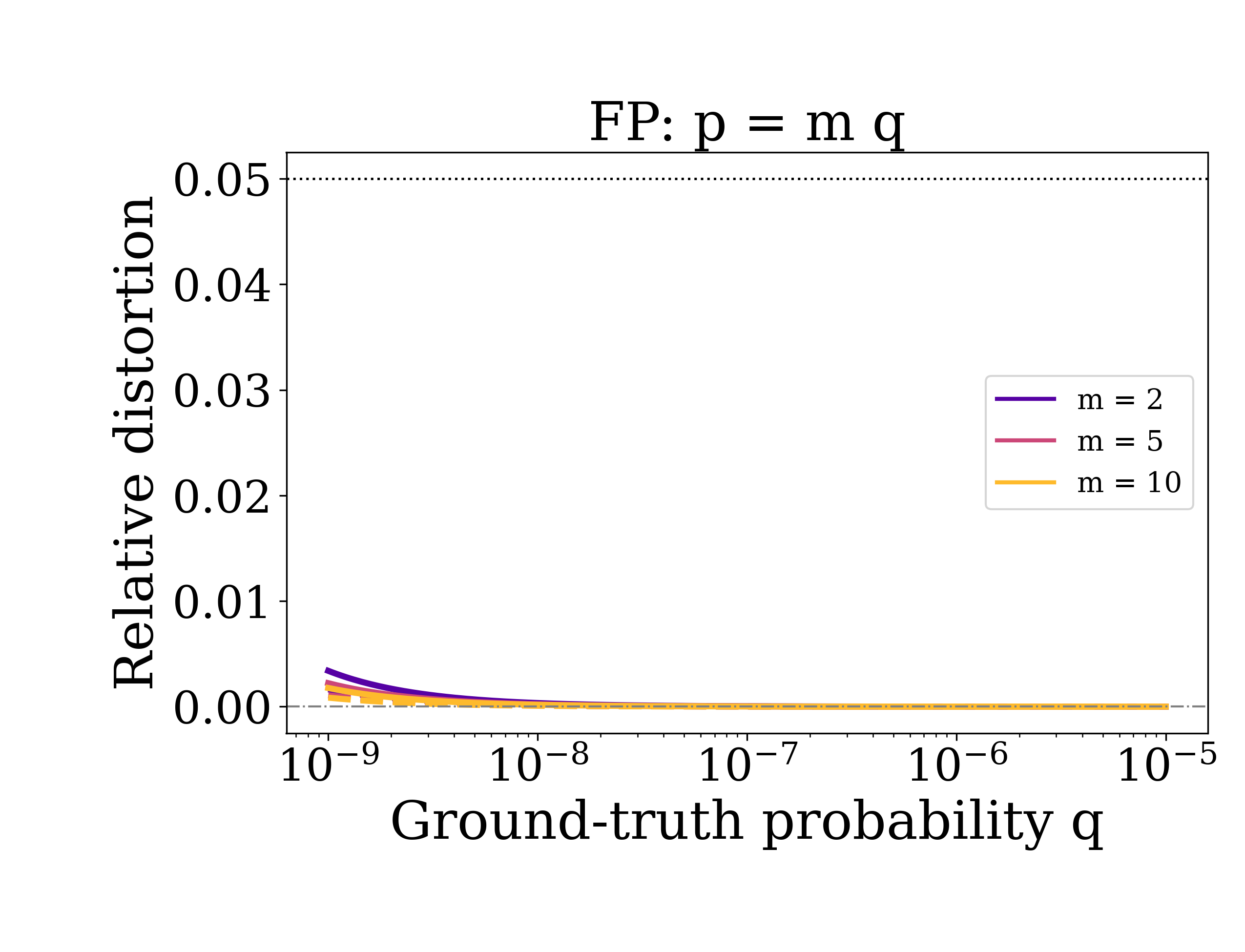}
            \caption{FP distortion.}
            \label{fig:scaleinvar-b}
        \end{subfigure}
        &
        \begin{subfigure}[!htbp]{0.325\textwidth}
            \centering
            \includegraphics[width=\linewidth]{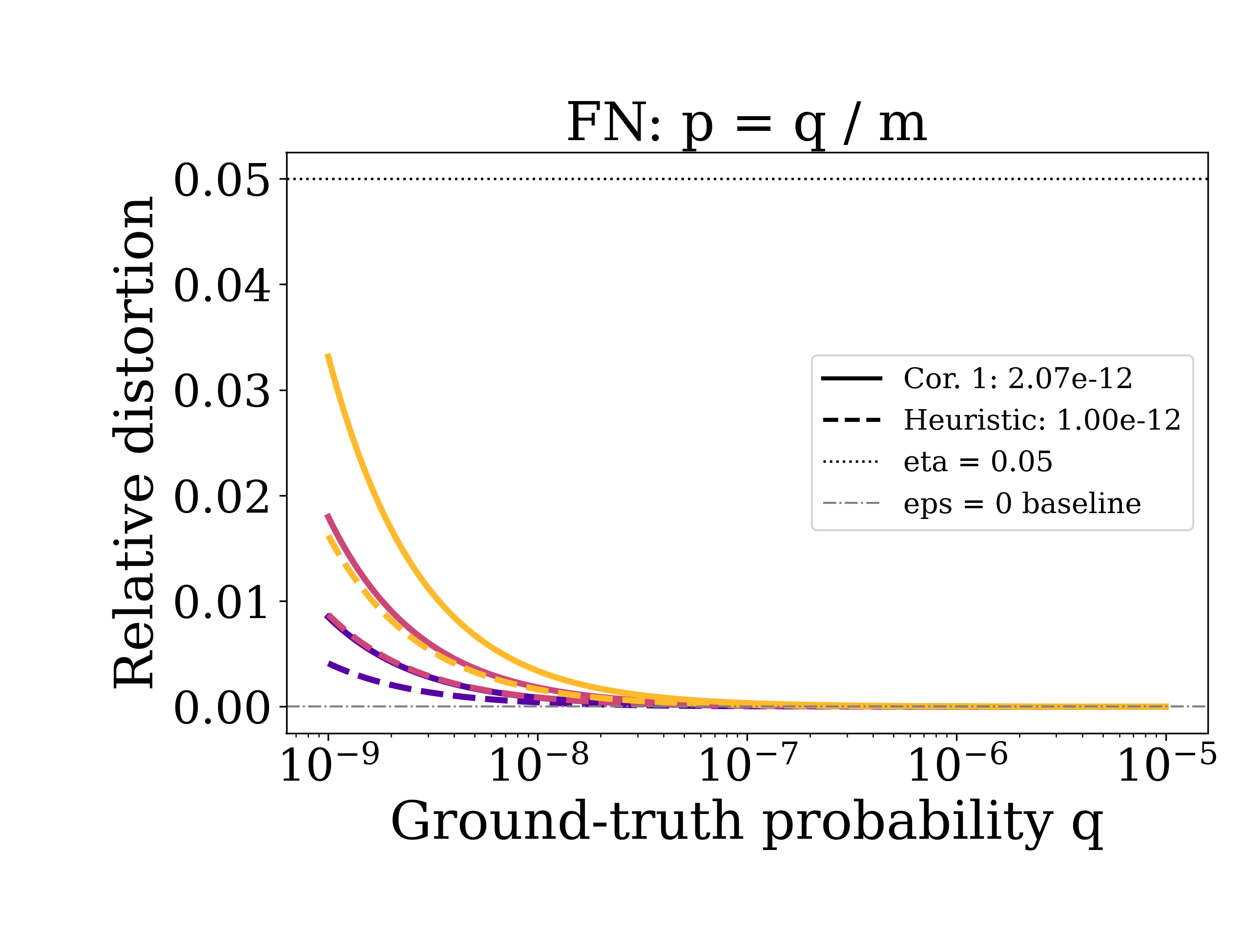}
            \caption{FN distortion.}
            \label{fig:scaleinvar-c}
        \end{subfigure}
    \end{tabular}
    \caption{\textbf{Scale invariance and finite-shift distortion of SPB Loss.}
    \textbf{(a)} With \(\varepsilon=0\), the weighted loss is flat across \(q\) for each fixed multiplicative error, showing exact scale neutrality.
    \textbf{(b)--(c)} Introducing \(\varepsilon>0\) makes the loss finite at zero estimates but slightly perturbs the unshifted geometry. The observed distortion remains below the target tolerance \(\eta = 0.05\) for both overestimation (FP) and underestimation (FN).}
    \label{fig:ScaleInvar_distortion}
\end{figure}

\begin{figure}[t]
    \centering
    \setlength{\tabcolsep}{2pt}

    \begin{tabular}{@{}ccc@{}}
        \begin{subfigure}[!htbp]{0.325\textwidth}
            \centering
            \includegraphics[width=\linewidth]{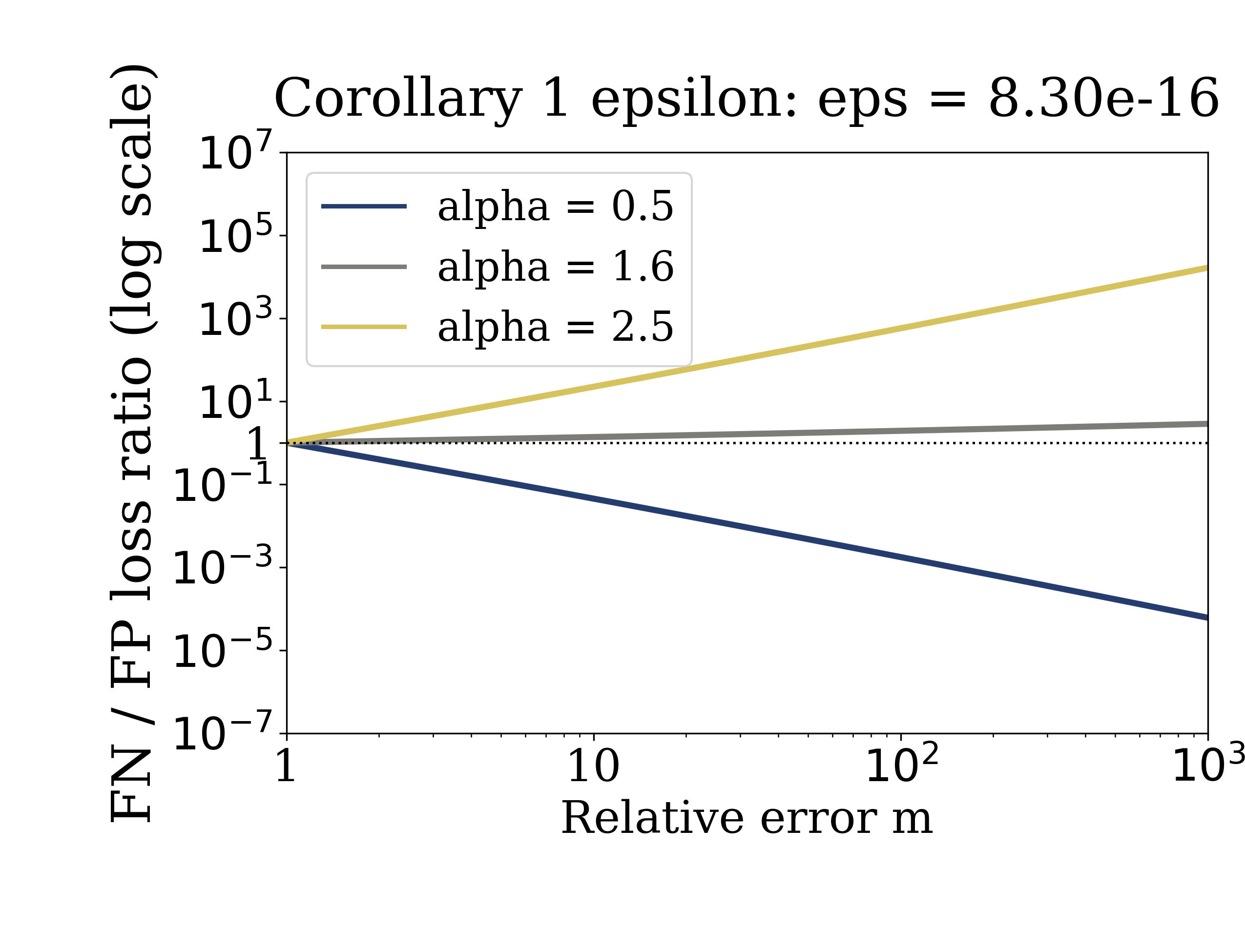}
            \caption{Corollary $\varepsilon$.}
            \label{fig:assymetry-a}
        \end{subfigure}
        &
        \begin{subfigure}[!htbp]{0.325\textwidth}
            \centering
            \includegraphics[width=\linewidth]{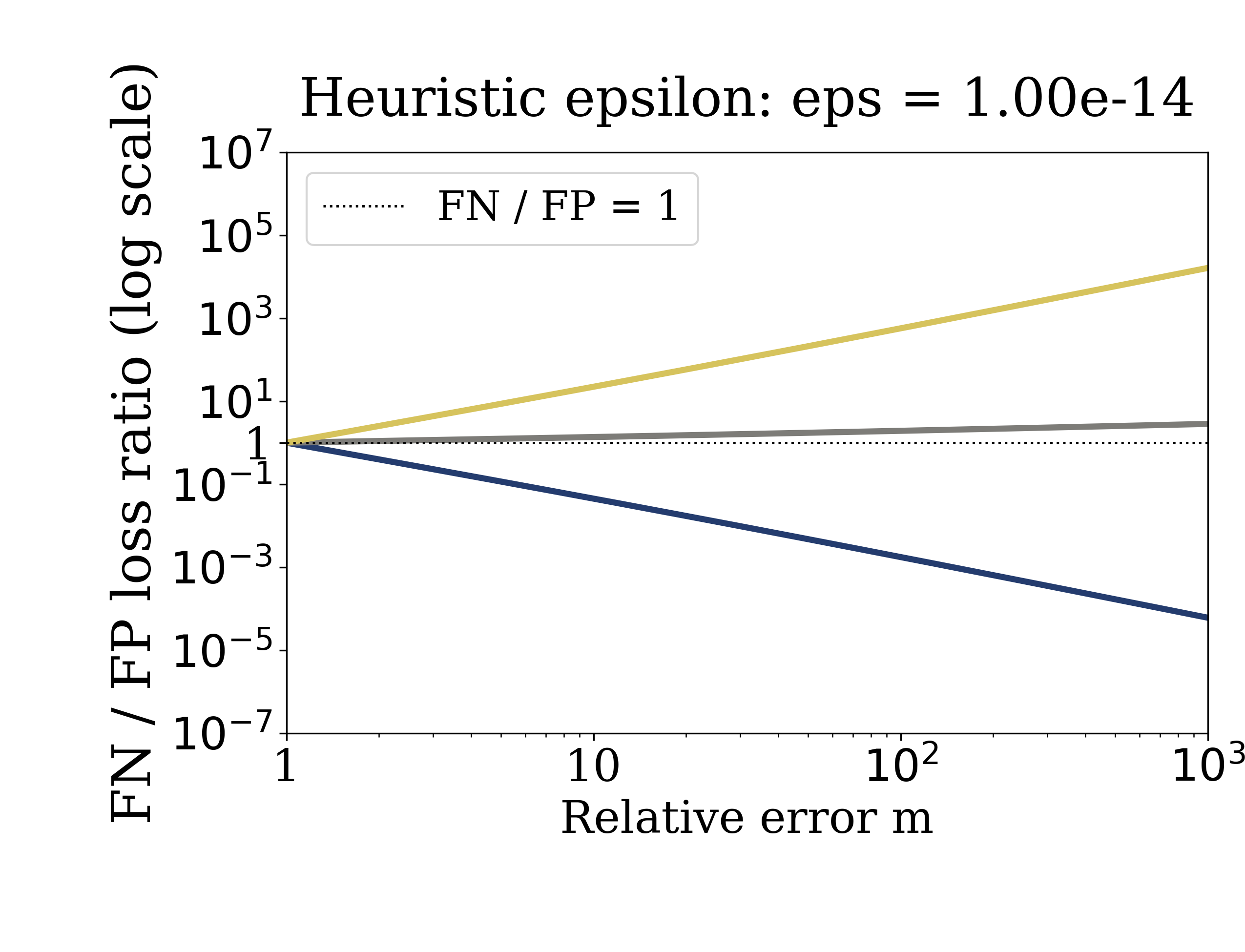}
            \caption{Heuristic $\varepsilon$.}
            \label{fig:assymetry-b}
        \end{subfigure}
        &
        \begin{subfigure}[!htbp]{0.325\textwidth}
            \centering
            \includegraphics[width=\linewidth]{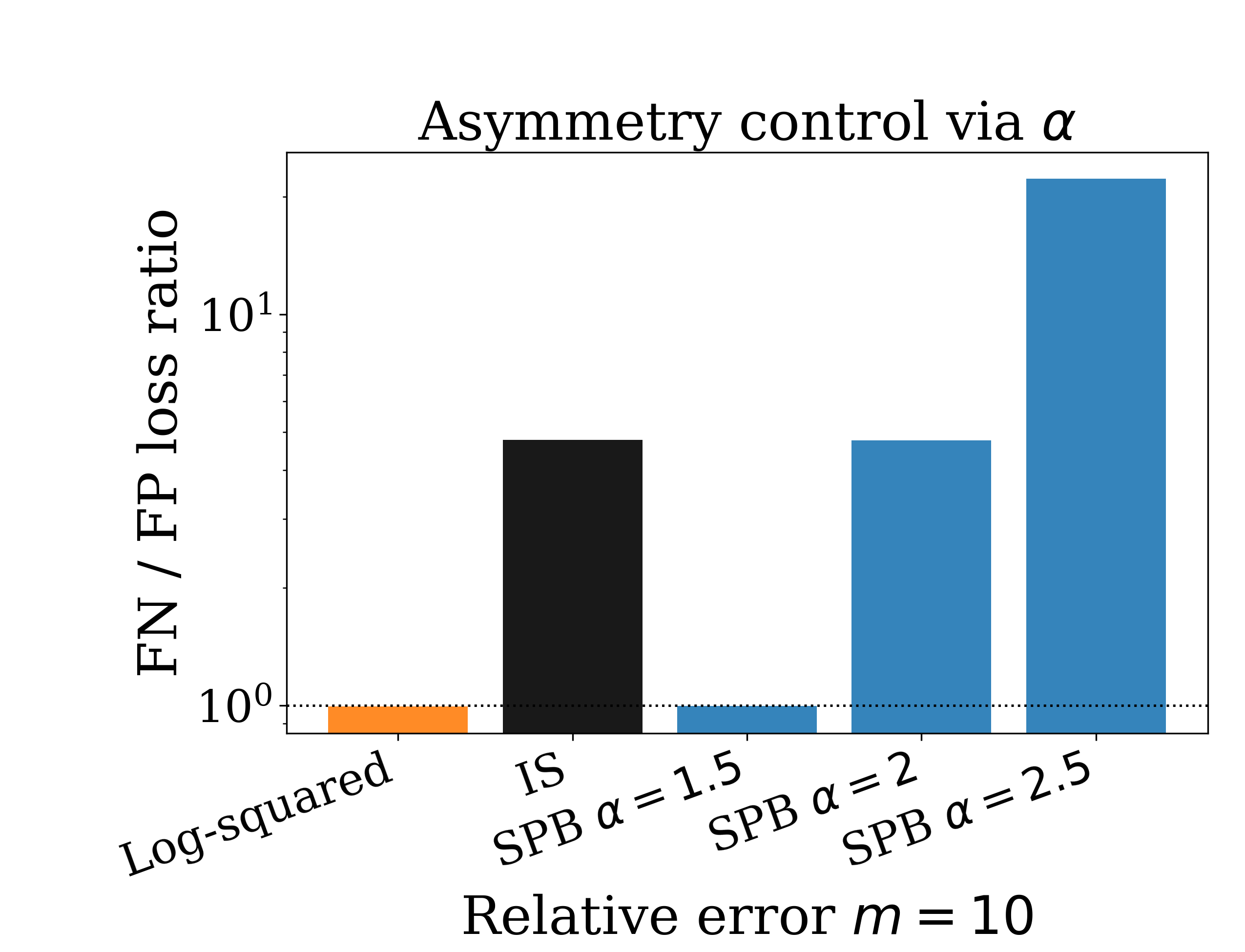}
            \caption{Comparison with standard losses.}
            \label{fig:assymetry-c}
        \end{subfigure}
    \end{tabular}
    \caption{\textbf{Asymmetry control via $\alpha$.}
    \textbf{(a)--(b)} FN/FP loss ratios for reciprocal multiplicative errors under the corollary-based and heuristic choices of \(\varepsilon\). Ratios above one indicate that underestimation is penalized more strongly than overestimation.
    \textbf{(c)} At \(m=10\), SPB Loss provides a tunable continuum between approximately symmetric behavior and strongly FN-averse behavior, compared with log-space squared error and IS Loss.}
    \label{fig:Assymetry}
\end{figure}

\textbf{Scale invariance and shift distortion.}
Figure~\ref{fig:ScaleInvar_distortion} examines whether SPB Loss treats the same relative error similarly across different rarity levels. In Panel~\ref{fig:scaleinvar-a}, the unshifted weighted loss is flat as a function of \(q\): for a fixed multiplicative error, the loss is essentially independent of whether the true probability is \(10^{-5}\) or \(10^{-9}\). This is the desired scale-neutral behavior.
Panels~\ref{fig:scaleinvar-b} and~\ref{fig:scaleinvar-c} then show what changes when the finite shift \(\varepsilon>0\) is introduced. The shift is needed to keep the loss finite when an estimator outputs zero, but it slightly perturbs the ideal unshifted geometry. The plots report this perturbation separately for overestimation and underestimation. For both the corollary-based shift and the heuristic shift, the observed distortion remains below the target tolerance \(\eta=0.05\) over the displayed range. Thus, in this regime, the finite correction preserves the practical scale-neutral behavior of the loss.

\textbf{Asymmetry control via \(\alpha\).}
Figure~\ref{fig:Assymetry} studies the relative penalty for FN versus FP. Panels~\ref{fig:assymetry-a} and~\ref{fig:assymetry-b} plot the FN/FP loss ratio for reciprocal errors: a value above one means that underestimating \(q\) by a factor \(m\) is penalized more than overestimating \(q\) by the same factor. The curves show that \(\alpha\) controls this asymmetry. Values above the threshold \(\alpha=1.5\), such as \(\alpha=1.6\) and \(\alpha=2.5\), produce FN/FP ratios above one across the plotted range, consistent with Theorem~\ref{thm:asymmetry}. The \(\alpha=0.5\) curve is included as a contrasting setting outside this FN-averse regime.
Panel~\ref{fig:assymetry-c} compares this behavior with standard losses at a fixed multiplicative error \(m=10\). Log-space squared error is approximately symmetric between underestimation and overestimation, while IS Loss is FN-averse. SPB Loss interpolates between these behaviors through \(\alpha\): \(\alpha\approx1.5\) is close to symmetric, while larger \(\alpha\) increasingly penalizes underestimation. This gives SPB Loss an explicit asymmetry knob while retaining finiteness at zero estimates.

\subsection{GA-AMLS: Estimator Behavior on the Real-Data Benchmark}
\label{sec:gaamls_real_data_exp}

We now turn from the evaluation instrument to the estimators themselves, using the same benchmark, baselines, and evaluation profiles as Section~\ref{sec:spb_synth_data_exp}.

\begin{figure}[t]
    \centering
    \includegraphics[width=1\linewidth]{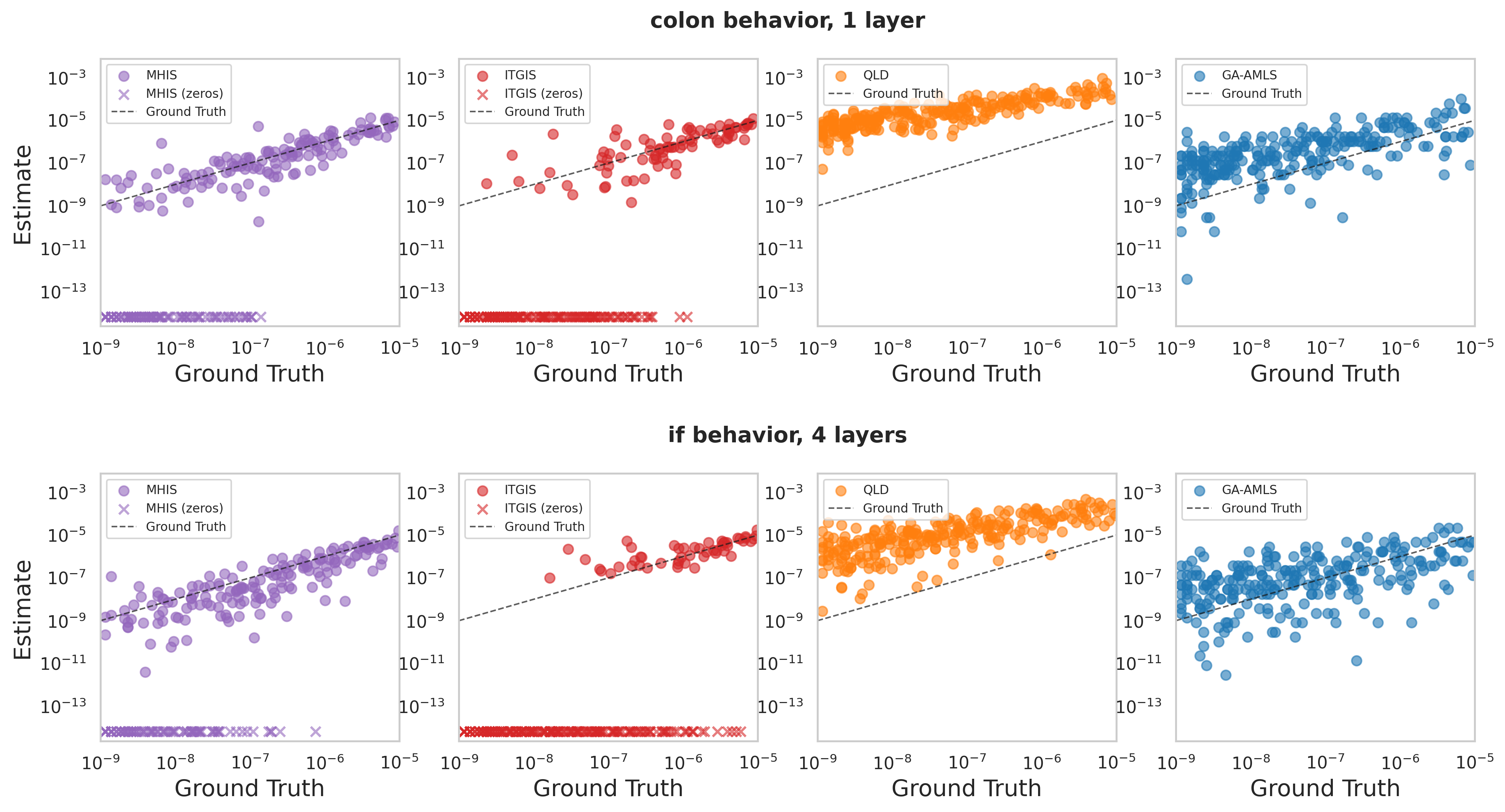}
    \vspace{-0.5em}
    \caption{\textbf{Examples of probability estimates vs.\ ground truth on two input distributions and models.} Input-space methods (ITGIS/MHIS) frequently fail to find valid triggers, causing their estimates to collapse to zero (placed at the bottom of each graph for visibility). While QLD shows a systematic overestimation bias, GA-AMLS centers closest to the ground truth, at the cost of higher cross-token dispersion.}
    \label{fig:2dist_eg}
\end{figure}

\textbf{Deep-tail overestimation.}
\autoref{fig:2dist_eg} plots the methods' individual point estimates against the true probabilities to visualize estimator bias.
The scatterplots (\autoref{fig:2dist_eg} and Appendix \ref{sec:all_scatterplots}) reveal that QLD's overestimation bias is not uniform; it becomes markedly more pronounced at the deepest tail (the rarest events). This deep-tail overestimation is the behavior that the rarity premium of Section~\ref{sec:spb_synth_data_exp} penalizes increasingly as $\gamma$ grows.

\textbf{Estimator failure modes: zero collapse, overestimation bias, and cross-token dispersion.}
The scatterplots of method estimates (\autoref{fig:2dist_eg} and Appendix \ref{sec:all_scatterplots}) also directly capture the distribution and structural failure modes of each estimator's predictions. Since each point is a different target token, the visible spread reflects cross-token error dispersion. The input-space Importance Sampling methods (MHIS/ITGIS) suffer from frequently collapsing to zero for the rarest events. QLD avoids zero-estimates but shifts the entire distribution of predictions upward. GA-AMLS shows lower bias and is more centred around the ground truth line, but exhibits larger cross-token dispersion than QLD (as seen visually in scatterplots in Appendix \ref{sec:all_scatterplots}). A further source of variation is finite-sample noise in the ground-truth probabilities themselves, which is small relative to the cross-token dispersion but non-negligible in the deepest tail; Appendix~\ref{app:gt_reliability} quantifies it and Section~\ref{sec:discussion} discusses its role as an evaluation confounder. Appendix \ref{sec:repeated-run variability} confirms that run-to-run variability for fixed tokens is small relative to this cross-token dispersion. Appendices~\ref{app:mcmc_diagnostics} and~\ref{app:prior_sensitivity} trace specific deep underestimates, most visibly on the \texttt{hex} distribution, to tail mismatch in the diagonal activation prior rather than to insufficient MCMC mixing.

\section{Related Work}

\textbf{Low-probability estimation in LLMs.}
Low-probability estimation for language models was formalized by
\citet{wu2025estimatingprobabilitiesrareoutputs}, who introduced input-space
Importance Sampling methods and the activation-space QLD baseline used in our
experiments. We adopt their benchmark setting, but replace discrete
input-space search with gradient-guided sampling in activation space and
introduce SPB Loss to handle zero estimates and asymmetric safety costs. In a complementary line of work, \cite{jones2025forecastingrarelanguagemodel} address rare behavior estimation from a forecasting perspective rather than direct probability estimation. Their approach leverages elicitation probabilities defined as the probability that a given query produces a target behavior under repeated sampling, and uses extreme value theory to forecast how deployment-scale risks grow from smaller evaluation sets. Their method is particularly suited for anticipating when rare failures might emerge at scale, whereas our approach provides precise probability estimates for individual rare events. The two approaches are complementary: GA-AMLS can provide accurate point estimates of elicitation probabilities that could serve as inputs to the forecasting framework of \cite{jones2025forecastingrarelanguagemodel}, potentially improving the precision of deployment-scale risk predictions. 
Recent work by \cite{mcallisterdorman2026rare} develops an end-to-end framework for rare event analysis over full text completions in the discrete token space, using MCMC-based trajectory sampling from tilted completion distributions. This requires repeated full-sequence generation and mixing in sequence space, whereas GA-AMLS operates in continuous activation space and avoids generating or scoring full completions. Another complementary line of work estimates rare harmful behaviors in stochastic
language-model output distributions. \citet{angell2026estimatingtailriskslanguage}
estimate, for a fixed query, the probability that repeated sampling from a target
model produces a harmful completion. They construct unsafe proposal models using
activation steering and use importance sampling to reweight generated completions
back to the target model, yielding substantial sample savings relative to naive
Monte Carlo. This setting differs from ours in both the source of randomness and
the event definition: they estimate behavior-level probabilities over full
stochastic completions for fixed prompts, whereas we estimate single-token
argmax probabilities over random inputs under deterministic decoding.
Beyond trained transformers, \citet{wu2026estimatingexpectedoutputwide} study
sample-free mechanistic estimation for wide random MLPs with Gaussian inputs.
Their cumulant-propagation algorithms approximate activation distributions
through fixed randomly initialized networks using cumulants and Hermite
expansions, and can outperform Monte Carlo sampling, including on low-probability
threshold events. This is complementary to our approach: their guarantees and
experiments apply to analytically tractable random MLPs and closed-form input
distributions, whereas we study trained transformer language models and use
adaptive rare-event sampling in activation space.

\textbf{Activation-space methods.}
A growing interpretability literature supports treating language-model
activations as a structured domain for analysis and intervention. The linear
representation hypothesis \citep{park2024linearrepresentationhypothesis},
representation engineering \citep{zou2023representationtopdownapproach}, and
activation steering methods
\citep{turner2023steeringlanguagemodelsactivation, panickssery2024steeringllama2contrastive}
show that model behaviors can often be manipulated through directions in
residual-stream space. Recent work further suggests that such behavioral
structure need not be captured by a single global direction: heterogeneity-aware
steering methods model activations as clustered distributions and use optimal
transport to construct input-dependent interventions
\citep{abdullaev2026conceptheterogeneityawarerepresentationsteering}. Related work finds
low-dimensional structure for refusal behavior
\citep{arditi2024refusalmediatedsingledirection}, while sparse autoencoders
reveal richer monosemantic feature structure in activations
\citep{bricken2023towardsmonosemanticity, templeton2024scalingmonosemanticity,
gao2024scalingevaluatingsparseautoencoders}. The Tuned Lens
\citep{belrose2023elicitinglatentpredictionstuned} further motivates operating
near the pre-unembedding layer, where activations are predictive of output
tokens. GA-AMLS builds on this view of activation space, but treats it as a
probability space: instead of identifying causal directions, interventions, or
features, we estimate the probability mass of rare-event regions.

\textbf{Rare-event sampling.}
Our method adapts Adaptive Multilevel Splitting (AMLS), a rare-event Monte
Carlo technique from statistical physics and particle simulation
\citep{cerou_adaptive_2019}. \citet{webb_statistical_2019} applied AMLS to
neural-network robustness verification in computer vision using a random-walk
Metropolis--Hastings kernel in input space. Directly applying this approach to
language models is difficult because text is discrete. GA-AMLS instead moves
the estimator to continuous activation space, where gradient-based MCMC kernels
such as MALA can be used efficiently.

\textbf{Adversarial search and defenses.}
Red-teaming and jailbreaking methods search for individual inputs that trigger
undesired behavior, using methods ranging from discrete gradient-based attacks
such as GCG \citep{zou2023universal} to activation-space attacks such as
Subspace Rerouting \citep{winninger2025subspace}. These methods are useful for
finding failures and generating adversarial training data
\citep{goodfellow2015explainingharnessingadversarialexamples,
madry2019deeplearningmodelsresistant}, including in latent adversarial training
settings \citep{casper_defending_2024, sheshadri_latent_2025}. However, finding
one failure is not the same as estimating its probability, and transfer across
search methods can be unreliable
\citep{kang2019transferadversarialrobustnessperturbation,
wei2023jailbrokendoesllmsafety}. GA-AMLS is therefore diagnostic rather than
defensive: it estimates how much probability mass lies in the rare-event region.

\section{Discussion and Limitations}
\label{sec:discussion}
While GA-AMLS provides a principled activation-space estimator for rare model behaviours, several limitations remain.

\textbf{Residual bias and activation-space approximation.}
GA-AMLS is designed as an unbiased rare-event estimator under the assumed activation prior and exact MCMC sampling. In practice, residual bias can arise from three sources. First, we approximate the true activation distribution with a fitted Student-$t$ distribution, which introduces distributional mismatch. Second, finite MCMC budgets may prevent full mixing between AMLS levels, leaving residual correlations among particles. However, diagnostics in Appendix \ref{app:mcmc_diagnostics} suggest mixing is not the dominant cause; see Appendix \ref{app:prior_sensitivity}. Third, because GA-AMLS operates in continuous activation space, extreme activations found by the sampler may not correspond to any valid discrete token sequence. We partially mitigate this activation pre-image issue by fitting the prior to activations induced by samples from the input distribution $\mathcal{D}$, encouraging samples to remain near the empirical activation manifold. However, the gap between continuous activations and discrete inputs remains a fundamental limitation of activation-space methods, including QLD.

\textbf{Computational cost.}
GA-AMLS has a different computational profile from input-space importance 
sampling. It avoids full-model backward passes and operates from cached activations, but it requires many sequential activation-space scorer evaluations across AMLS levels.
GA-AMLS is computationally heavier than QLD: QLD reduces estimation to fast closed-form matrix operations, whereas GA-AMLS requires sequential Markov-chain evolution across adaptive levels. This additional cost is the price of replacing
QLD's independence-based recombination approximation with conditional sampling under a fitted activation prior. It may therefore limit very large sweeps over many target tokens, layers, or models. For details on computational cost please refer to Appendix \ref{app:compcost}.

\textbf{Evaluation scope.}
For comparability with prior low-probability estimation work, our experiments follow the benchmark setting of \citet{wu2025estimatingprobabilitiesrareoutputs}: the input distribution $\mathcal{D}$ consists of independent tokens, and rare events are defined by a single target token under deterministic argmax conditions. This controlled setting enables direct comparison with ITGIS, MHIS, and QLD, but it does not capture correlated natural prompts, instruction-following contexts, or multi-step autoregressive failures. Extending activation-space rare-event estimation to realistic prompt distributions and trajectory-level events is an important direction for future work. Finally, all of our evidence comes from the $1$--$4$ layer, $d=512$ transformers of this benchmark; whether the observed bias, cross-token error-dispersion profile of GA-AMLS, and the adequacy of a diagonal Student-$t$ prior in whitened coordinates, persist at modern scale is untested. Thus our empirical results should be interpreted within the controlled LPE benchmark used for evaluation. We therefore treat rare LLM failures and deployment-scale safety applications as motivation and future applicability, rather than as outcomes established by the present experiments. 

\textbf{Uncertainty in the reference probabilities.}
The benchmark's ground-truth probabilities are themselves finite-sample Monte Carlo estimates, so the evaluation carries a source of uncertainty beyond estimator error, and in the deepest tail this noise acts as a confounder: losses on the rarest individual targets partly measure reference noise rather than estimator quality, and the rarity premium $\gamma>0$ upweights the targets where this noise is largest. At $q=10^{-9}$, the $2^{32}$-sample reference rests on roughly four expected hits, a relative standard error of about $48\%$, and the standard deviation of $\log_{10}\hat{q}_{\mathrm{ref}}$ reaches about $0.19$ in the lowest count band (Table~\ref{tab:rarity_bands}). This is small relative to the cross-token error dispersion of the estimators ($0.6$--$3.1$; Appendix~\ref{sec:repeated-run variability}), so reference noise does not explain the estimator profiles of Section~\ref{sec:gaamls_real_data_exp}, but it is non-negligible for deep-tail aggregate losses. Appendix~\ref{app:gt_reliability} therefore quantifies this noise floor and repeats the main comparisons under reference-count filters.

\textbf{Empirical estimation versus certification.}
GA-AMLS is an empirical risk-estimation method, not a formal safety-certification method. Its estimates are tail probabilities under a fitted activation prior $\pi_{\theta,D}$ and a chosen rare-event score, not certified bounds on harmful behavior under arbitrary deployment prompts. Prior mismatch, finite MCMC error,  and the possibility that continuous activation samples lack valid discrete pre-images all limit the direct safety interpretation of the estimates. We therefore treat GA-AMLS estimates as diagnostic evidence that should be combined with input-space validation, and deployment-specific safety analysis rather than as standalone guarantees.

\todo[inline]{
\textbf{The "So What?"} \\
Wu et al. suggest minimizing probability estimates requires differentiable analytical functions (like QLD). \\
\textbf{rebuttal:} While AMLS isn't analytically differentiable, it generates diverse \textit{latent} failure examples. This enables \textbf{Latent Adversarial Training (LAT)}. Unlike input-space "whack-a-mole," latent sampling finds regions of failure to regularize.
}

\section{Conclusion}
\label{sec:conclusion}
We introduced SPB Loss to evaluate rare-event estimates under asymmetric costs for underestimation and overestimation, and GA-AMLS, an activation-space adaptive multilevel splitting method for estimating rare LLM behaviors that are infeasible to measure accurately with naive sampling. Empirically, GA-AMLS reduces zero-estimate failures and under symmetric penalties often improves rare-probability estimation relative to baselines.

\subsubsection*{Acknowledgments}
This project is partially supported by Coefficient Giving. The authors would like to thank anonymous reviewers for their helpful feedback.

\bibliography{main}
\bibliographystyle{tmlr}

\newpage
\appendix

\paragraph{Code availability.}
This repository contains the implementation and example notebook: \url{https://github.com/ny-parulekar/rare-event-estimation}.

\paragraph{Appendix outline.}
Appendix~\ref{sec:AppendixA-Derivations} develops the SPB Loss theory from its
Bregman-divergence representation, derives the unshifted power family, and
analyzes the resulting false-negative/false-positive asymmetry. 
Appendix~\ref{sec:spb_guarantees} gives the shifted-power parameterization used
in the paper, proves bounded scale distortion, explains how to choose
\(\varepsilon\), \(\alpha\), and \(\gamma\), provides a closed-form computation,
connects SPB to standard losses, and empirically verifies the intended scaling
behavior. Appendix~\ref{sec:pseudocode} gives the full GA-AMLS procedure in
Algorithm~\ref{alg:ga-amls}. Appendix~\ref{sec:exp_details} reports experimental and implementation
details, including hyperparameters, baseline implementation notes,
hyperparameter sweeps and ablations, and the activation-whitening diagnostic.
Appendix~\ref{sec:all_scatterplots} provides the full set of estimate versus 
ground-truth scatterplots across model sizes and input distributions.
Appendix~\ref{sec:loss_by_dist} breaks down SPB Losses by distribution, model
size, method, and evaluation profile. Appendix~\ref{sec:repeated-run variability} reports repeated-run variability for GA-AMLS
and QLD, separating run-to-run noise from cross-token dispersion.
Appendix~\ref{app:compcost} reports per-target computational cost for all four estimators.
Appendix~\ref{app:mcmc_diagnostics} reports MALA acceptance-rate and within-level
autocorrelation diagnostics. Appendix~\ref{app:prior_sensitivity} probes sensitivity to
the fitted Student-$t$ activation prior, including degrees of freedom and a diagonal
versus shared-joint-tail comparison. Appendix~\ref{app:eps_sensitivity} sweeps the shift
$\varepsilon$ and shows the empirical rankings are stable across four orders of magnitude.
Appendix~\ref{app:gt_reliability} quantifies finite-sample noise in the reference
probabilities and repeats the main comparisons under reliability filtering.

\section{SPB Loss Theory}\label{sec:AppendixA-Derivations}

\subsection{Bregman Integral Representation}
\label{app:bregman_integral}

Let \(q\in(0,1)\) denote the true probability and \(p\in[0,1]\) the predicted
probability. A one-dimensional Bregman divergence generated by a twice
differentiable convex potential with second derivative \(\omega(t)>0\) can be
written as
\begin{equation}
    B_\omega(q\mid p)
    =
    \int_q^p (t-q)\omega(t)\,dt.
    \label{eq:bregman_integral_oriented}
\end{equation}
This is an oriented integral. Equivalently,
\[
B_\omega(q\mid p)
=
\begin{cases}
\displaystyle
\int_q^p (t-q)\omega(t)\,dt, & p\ge q, \\[8pt]
\displaystyle
\int_p^q (q-t)\omega(t)\,dt, & p<q.
\end{cases}
\]
Thus the divergence is nonnegative, and if \(\omega(t)>0\) almost everywhere,
then \(B_\omega(q\mid p)=0\) if and only if \(p=q\), whenever the integral is
finite.

\subsection{The Unshifted Power Family}
\label{app:unshifted_power_family}

Consider the unshifted power weight
\begin{equation}
    \omega_0(t)=t^{-\alpha},
    \qquad \alpha>0.
    \label{eq:unshifted_weight}
\end{equation}
For a multiplicative overestimate \(p=mq\), with \(m>1\), the change of
variables \(t=qu\) gives
\begin{align}
    B_0(q\mid mq)
    &=
    \int_q^{mq}(t-q)t^{-\alpha}\,dt \nonumber\\
    &=
    q^{2-\alpha}
    \int_1^m (u-1)u^{-\alpha}\,du \nonumber\\
    &=
    q^{2-\alpha}I_{FP}(\alpha,m),
    \label{eq:unshifted_fp}
\end{align}
where
\begin{equation}
    I_{FP}(\alpha,m)
    :=
    \int_1^m (u-1)u^{-\alpha}\,du.
    \label{eq:ifp_def}
\end{equation}
Similarly, for a multiplicative underestimate \(p=q/m\),
\begin{align}
    B_0(q\mid q/m)
    &=
    \int_{q/m}^{q}(q-t)t^{-\alpha}\,dt \nonumber\\
    &=
    q^{2-\alpha}
    \int_{1/m}^1 (1-u)u^{-\alpha}\,du \nonumber\\
    &=
    q^{2-\alpha}I_{FN}(\alpha,m),
    \label{eq:unshifted_fn}
\end{align}
where
\begin{equation}
    I_{FN}(\alpha,m)
    :=
    \int_{1/m}^1 (1-u)u^{-\alpha}\,du.
    \label{eq:ifn_def}
\end{equation}

Therefore both false-positive and false-negative multiplicative errors factor
into a ground-truth scale term \(q^{2-\alpha}\) and a relative-error term. The
dataset-level weight
\begin{equation}
    w(q)=q^{\alpha-2}
    \label{eq:scale_neutral_weight}
\end{equation}
cancels the \(q^{2-\alpha}\) prefactor exactly:
\[
w(q)B_0(q\mid mq)=I_{FP}(\alpha,m),
\qquad
w(q)B_0(q\mid q/m)=I_{FN}(\alpha,m).
\]
Thus the weighted unshifted power family is exactly multiplicatively
scale-invariant: for fixed relative error \(m\), the weighted loss is
independent of the absolute scale \(q\).

More generally, if one wants to add an explicit rarity premium \(\gamma\ge0\),
one can use
\[
    w_\gamma(q)=q^{\alpha-2-\gamma},
\]
in which case the weighted unshifted loss scales as \(q^{-\gamma}\) times the
relative-error term. The formal scale-invariance results below correspond to
the scale-neutral case \(\gamma=0\).

\subsection{Unshifted FN/FP Asymmetry}
\label{app:unshifted_asymmetry}

Define the unshifted FN-to-FP ratio
\begin{equation}
    R_\alpha(m)
    :=
    \frac{I_{FN}(\alpha,m)}{I_{FP}(\alpha,m)}.
    \label{eq:Ralpha_def}
\end{equation}
The following lemma shows that for \(\alpha>3/2\), underestimation is penalized
more than overestimation, and the asymmetry increases with the multiplicative
error size.

\begin{lemma}
\label{lem:Rmono}
For \(\alpha>3/2\), \(R_\alpha(m)\) is strictly increasing on \((1,\infty)\).
Moreover, \(R_\alpha(m)>1\) for all \(m>1\).
\end{lemma}

\begin{proof}
Starting from the definition of \(I_{FN}\), apply the substitution \(u=1/v\).
Then \(du=-v^{-2}\,dv\), and as \(u\) goes from \(1/m\) to \(1\), \(v\) goes
from \(m\) to \(1\). Hence
\begin{align*}
I_{FN}(\alpha,m)
&=
\int_{1/m}^{1}(1-u)u^{-\alpha}\,du \\
&=
\int_m^1
\left(1-\frac{1}{v}\right)v^\alpha(-v^{-2})\,dv \\
&=
\int_1^m
\left(1-\frac{1}{v}\right)v^{\alpha-2}\,dv \\
&=
\int_1^m (v-1)v^{\alpha-3}\,dv.
\end{align*}
Renaming \(v\) as \(u\),
\[
I_{FN}(\alpha,m)
=
\int_1^m (u-1)u^{\alpha-3}\,du
=
\int_1^m (u-1)u^{-\alpha}u^{2\alpha-3}\,du.
\]
Let
\[
f(u):=(u-1)u^{-\alpha},
\qquad
g(u):=u^{2\alpha-3}.
\]
Then \(f(u)>0\) for \(u>1\), and because \(\alpha>3/2\), \(g(u)\) is strictly
increasing on \((1,\infty)\). Moreover,
\[
I_{FP}(\alpha,m)=\int_1^m f(u)\,du,
\qquad
I_{FN}(\alpha,m)=\int_1^m f(u)g(u)\,du.
\]
Thus
\[
R_\alpha(m)
=
\frac{\int_1^m f(u)g(u)\,du}{\int_1^m f(u)\,du},
\]
so \(R_\alpha(m)\) is the \(f\)-weighted average of the increasing function
\(g\) over \([1,m]\).

To prove strict monotonicity, define
\[
A(m):=\int_1^m f(u)g(u)\,du,
\qquad
C(m):=\int_1^m f(u)\,du.
\]
Then \(R_\alpha(m)=A(m)/C(m)\), and
\[
A'(m)=f(m)g(m),
\qquad
C'(m)=f(m).
\]
Therefore
\begin{align*}
R_\alpha'(m)
&=
\frac{A'(m)C(m)-A(m)C'(m)}{C(m)^2} \\
&=
\frac{f(m)}{C(m)}
\left[
g(m)-\frac{A(m)}{C(m)}
\right] \\
&=
\frac{f(m)}{C(m)}
\left[g(m)-R_\alpha(m)\right].
\end{align*}
Since \(g\) is strictly increasing, its weighted average over \([1,m]\) is
strictly smaller than its endpoint value \(g(m)\). Hence
\[
g(m)-R_\alpha(m)>0.
\]
Also \(f(m)>0\) and \(C(m)>0\) for \(m>1\). Therefore
\[
R_\alpha'(m)>0.
\]
So \(R_\alpha\) is strictly increasing on \((1,\infty)\).

Finally, because \(g(u)=u^{2\alpha-3}>1\) for all \(u>1\), its weighted average
is also greater than \(1\). Hence \(R_\alpha(m)>1\) for all \(m>1\).
\end{proof}

At the boundary \(\alpha=3/2\), we have \(g(u)\equiv1\), so
\(R_{3/2}(m)=1\) for all \(m>1\). Thus \(\alpha=3/2\) is the symmetric point,
while \(\alpha>3/2\) makes false negatives more costly than false positives; the mirrored argument, with \(g\) strictly decreasing, shows that \(\alpha<3/2\) makes false positives more costly.

\section{SPB Loss Guarantees and Parameterization}
\label{sec:spb_guarantees}

\subsection{The Shifted Power Family}
\label{app:shifted_family}

The unshifted power weight \(t^{-\alpha}\) gives exact scale invariance after
dataset weighting, but it is singular at \(t=0\). In particular, for
\(\alpha\ge1\), predicting \(p=0\) can produce an infinite penalty. To make the
loss finite at zero, define the shifted power weight
\begin{equation}
    \omega_\varepsilon(t)
    :=
    (t+\varepsilon)^{-\alpha},
    \qquad
    \alpha>0,
    \quad
    \varepsilon>0.
\end{equation}
The corresponding shifted divergence is
\begin{equation}
    B_\varepsilon(q\mid p)
    =
    \int_q^p (t-q)(t+\varepsilon)^{-\alpha}\,dt.
    \label{eq:shifted_bregman}
\end{equation}
Because \((t+\varepsilon)^{-\alpha}\le \varepsilon^{-\alpha}<\infty\) on
\([0,1]\), the shifted divergence is finite even when \(p=0\).

The price of this shift is that exact multiplicative scale invariance is lost:
after the change of variables \(t=qu\), the ratio
\[
    x:=\frac{\varepsilon}{q}
\]
remains inside the integral. The following results show that the desired
properties are preserved on a finite evaluation box
\[
    q\in[q_{\min},q_{\max}],
    \qquad
    m\in(1,m_{\max}],
\]
provided \(\varepsilon\) is sufficiently small relative to \(q_{\min}\).

\subsection{Bounded Scale Distortion}
\label{app:proof_scale}

\begin{theorem}[Bounded Scale Distortion]
\label{thm:scale}
Fix \(\alpha>0\), \(\eta\in(0,1)\), and \(m_{\max}>1\). Let
\(w(q)=q^{\alpha-2}\). If
\begin{equation}
    \varepsilon
    \le
    \frac{q_{\min}}{m_{\max}}
    \left((1-\eta)^{-1/\alpha}-1\right),
    \label{eq:eps_scale_thm}
\end{equation}
then, uniformly over \(q\in[q_{\min},q_{\max}]\) and
\(m\in(1,m_{\max}]\),
\begin{align}
0
\le
\frac{
I_{FP}(\alpha,m)-w(q)B_\varepsilon(q\mid mq)
}{
I_{FP}(\alpha,m)
}
&\le
\eta,
\label{eq:fp_scale_distortion_bound}
\\[6pt]
0
\le
\frac{
I_{FN}(\alpha,m)-w(q)B_\varepsilon(q\mid q/m)
}{
I_{FN}(\alpha,m)
}
&\le
\eta.
\label{eq:fn_scale_distortion_bound}
\end{align}
\end{theorem}

\begin{proof}
Fix \(q\in[q_{\min},q_{\max}]\) and \(m\in(1,m_{\max}]\), and define
\[
    x:=\frac{\varepsilon}{q}.
\]

First consider the overestimate \(p=mq\). Using \(t=qu\),
\begin{align}
B_\varepsilon(q\mid mq)
&=
\int_q^{mq}(t-q)(t+\varepsilon)^{-\alpha}\,dt \nonumber\\
&=
q^{2-\alpha}
\int_1^m (u-1)(u+x)^{-\alpha}\,du \nonumber\\
&=
q^{2-\alpha}I_{FP,\varepsilon}(\alpha,m;x),
\end{align}
where
\[
    I_{FP,\varepsilon}(\alpha,m;x)
    :=
    \int_1^m (u-1)(u+x)^{-\alpha}\,du.
\]
For \(u\in[1,m]\),
\[
    u\le u+x\le u(1+x).
\]
Since \(\alpha>0\), raising to the power \(-\alpha\) reverses the inequalities:
\[
    u^{-\alpha}
    \ge
    (u+x)^{-\alpha}
    \ge
    u^{-\alpha}(1+x)^{-\alpha}.
\]
Multiplying by \((u-1)\ge0\) and integrating gives
\begin{equation}
    (1+x)^{-\alpha}I_{FP}(\alpha,m)
    \le
    I_{FP,\varepsilon}(\alpha,m;x)
    \le
    I_{FP}(\alpha,m).
    \label{eq:fp_shift_bounds}
\end{equation}
Therefore
\begin{equation}
0
\le
\frac{
I_{FP}(\alpha,m)-I_{FP,\varepsilon}(\alpha,m;x)
}{
I_{FP}(\alpha,m)
}
\le
1-(1+x)^{-\alpha}.
\label{eq:fp_relative_distortion}
\end{equation}

Now consider the underestimate \(p=q/m\). Again using \(t=qu\),
\begin{align}
B_\varepsilon(q\mid q/m)
&=
\int_{q/m}^{q}(q-t)(t+\varepsilon)^{-\alpha}\,dt \nonumber\\
&=
q^{2-\alpha}
\int_{1/m}^1 (1-u)(u+x)^{-\alpha}\,du \nonumber\\
&=
q^{2-\alpha}I_{FN,\varepsilon}(\alpha,m;x),
\end{align}
where
\[
    I_{FN,\varepsilon}(\alpha,m;x)
    :=
    \int_{1/m}^1 (1-u)(u+x)^{-\alpha}\,du.
\]
For \(u\in[1/m,1]\), we have \(1/u\le m\). Hence
\[
    u+x
    =
    u\left(1+\frac{x}{u}\right)
    \le
    u(1+mx),
\]
and also \(u+x\ge u\). Therefore
\[
    u\le u+x\le u(1+mx).
\]
Raising to the power \(-\alpha\) gives
\[
    u^{-\alpha}
    \ge
    (u+x)^{-\alpha}
    \ge
    u^{-\alpha}(1+mx)^{-\alpha}.
\]
Multiplying by \((1-u)\ge0\) and integrating gives
\begin{equation}
    (1+mx)^{-\alpha}I_{FN}(\alpha,m)
    \le
    I_{FN,\varepsilon}(\alpha,m;x)
    \le
    I_{FN}(\alpha,m).
    \label{eq:fn_shift_bounds}
\end{equation}
Therefore
\begin{equation}
0
\le
\frac{
I_{FN}(\alpha,m)-I_{FN,\varepsilon}(\alpha,m;x)
}{
I_{FN}(\alpha,m)
}
\le
1-(1+mx)^{-\alpha}.
\label{eq:fn_relative_distortion}
\end{equation}

Now assume \eqref{eq:eps_scale_thm}. Since \(q\ge q_{\min}\) and
\(m\le m_{\max}\),
\[
    mx
    =
    \frac{m\varepsilon}{q}
    \le
    \frac{m_{\max}\varepsilon}{q_{\min}}
    \le
    (1-\eta)^{-1/\alpha}-1.
\]
Thus
\[
    1+mx
    \le
    (1-\eta)^{-1/\alpha}.
\]
Equivalently,
\[
    (1+mx)^{-\alpha}\ge 1-\eta,
\]
so
\[
    1-(1+mx)^{-\alpha}\le\eta.
\]
This proves the FN distortion bound. Since \(x\le mx\), the same condition also
implies
\[
    1-(1+x)^{-\alpha}\le\eta,
\]
and therefore proves the FP distortion bound.

Finally, because \(w(q)=q^{\alpha-2}\),
\[
    w(q)B_\varepsilon(q\mid mq)
    =
    I_{FP,\varepsilon}(\alpha,m;x),
\]
and
\[
    w(q)B_\varepsilon(q\mid q/m)
    =
    I_{FN,\varepsilon}(\alpha,m;x).
\]
Substituting these identities into the relative distortion bounds proves the
claim.
\end{proof}

\subsection{Asymmetric Sensitivity}
\label{app:proof_asymmetry}

\begin{theorem}[Asymmetric Sensitivity]
\label{thm:asymmetry}
Fix \(\alpha>3/2\) and \(1<m_0\le m_{\max}\). Let
\[
    R_\alpha(m)
    :=
    \frac{I_{FN}(\alpha,m)}{I_{FP}(\alpha,m)}.
\]
If
\begin{equation}
    \varepsilon
    <
    \frac{q_{\min}}{m_{\max}}
    \left[
        R_\alpha(m_0)^{1/\alpha}-1
    \right],
    \label{eq:eps_asym_thm}
\end{equation}
then
\begin{equation}
    B_\varepsilon(q\mid q/m)
    >
    B_\varepsilon(q\mid mq)
\end{equation}
for all \(q\in[q_{\min},q_{\max}]\) and all
\(m\in[m_0,m_{\max}]\).
\end{theorem}

\begin{proof}
Fix \(q\in[q_{\min},q_{\max}]\) and \(m\in[m_0,m_{\max}]\), and define
\[
    x:=\frac{\varepsilon}{q}.
\]

For the overestimate \(p=mq\),
\[
B_\varepsilon(q\mid mq)
=
q^{2-\alpha}I_{FP,\varepsilon}(\alpha,m;x),
\]
where
\[
I_{FP,\varepsilon}(\alpha,m;x)
=
\int_1^m (u-1)(u+x)^{-\alpha}\,du.
\]
Since \(x\ge0\) and \(\alpha>0\),
\[
    (u+x)^{-\alpha}\le u^{-\alpha}.
\]
Therefore
\begin{equation}
    I_{FP,\varepsilon}(\alpha,m;x)
    \le
    I_{FP}(\alpha,m).
    \label{eq:fp_upper_for_asym}
\end{equation}

For the underestimate \(p=q/m\),
\[
B_\varepsilon(q\mid q/m)
=
q^{2-\alpha}I_{FN,\varepsilon}(\alpha,m;x),
\]
where
\[
I_{FN,\varepsilon}(\alpha,m;x)
=
\int_{1/m}^1 (1-u)(u+x)^{-\alpha}\,du.
\]
For \(u\in[1/m,1]\), we have
\[
    u+x
    =
    u\left(1+\frac{x}{u}\right)
    \le
    u(1+mx).
\]
Hence
\[
    (u+x)^{-\alpha}
    \ge
    u^{-\alpha}(1+mx)^{-\alpha}.
\]
Multiplying by \((1-u)\ge0\) and integrating gives
\begin{equation}
    I_{FN,\varepsilon}(\alpha,m;x)
    \ge
    (1+mx)^{-\alpha}I_{FN}(\alpha,m).
    \label{eq:fn_lower_for_asym}
\end{equation}

It is therefore enough to ensure
\[
    (1+mx)^{-\alpha}I_{FN}(\alpha,m)
    >
    I_{FP}(\alpha,m).
\]
Equivalently, using
\[
    R_\alpha(m)
    =
    \frac{I_{FN}(\alpha,m)}{I_{FP}(\alpha,m)},
\]
it is enough to have
\[
    (1+mx)^\alpha<R_\alpha(m),
\]
or
\[
    1+mx<R_\alpha(m)^{1/\alpha}.
\]

By Lemma~\ref{lem:Rmono}, \(R_\alpha(m)\) is strictly increasing in \(m\) for
\(\alpha>3/2\). Since \(m\ge m_0\),
\[
    R_\alpha(m)\ge R_\alpha(m_0).
\]
Also, since \(m\le m_{\max}\) and \(q\ge q_{\min}\),
\[
    mx
    =
    \frac{m\varepsilon}{q}
    \le
    \frac{m_{\max}\varepsilon}{q_{\min}}.
\]
The assumed condition \eqref{eq:eps_asym_thm} implies
\[
    1+\frac{m_{\max}\varepsilon}{q_{\min}}
    <
    R_\alpha(m_0)^{1/\alpha}.
\]
Therefore
\[
    1+mx
    <
    R_\alpha(m_0)^{1/\alpha}
    \le
    R_\alpha(m)^{1/\alpha}.
\]
Hence
\[
    (1+mx)^{-\alpha}I_{FN}(\alpha,m)
    >
    I_{FP}(\alpha,m).
\]
Combining this strict inequality with
\eqref{eq:fp_upper_for_asym} and \eqref{eq:fn_lower_for_asym}, we obtain
\[
    I_{FN,\varepsilon}(\alpha,m;x)
    >
    I_{FP,\varepsilon}(\alpha,m;x).
\]
Multiplying by the positive factor \(q^{2-\alpha}\) gives
\[
    B_\varepsilon(q\mid q/m)
    >
    B_\varepsilon(q\mid mq).
\]
This proves the theorem.
\end{proof}

\subsection{Simultaneous Finite-Error Guarantee}
\label{app:simultaneous_guarantee}

\begin{corollary}[Simultaneous finite-error guarantee]
\label{cor:simultaneous}
Fix \(\alpha>3/2\), \(\eta\in(0,1)\), and
\[
    1<m_0\le m\le m_{\max}.
\]
Let
\[
    R_\alpha(m)
    :=
    \frac{I_{FN}(\alpha,m)}{I_{FP}(\alpha,m)}.
\]
If
\begin{equation}
    0\le\varepsilon
    <
    \min\left\{
    \frac{q_{\min}}{m_{\max}}
    \left((1-\eta)^{-1/\alpha}-1\right),
    \;
    \frac{q_{\min}}{m_{\max}}
    \left[
        R_\alpha(m_0)^{1/\alpha}-1
    \right]
    \right\},
    \label{eq:eps_simultaneous}
\end{equation}
then the shifted weighted loss simultaneously satisfies:

\begin{enumerate}
    \item approximate scale invariance within relative tolerance \(\eta\):
    \[
    0
    \le
    \frac{
    I_{FP}(\alpha,m)-w(q)B_\varepsilon(q\mid mq)
    }{
    I_{FP}(\alpha,m)
    }
    \le
    \eta,
    \]
    and
    \[
    0
    \le
    \frac{
    I_{FN}(\alpha,m)-w(q)B_\varepsilon(q\mid q/m)
    }{
    I_{FN}(\alpha,m)
    }
    \le
    \eta;
    \]

    \item finite-error asymmetric sensitivity:
    \[
        B_\varepsilon(q\mid q/m)>B_\varepsilon(q\mid mq).
    \]
\end{enumerate}
Both guarantees hold uniformly for
\(q\in[q_{\min},q_{\max}]\) and \(m\in[m_0,m_{\max}]\).
\end{corollary}

\begin{proof}
The first term in the minimum in \eqref{eq:eps_simultaneous} is exactly the
sufficient condition of Theorem~\ref{thm:scale}. The second term is exactly the
sufficient condition of Theorem~\ref{thm:asymmetry}. Therefore both conclusions
hold simultaneously.
\end{proof}

\subsection{Interpretation of the Shift Bound}\label{sec:choose_eps}
\label{app:shift_bound_interpretation}

Setting \(\varepsilon=0\) gives exact multiplicative scale invariance and, for
\(\alpha>3/2\), strict FN-over-FP asymmetry. However, the unshifted loss can be
infinite when a model predicts \(p=0\). Introducing \(\varepsilon>0\) removes
this singularity, but it also introduces an additive floor into the local
weight \((t+\varepsilon)^{-\alpha}\). If \(\varepsilon\) is too large relative
to \(q_{\min}\), this floor overwhelms the relevant probability scale and
destroys both approximate scale invariance and asymmetric sensitivity.

The bounds above quantify the permissible size of \(\varepsilon\). They are
sufficient rather than necessary, and are intentionally conservative. A simple
practical choice is to set
\[
    \varepsilon
    =
    c\frac{q_{\min}}{m_{\max}},
    \qquad
    c\ll1,
\]
and then verify that this value lies below the theorem-specific upper bounds.

\subsection{Parameter Guidance}
\label{sec:param_guidance_table}

Table~\ref{tab:alpha_gamma_use_cases} summarizes suggested
\((\alpha,\gamma)\) settings for the evaluation use cases referenced in
Section~\ref{sec:param_guidance}. The parameter \(\gamma\) controls any
additional dataset-level rarity premium, while \(\alpha\) controls the
asymmetry between reciprocal overestimation and underestimation. We now describe
a more fine-grained way to choose \(\alpha\) by specifying a desired worst-case
FN/FP ratio.

\subsubsection{Choosing \texorpdfstring{\(\alpha\)}{alpha} via a Worst-Case FN/FP Ratio}
\label{sec:choose_alpha}

A natural way to calibrate the asymmetry parameter \(\alpha\) is to specify a
desired lower bound \(r_\star>1\) on the cost of reciprocal underestimation
relative to reciprocal overestimation. In the unshifted geometry,
\[
    \frac{B_0(q\mid q/m)}{B_0(q\mid mq)}
    =
    \frac{I_{FN}(\alpha,m)}{I_{FP}(\alpha,m)}
    =
    R_\alpha(m),
\]
so this ratio is independent of the base probability \(q\). The same is true
after multiplying both losses by any dataset-level weight depending only on
\(q\), since the weight cancels in the ratio.

Such a guarantee cannot hold uniformly as \(m\downarrow1\). Indeed, for every
fixed \(\alpha\),
\[
    \lim_{m\downarrow1} R_\alpha(m)=1.
\]
Thus no finite \(\alpha\) can guarantee \(R_\alpha(m)\ge r_\star>1\) for
arbitrarily small multiplicative errors. A meaningful worst-case calibration
must therefore specify a minimum multiplicative error \(m_0>1\).

Fix \(m_0>1\) and \(m_{\max}\ge m_0\). We require
\[
    R_\alpha(m)\ge r_\star
    \qquad
    \forall m\in[m_0,m_{\max}].
\]
For \(r_\star>1\), any feasible choice must have \(\alpha>3/2\). By
Lemma~\ref{lem:Rmono}, \(R_\alpha(m)\) is then increasing in \(m\), so the
worst case occurs at \(m=m_0\). Therefore, in the unshifted geometry, the
uniform calibration condition over \(m\in[m_0,m_{\max}]\) is equivalent to
\[
    R_\alpha(m_0)\ge r_\star.
    \label{eq:alpha_calibration_condition}
\]

Using the identity \(I_{FN}(\alpha,m)=I_{FP}(3-\alpha,m)\), we can write
\[
    R_\alpha(m)
    =
    \frac{I_{FP}(3-\alpha,m)}{I_{FP}(\alpha,m)}.
\]
Here
\[
I_{FP}(\beta,m)
=
\int_1^m (u-1)u^{-\beta}\,du
=
\begin{cases}
\displaystyle
\frac{m^{2-\beta}-1}{2-\beta}
-
\frac{m^{1-\beta}-1}{1-\beta},
& \beta\notin\{1,2\},\\[1.25em]
m-1-\log m,
& \beta=1,\\[0.5em]
\log m+\frac1m-1,
& \beta=2.
\end{cases}
\]

For fixed \(m_0>1\), the map \(\alpha\mapsto R_\alpha(m_0)\) is strictly
increasing. Indeed, defining
\[
    F(\beta,m)=I_{FP}(\beta,m),
\]
we have
\[
    \partial_\beta F(\beta,m)
    =
    -\int_1^m (u-1)u^{-\beta}\log u\,du <0.
\]
Thus \(F(\beta,m)\) is strictly decreasing in \(\beta\), and
\[
    R_\alpha(m_0)
    =
    \frac{F(3-\alpha,m_0)}{F(\alpha,m_0)}
\]
is strictly increasing in \(\alpha\). Moreover,
\[
    R_{3/2}(m_0)=1,
    \qquad
    \lim_{\alpha\to\infty}R_\alpha(m_0)=\infty.
\]
Hence for every \(r_\star>1\) there is a unique threshold
\[
    \alpha_\star=\alpha_\star(m_0,r_\star)>3/2
\]
satisfying
\[
    R_{\alpha_\star}(m_0)=r_\star.
\]
Choosing \(\alpha\ge\alpha_\star\) guarantees
\[
    \min_{m\in[m_0,m_{\max}]}R_\alpha(m)\ge r_\star
\]
for the unshifted loss.

\paragraph{Computing \(\alpha_\star\).}
The threshold \(\alpha_\star\) generally has no closed-form inverse, but it can
be computed by bisection because \(\alpha\mapsto R_\alpha(m_0)\) is continuous
and strictly increasing. A practical procedure is:

\begin{enumerate}
    \item Choose the minimum relevant multiplicative error \(m_0>1\), the
    desired worst-case ratio \(r_\star>1\), and a numerical tolerance
    \(\tau>0\).

    \item Set
    $\alpha_{\mathrm{lo}}=\frac32.$
    Then
   $R_{\alpha_{\mathrm{lo}}}(m_0)=1<r_\star.$
    \item Choose an initial upper bracket
    \(\alpha_{\mathrm{hi}}>3/2\). Increase it, for example by repeatedly
    doubling, $\alpha_{\mathrm{hi}}\leftarrow 2\alpha_{\mathrm{hi}},$
    until $R_{\alpha_{\mathrm{hi}}}(m_0)\ge r_\star.$
    \item While $
        \alpha_{\mathrm{hi}}-\alpha_{\mathrm{lo}}>\tau,$
    set $ \alpha_{\mathrm{mid}}
        =
        \frac{\alpha_{\mathrm{lo}}+\alpha_{\mathrm{hi}}}{2}.$
    If $
        R_{\alpha_{\mathrm{mid}}}(m_0)<r_\star,$
    set $
        \alpha_{\mathrm{lo}}\leftarrow\alpha_{\mathrm{mid}}.$
    Otherwise set $
        \alpha_{\mathrm{hi}}\leftarrow\alpha_{\mathrm{mid}}.$

    \item Return \(\alpha_{\mathrm{hi}}\). This gives a conservative numerical
    choice satisfying $
        R_{\alpha_{\mathrm{hi}}}(m_0)\ge r_\star.$
\end{enumerate}

\paragraph{Effect of the shift.}
For the shifted loss, the exact unshifted ratio no longer holds. However, the
proof of Theorem~\ref{thm:asymmetry} gives the conservative lower bound
\[
    \frac{B_\varepsilon(q\mid q/m)}
         {B_\varepsilon(q\mid mq)}
    \ge
    (1+m\varepsilon/q)^{-\alpha}R_\alpha(m).
\]
Therefore, uniformly over \(q\in[q_{\min},q_{\max}]\) and
\(m\in[m_0,m_{\max}]\),
\[
    \frac{B_\varepsilon(q\mid q/m)}
         {B_\varepsilon(q\mid mq)}
    \ge
    \left(1+\frac{m_{\max}\varepsilon}{q_{\min}}\right)^{-\alpha}
    R_\alpha(m_0).
\]
Thus a sufficient condition for the shifted loss to achieve worst-case ratio at
least \(r_\star\) is
\[
    \left(1+\frac{m_{\max}\varepsilon}{q_{\min}}\right)^{-\alpha}
    R_\alpha(m_0)
    \ge r_\star.
    \label{eq:shifted_alpha_condition}
\]
Equivalently, after choosing \(\alpha\), it is sufficient to choose
\[
    \varepsilon
    \le
    \frac{q_{\min}}{m_{\max}}
    \left[
        \left(\frac{R_\alpha(m_0)}{r_\star}\right)^{1/\alpha}
        -1
    \right].
    \label{eq:eps_ratio_rstar}
\]
If a strict ratio \(>r_\star\) is desired, use a strict inequality. This bound
is meaningful for positive \(\varepsilon\) only when
\[
    R_\alpha(m_0)>r_\star.
\]
Consequently, if one intends to use a positive shift \(\varepsilon>0\), one
should choose \(\alpha\) with some margin above the unshifted threshold
\(\alpha_\star\), rather than choosing exactly \(\alpha=\alpha_\star\).
Theorem~\ref{thm:asymmetry} corresponds to the special case \(r_\star=1\).

\begin{table}[!htbp]
\centering
\small
\begin{tabular}{p{0.34\linewidth} p{0.44\linewidth} c c}
\toprule
\textbf{Use case} & \textbf{What you want the metric to emphasize} & $\boldsymbol{\alpha}$ & $\boldsymbol{\gamma}$ \\
\midrule
Unbiased method benchmarking
& Scale-neutral comparison across a wide ground truth range where FN and FP treated approximately symmetrically.
& $1.5$ & $0$ \\
\midrule
Product safety oversight
& Mild FN aversion (similar to Itakura--Saito Loss) but avoid excessive false alarms; performance should still matter across the full ground truth range.
& $1.7$--$2.2$ & $0$ \\
\midrule
Deep-tail stress testing (tail-focused but not fully worst-case)
& Put substantially more weight on the smallest $q$ examples (rarest events), while still comparing methods on estimation quality (not just conservatism).
& $1.5$--$2.5$ & $0.5$--$1$ \\
\midrule
Catastrophic-risk audit
& Extreme-tail priority and willing to tolerate many false alarms if it reduces underestimation in the deepest tail.
& $3$--$6$ & $1$--$2$ \\
\bottomrule
\end{tabular}
\caption{Suggested shifted-power metric parameters by evaluation use case. Here $\alpha$ controls FN-vs-FP asymmetry (higher $\alpha$ penalizes underestimation more), and $\gamma$ controls dataset-level rarity premium (higher $\gamma$ upweights smaller ground-truth probabilities $q$).}
\label{tab:alpha_gamma_use_cases}
\end{table}

\subsection{Computation: Closed Form }\label{sec:computation}

\textbf{Closed-form $B_\varepsilon(q\mid p)$.}
For $\alpha\neq 1,2$, the shifted power family admits the closed form
\begin{equation}
B_\varepsilon(q\mid p)=
\frac{(p+\varepsilon)^{2-\alpha}-(q+\varepsilon)^{2-\alpha}}{2-\alpha}
-\frac{(\varepsilon+q)\left[(p+\varepsilon)^{1-\alpha}-(q+\varepsilon)^{1-\alpha}\right]}{1-\alpha}.
\label{eq:Beps_closed}
\end{equation}
For $\alpha=1$:
\begin{equation}
B_\varepsilon(q\mid p)=(p-q)-(\varepsilon+q)\log\!\left(\frac{p+\varepsilon}{q+\varepsilon}\right).
\label{eq:Beps_a1}
\end{equation}
For $\alpha=2$:
\begin{equation}
B_\varepsilon(q\mid p)=
\log\!\left(\frac{p+\varepsilon}{q+\varepsilon}\right)
+(\varepsilon+q)\left(\frac{1}{p+\varepsilon}-\frac{1}{q+\varepsilon}\right).
\label{eq:Beps_a2}
\end{equation}

\subsection{Connections to Standard Losses}
\label{app:standard_loss_connections}

For the unshifted family \(\omega_0(t)=t^{-\alpha}\), the divergence is
\[
B_0(q\mid p)=\int_q^p (t-q)t^{-\alpha}\,dt.
\]

For \(\alpha=0\),
\[
B_0(q\mid p)=\int_q^p (t-q)\,dt=\frac{(p-q)^2}{2},
\]
so the family recovers squared error up to a constant factor
\(1/2\).

For \(\alpha=2\),
\[
B_0(q\mid p)
=
\int_q^p \left(\frac{1}{t}-\frac{q}{t^2}\right)\,dt
=
\log\frac{p}{q}+\frac{q}{p}-1
=
\frac{q}{p}-\log\frac{q}{p}-1,
\]
which is exactly the Itakura--Saito divergence \(D_{\mathrm{IS}}(q,p)\).

\subsection{Empirical Verification of SPB Scaling Behavior}
\label{sec:metric_verification}

Figure~\ref{fig:scale invariance across metrics} contrasts the scale behavior of the metrics for a fixed relative error across the ground-truth range.

\begin{figure}
    \centering
    \includegraphics[width=0.5\linewidth]{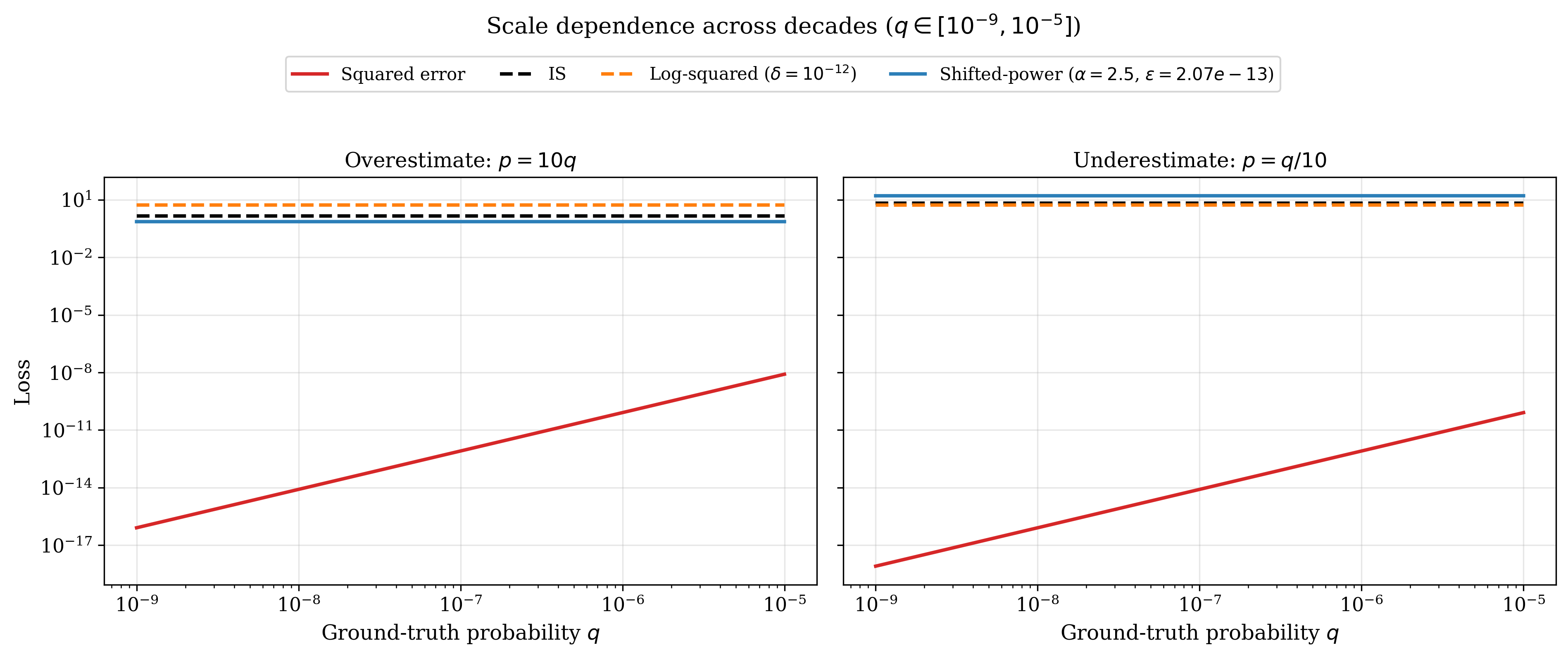}
    \caption{Scale dependence of each metric for a fixed relative error (overestimate $p=10q$, left; underestimate $p=q/10$, right) across ground-truth probabilities $q\in[10^{-9},10^{-5}]$. Squared error varies by orders of magnitude with $q$, while IS Loss, log-space squared error, and the weighted shifted-power loss are flat (scale invariant).}
    \label{fig:scale invariance across metrics}
\end{figure}

\section{Pseudocode}
\label{sec:pseudocode}

Algorithm~\ref{alg:ga-amls} gives the full GA-AMLS procedure described in Section~\ref{sec:method}.

\begin{algorithm}
    \caption{GA-AMLS with MALA moves and a Student-$t$ prior}
    \label{alg:ga-amls}
    \begin{algorithmic}[1]
        \Require Input distribution $\mathcal{D}$, model $M$, target token $t$, population size $N$, quantile level $\rho$, target score $\tau$, MALA step size $h$, MCMC steps $T$
        \Ensure Probability estimate $\widehat{P}$
        \State Draw calibration inputs from \(\mathcal{D}\) and compute raw activations \(\{a_j\}_{j=1}^{n_{\mathrm{cal}}}\) at the chosen site.
        \State Estimate \(\mu,\Sigma\), choose \(A\) such that \(AA^\top=\Sigma\), and transform activations by
        \[
        u_j \leftarrow (a_j-\mu)A^{-\top}.
        \]
        \State Fit the prior in whitened coordinates as an independent Student-\(t\) distribution
        \[
        \pi(u)=\prod_{j=1}^{d} t_{\nu=5}(u_j;0,1).
        \]
        \State Define the score \(s(u)\) by inserting the unwhitened activation \(a(u)=uA^\top+\mu\) into the model.        \State Initialize particles $\{u_i\}_{i=1}^{N}$ by sampling from the whitened calibration activations.
        \State $k\gets 0$, $L_0\gets -\infty$, $\widehat{P}\gets 1$
        \While{$L_k < \tau$}
            \State Compute scores $\{s(u_i)\}_{i=1}^{N}$.
            \State Set the next level
            \[
                L_{k+1}\gets \min\left\{\tau,\operatorname{Quantile}_{\rho}\left(\{s(u_i)\}_{i=1}^{N}\right)\right\}.
            \]
            \State Estimate the conditional survival probability
            \[
                \widehat{p}_k \gets \frac{1}{N}\sum_{i=1}^{N}\mathbf{1}\{s(u_i)\ge L_{k+1}\},
                \qquad
                \widehat{P}\gets \widehat{P}\widehat{p}_k.
            \]
            \State Keep the survivors $\mathcal{S}\gets\{u_i:s(u_i)\ge L_{k+1}\}$ and resample with replacement from $\mathcal{S}$ until the population again has size $N$.
            \For{$r=1,\dots,T$}
                \For{each particle $u_i$}
                    \State Compute the prior score gradient
                    \[
                        g_i \gets \nabla_u \log \pi(u_i).
                    \]
                    \State Propose a MALA move
                    \[
                        u_i' \gets u_i + h g_i + \sqrt{2h}\,\xi_i,
                        \qquad \xi_i\sim\mathcal{N}(0,I).
                    \]
                    \If{$s(u_i') < L_{k+1}$}
                        \State Reject the proposal.
                    \Else
                        \State Compute $g_i'\gets \nabla_u\log\pi(u_i')$ and proposal densities
                        \[
                            q(u_i'\mid u_i)=\mathcal{N}(u_i';\,u_i+h g_i,\,2hI),
                            \qquad
                            q(u_i\mid u_i')=\mathcal{N}(u_i;\,u_i'+h g_i',\,2hI).
                        \]
                        \State Accept $u_i'$ with probability
                        \[
                            a_i=\min\left\{1,\,
                            \frac{\pi(u_i')q(u_i\mid u_i')}{\pi(u_i)q(u_i'\mid u_i)}
                            \right\}.
                        \]
                    \EndIf
                \EndFor
                \State Adapt $h$ during burn-in toward the MALA target acceptance rate.
            \EndFor
            \State $k\gets k+1$
        \EndWhile
        \State \Return $\widehat{P}$
    \end{algorithmic}
    \end{algorithm}

\section{Experimental and Implementation Details}
\label{sec:exp_details}

\subsection{Hyperparameters}
\label{sec:hyperparameters}

\paragraph{MAPLA tuning candidate.}
In the tuning sweep we also considered MAPLA, the Metropolis-adjusted
Preconditioned Langevin Algorithm \citep{srinivasan_high-accuracy_2025}.
MAPLA generalizes MALA by replacing the isotropic proposal covariance with a
position-dependent metric. For a target density $\pi$, its proposal has the form
\[
    u'
    \sim
    \mathcal N\!\left(
        u + h G(u)^{-1}\nabla_u \log \pi(u),
        2h G(u)^{-1}
    \right),
\]
followed by the usual Metropolis--Hastings correction. In our AMLS setting this
kernel is applied to the constrained target
\[
    \pi_k(u) \propto \pi(u)\mathbf 1\{s(u)\ge L_k\}.
\]
We used MAPLA only as a hyperparameter-tuning candidate: although it can reduce
rejections near level-set boundaries by adapting proposal geometry, its extra
metric computations did not yield a sufficient improvement over MALA under a
matched compute budget. Therefore all main GA-AMLS results use MALA.

\textbf{Tuning protocol.}
To prevent overfitting, our method was only run on the first four distributions during development, and architectural choices were finalized before testing on the last four distributions. We chose the MH kernel and the prior as those which minimized squared error loss and IS Loss on 100 randomly chosen tokens with ground-truth probabilities in the range $[10^{-5},10^{-3}]$; probabilities in this range are directly estimable by naive Monte Carlo with a modest budget, so tuning does not require privileged knowledge of the extreme tail. The selected prior and kernel were consistent across metrics, distributions $\mathcal{D}$, and model sizes. Both MALA and MAPLA outperformed the random-walk kernel (RWMH), and the Student-$t$ prior outperformed a Gaussian prior; since MAPLA's higher computational cost did not yield a commensurate performance gain over MALA, we chose the MALA kernel. To ensure a fair comparison across MH kernels, we equalized the total computational budget per rejuvenation step: since MALA and MAPLA require gradient computations, we allocate 1500 MH steps to RWMH and 750 steps to the gradient-based kernels. 

\textbf{GA-AMLS configuration.}
All GA-AMLS estimates are from a single run per (token, distribution, model). We use a population of \(N=2000\) particles, threshold
quantile \(\rho=0.6\), \(T=900\) MALA steps per level with a burn-in of
\(T_{\mathrm{burn}}=126\) steps and initial step size \(h_0=10^{-3}\), and a
diagonal Student-\(t\) prior with \(\nu=5\) degrees of freedom fit on
\(n_{\mathrm{cal}}=2^{16}=65{,}536\) calibration activations drawn from
\(\mathcal{D}\). The calibration activations are generated once per
(model, $\mathcal{D}$) pair and reused across target tokens.

\textbf{Baselines.} For the importance-sampling baselines ITGIS and MHIS we used the temperatures tuned by \citet{wu2025estimatingprobabilitiesrareoutputs}. We identified and fixed a numerical-tolerance bug in the public QLD implementation: in the randomized projection routine for finding the shortest accepting vector, the tolerance was meant to ignore only slightly negative constraint violations \(P x+b\ge -\texttt{tol}\), but the original code used \(P x+b<\texttt{tol}\), incorrectly marking feasible near-boundary constraints as violated; all QLD results in this paper use the corrected implementation, which improves QLD's performance relative to the numbers reported by \citet{wu2025estimatingprobabilitiesrareoutputs}.

\textbf{Ablation studies.} Figures~\ref{fig:hparam_monitoring}--\ref{fig:hparam_audit} report the per-distribution SPB Loss for every kernel--prior--hyperparameter configuration explored during tuning, under each of the four evaluation profiles of Table~\ref{tab:alpha_gamma_use_cases}. Two patterns hold consistently across distributions and model sizes: (i) the MALA kernel outperforms RWMH despite using half as many MH proposal steps, and (ii) the Student-$t$ prior outperforms a Gaussian prior.

\begin{figure}
    \centering
    \includegraphics[width=1\linewidth]{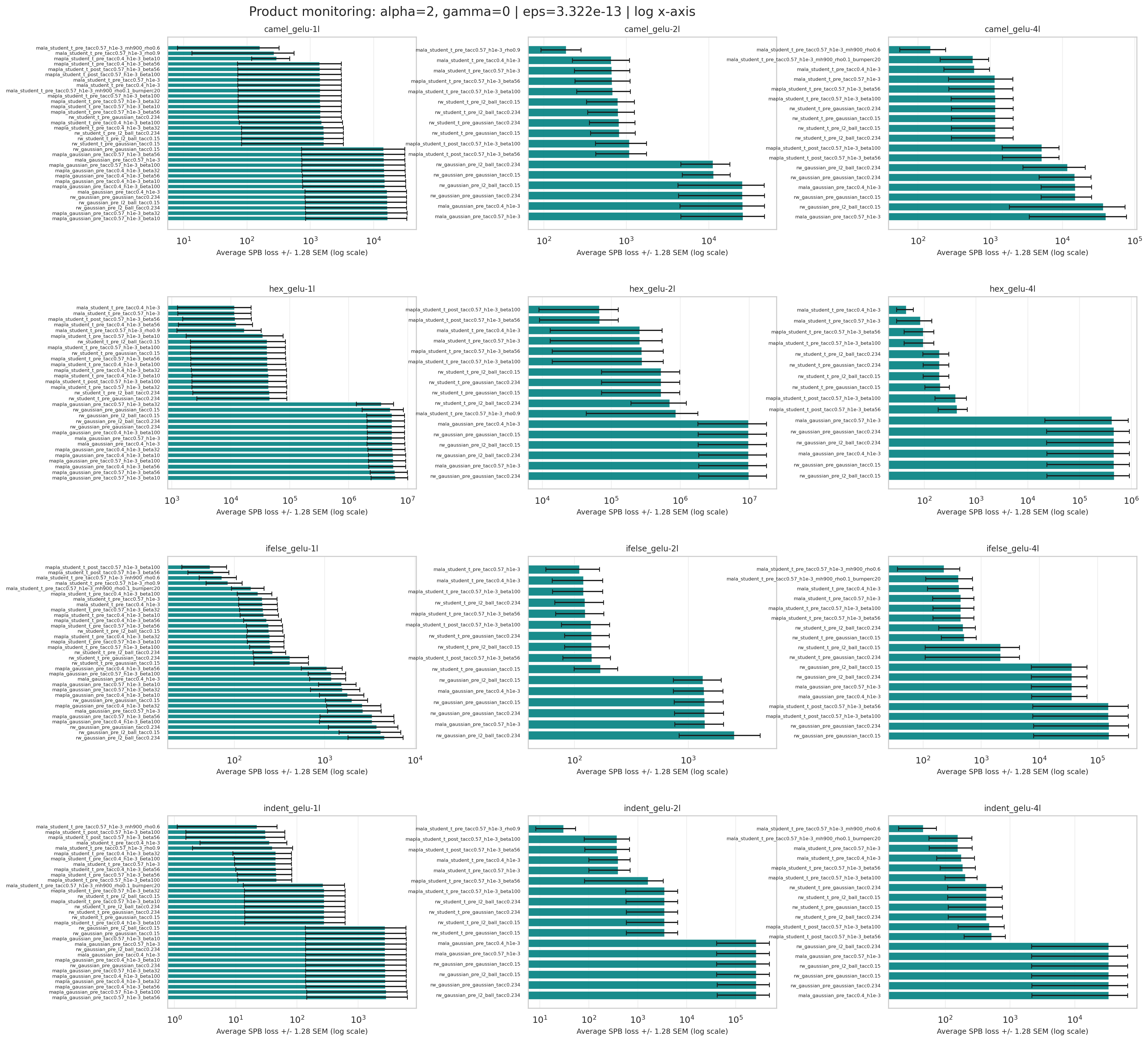}
    \caption{Hyperparameter sweep under the product-monitoring evaluation profile ($\alpha=2$, $\gamma=0$). Bars show average SPB Loss ($\pm 1.28$ SEM (an approximate 80\% normal-approximation confidence interval for the mean), log scale) for each kernel--prior--hyperparameter configuration, per input distribution (rows) and model size (columns). Configuration names encode kernel, prior, activation site, and tuning parameters.}
    \label{fig:hparam_monitoring}
\end{figure}

\begin{figure}
    \centering
    \includegraphics[width=1\linewidth]{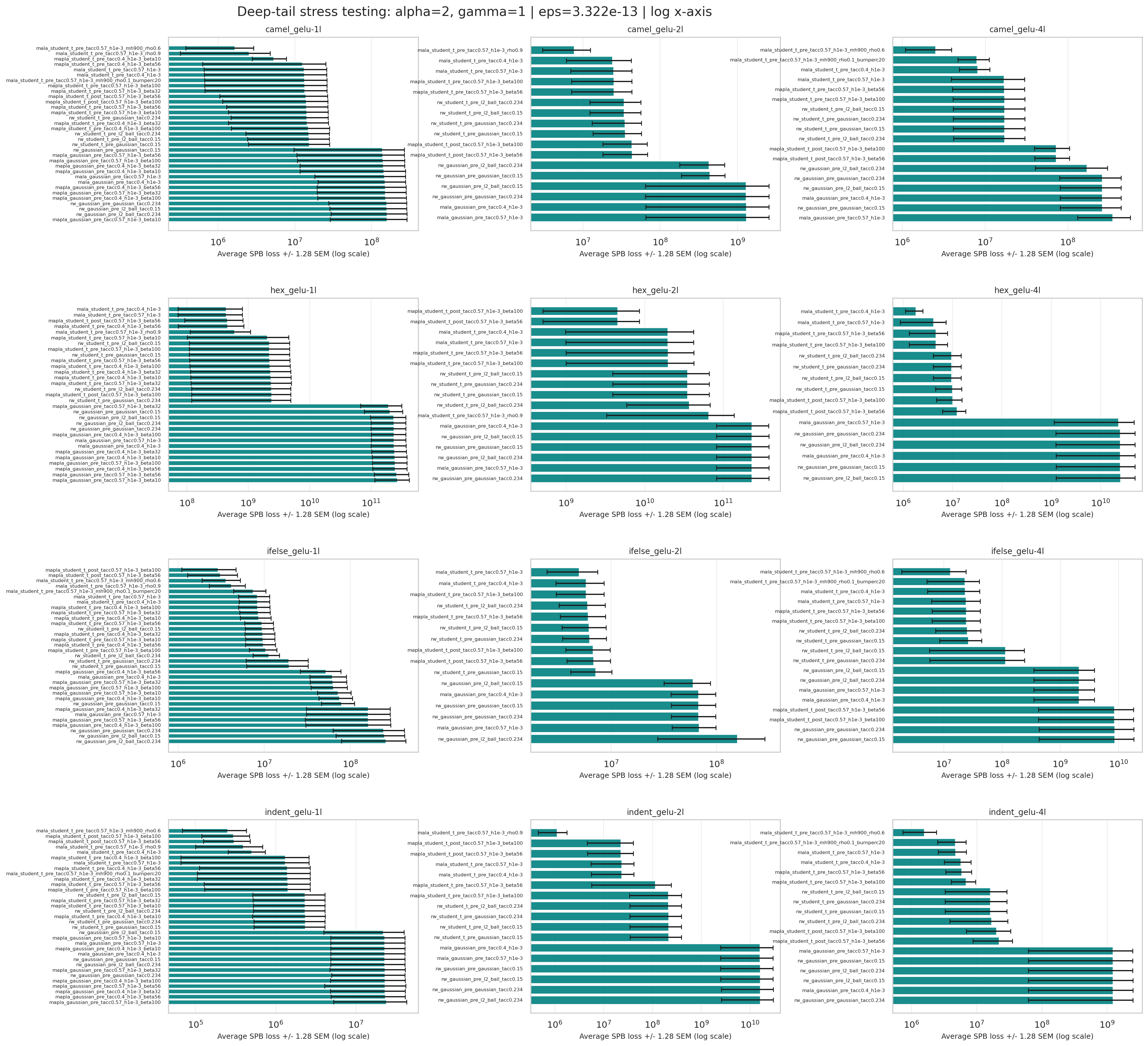}
    \caption{Hyperparameter sweep under the deep-tail stress-testing profile ($\alpha=2$, $\gamma=1$); format as in Figure~\ref{fig:hparam_monitoring}.}
    \label{fig:hparam_stress}
\end{figure}

\begin{figure}
    \centering
    \includegraphics[width=1\linewidth]{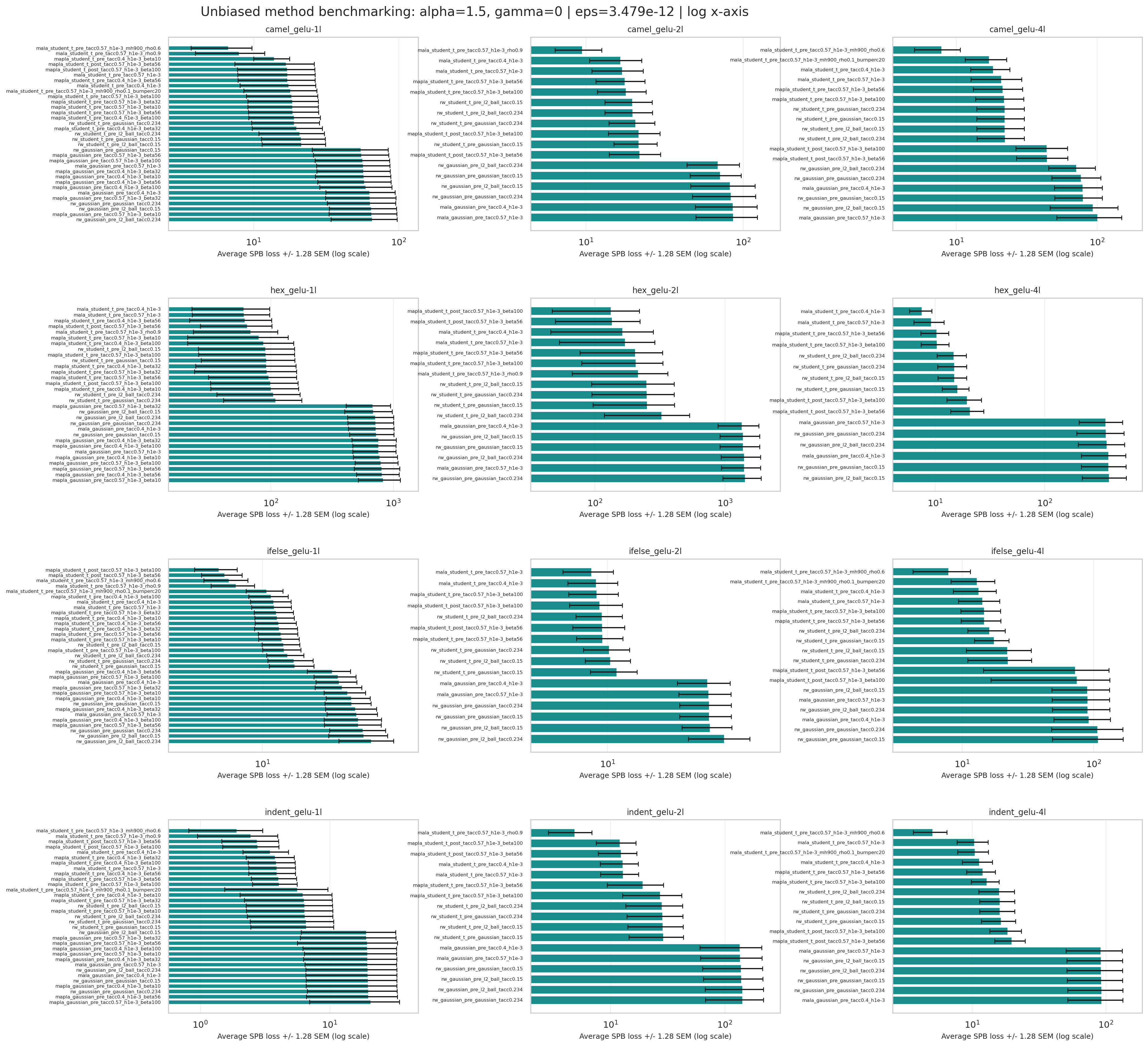}
    \caption{Hyperparameter sweep under the unbiased method benchmarking profile ($\alpha=1.5$, $\gamma=0$); format as in Figure~\ref{fig:hparam_monitoring}.}
    \label{fig:hparam_benchmarking}
\end{figure}

\begin{figure}
    \centering
    \includegraphics[width=1\linewidth]{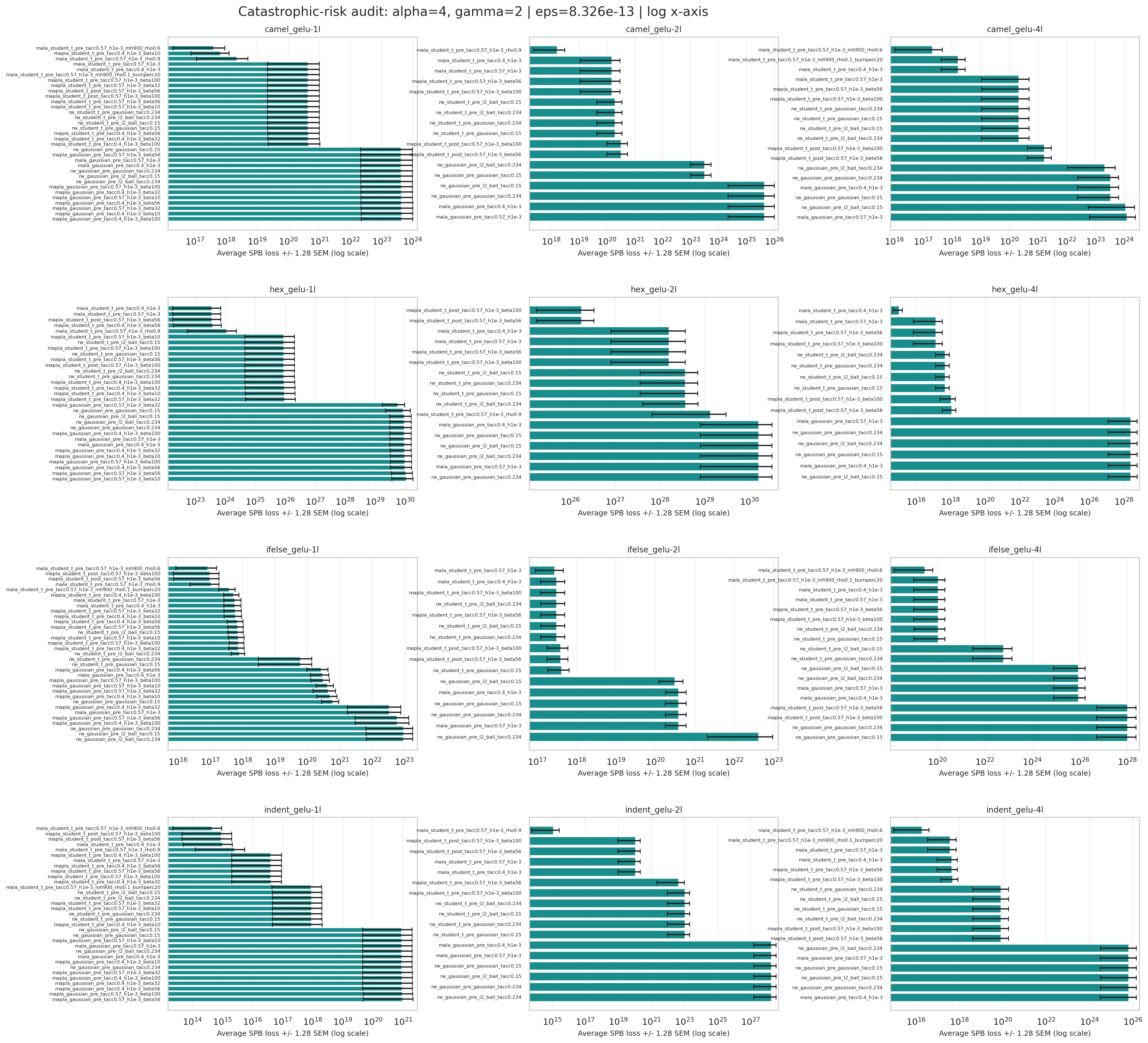}
    \caption{Hyperparameter sweep under the catastrophic-risk audit profile ($\alpha=4$, $\gamma=2$); format as in Figure~\ref{fig:hparam_monitoring}.}
    \label{fig:hparam_audit}
\end{figure}

\subsection{Activation Whitening Diagnostic}
\label{app:activation_whitening}

Figure~\ref{fig:anisot} plots the held-out marginal variance of each pre-LayerNorm, pre-unembedding activation dimension, sorted by variance, before and after whitening; the whitening transform used by GA-AMLS substantially flattens this marginal variance profile.

\begin{figure}
    \centering
    \includegraphics[width=1\linewidth]{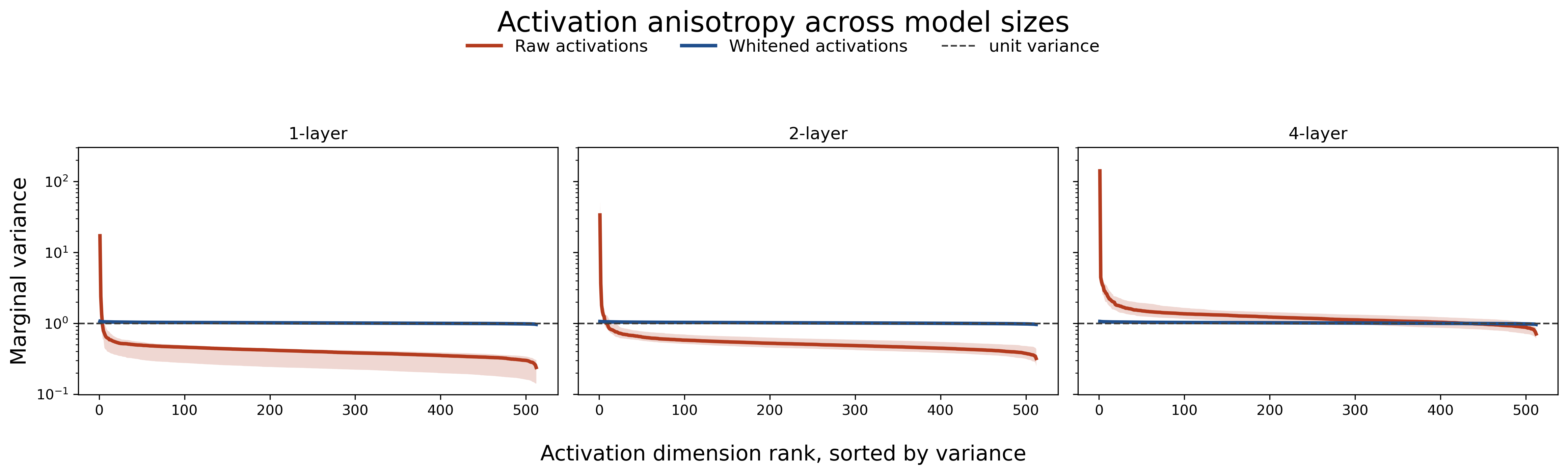}
\caption{
\textbf{Marginal activation variances before and after whitening.}
Let \(a_i \in \mathbb{R}^d\) denote held-out pre-LayerNorm, pre-unembedding
activations. For each model size and input distribution, we compute the empirical
marginal variance \( \mathrm{Var}_i(a_{ij}) \) of every activation dimension \(j\),
and repeat the calculation after applying the whitening transform used by GA-AMLS.
For visualization, dimensions are sorted from largest to smallest marginal variance
within each distribution; the x-axis is therefore the sorted dimension rank. Each panel reports the median sorted curve across the
eight input distributions, with shaded bands showing the interquartile range. Raw
activations exhibit strong marginal anisotropy, while whitening produces a much flatter
variance profile, making isotropic MALA proposals in whitened coordinates more
reasonable.
}
    \label{fig:anisot}
\end{figure}

\section{Scatterplots}
\label{sec:all_scatterplots}

Figures~\ref{fig:scatter_1l}--\ref{fig:scatter_4l} show estimates against ground truth for all four methods on all 8 input distributions, for each model size, complementing the two examples in Figure~\ref{fig:2dist_eg}.

\begin{figure}
    \centering
    \includegraphics[width=1\linewidth]{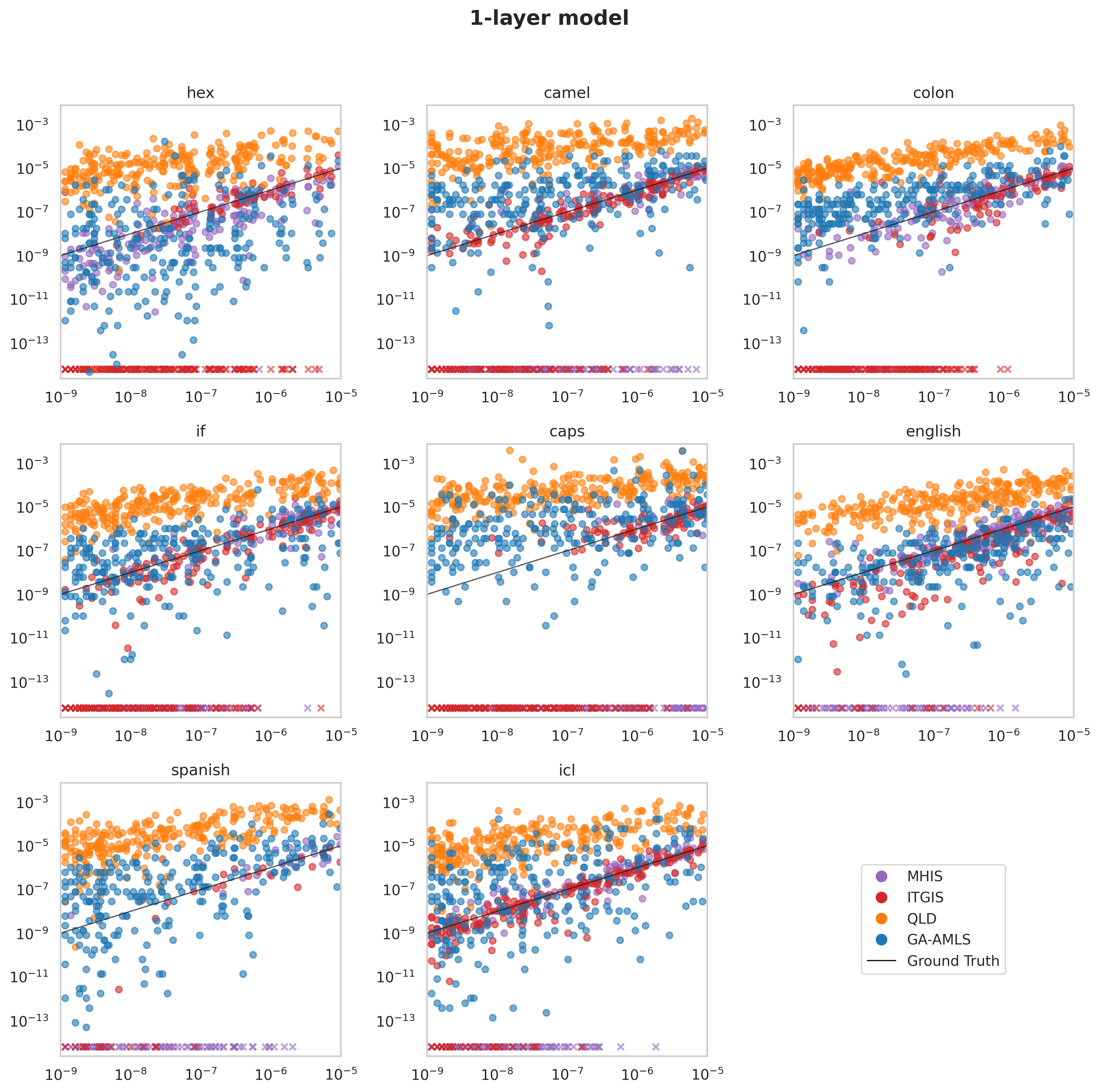}
    \caption{Estimates vs.\ ground truth for all methods across all 8 input distributions (1-layer model). MHIS/ITGIS zero estimates are plotted as crosses at the bottom of each panel for visibility.}
    \label{fig:scatter_1l}
\end{figure}

\begin{figure}
    \centering
    \includegraphics[width=1\linewidth]{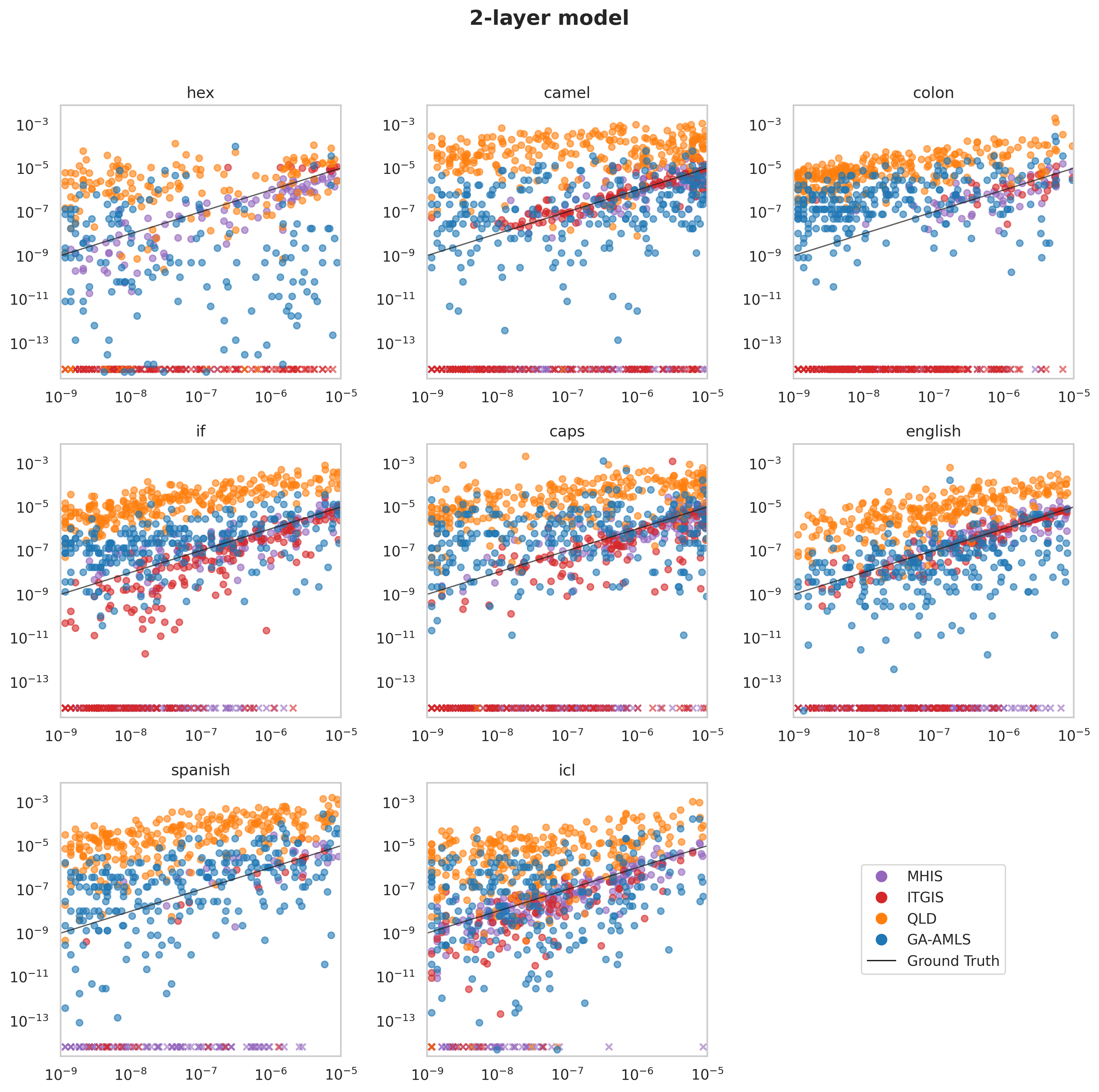}
    \caption{As Figure~\ref{fig:scatter_1l}, for the 2-layer model.}
    \label{fig:scatter_2l}
\end{figure}

\begin{figure}
    \centering
    \includegraphics[width=1\linewidth]{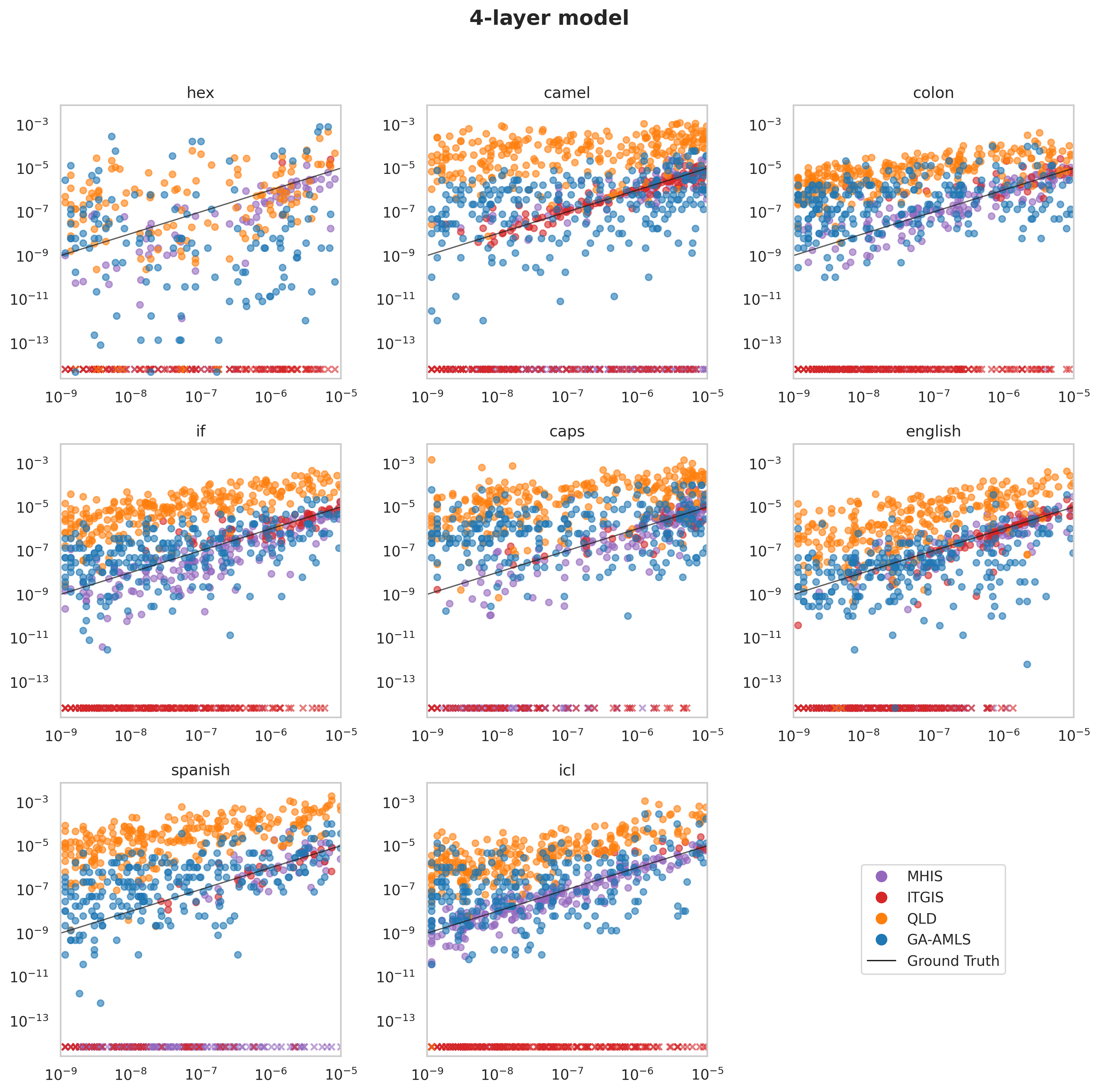}
    \caption{As Figure~\ref{fig:scatter_1l}, for the 4-layer model.}
    \label{fig:scatter_4l}
\end{figure}

\section{All Methods Loss by Distributions}
\label{sec:loss_by_dist}

Figure~\ref{fig:SPB-dist} breaks the aggregate SPB results down by input distribution, Figure~\ref{fig:Layers-SPB} reports the aggregate SPB Loss across model sizes for all four $(\alpha,\gamma)$ evaluation profiles, and Tables~\ref{tab:spb_alpha_1p5_gamma_0p0} and~\ref{tab:spb_alpha_2p0_gamma_0p0} report per-distribution SPB Loss.

\begin{figure}[!htbp]
    \centering
    \includegraphics[width=1\linewidth]{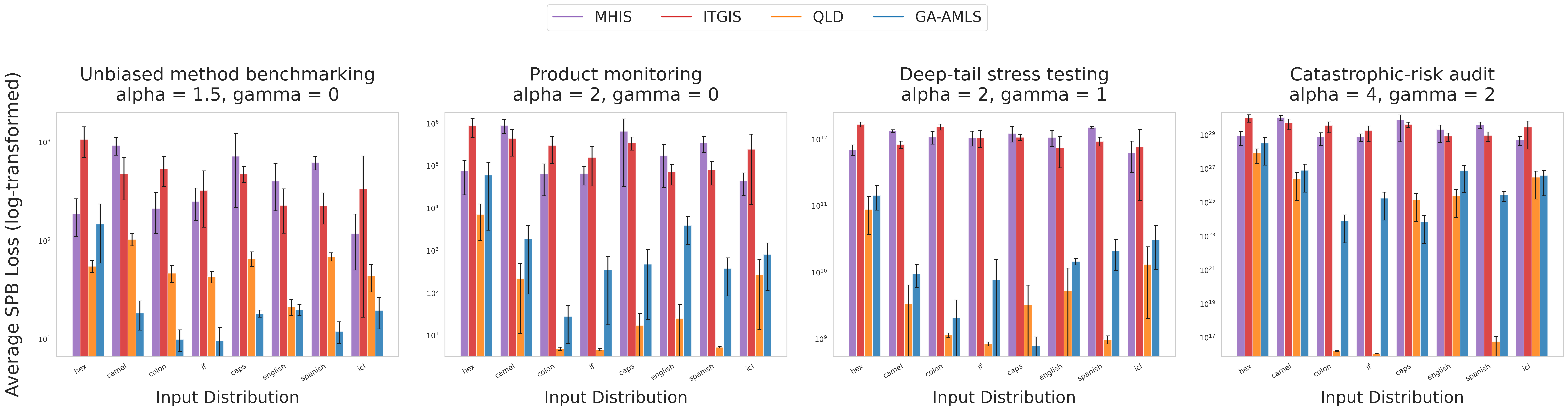}
   \caption{\textbf{SPB Loss by input distribution and evaluation profile.}
Under the approximately symmetric penalty ($\alpha=1.5$, leftmost panel),
GA-AMLS (blue) attains the lowest cross-model average loss on 7 of the 8
distributions, with hex the exception; the per-model breakdown in
Table~\ref{tab:spb_alpha_1p5_gamma_0p0} shows the 1- and 4-layer wins and the
2-layer tie behind these averages. As the penalty for underestimation
increases ($\alpha>1.5$), QLD (orange) gains the advantage because its
systematic overestimation bias guards against false-negative penalties.
Lower is better; error bars denote standard error across the three model
sizes.}
    \label{fig:SPB-dist}
\end{figure}

\begin{figure}[h]
    \centering
    \includegraphics[width=1\linewidth]{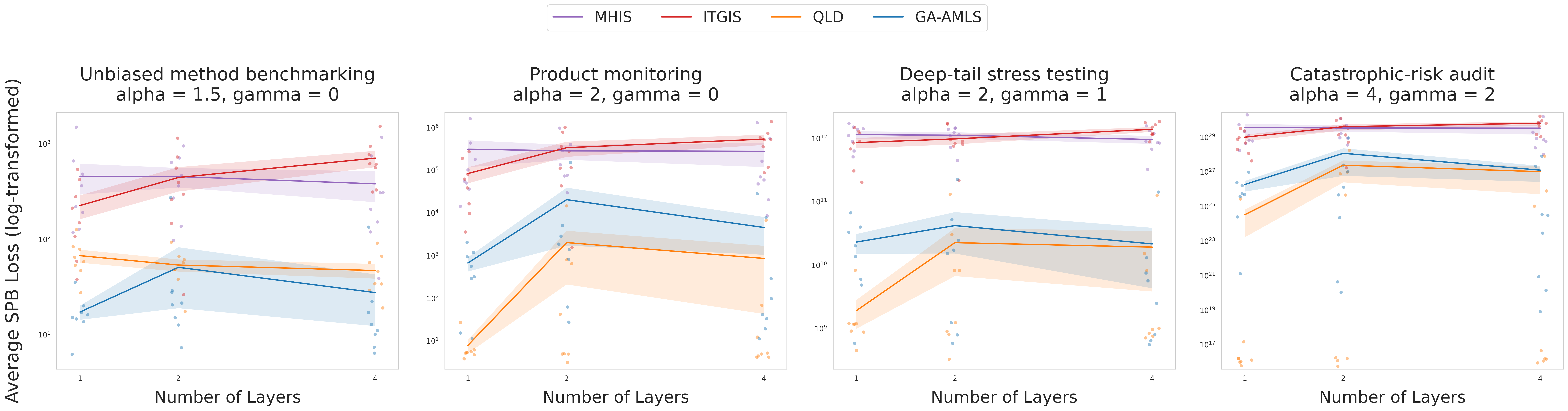}
    \caption{\textbf{The SPB Loss of all methods across different model sizes and parameter combinations.} The solid lines indicate the loss of each method averaged over all 8 distributions, with bands showing standard error. The colored points indicate the loss on individual distributions, with horizontal jitter added for visibility. Lower is better. We see that GA-AMLS (blue) often outperforms when the penalty for FN is approximately the same as FP. The effect of the rarity premium $\gamma$ depends on its interaction with $\alpha$: upweighting rarer events favors the method whose deep-tail errors are least penalized under the chosen asymmetry profile.}
    \label{fig:Layers-SPB}
\end{figure}

\begin{table}[ht!]
\centering
\scriptsize
\setlength{\tabcolsep}{2pt}
\caption{SPB Loss comparison across model sizes ($\alpha=1.5$, $\gamma=0$, $\varepsilon=\varepsilon_{\mathrm{corr}} = 3.5 \times 10^{-12}$). Lower is better.}
\label{tab:spb_alpha_1p5_gamma_0p0}

\begin{subtable}{0.315\textwidth}
\centering
\caption{1-layer model}
\begin{tabular}{lcccc}
\toprule
Dist. & MHIS & ITGIS & QLD & GA-AMLS \\
\midrule
\texttt{hex} & 117.64 & 538.82 & 53.26 & \textbf{35.44} \\
\texttt{camel} & 660.47 & 148.67 & 125.83 & \textbf{14.62} \\
\texttt{colon} & 127.09 & 278.54 & 58.18 & \textbf{6.24} \\
\texttt{if} & 362.68 & 211.57 & 46.96 & \textbf{15.16} \\
\texttt{caps} & 1489.97 & 421.47 & 83.47 & \textbf{17.05} \\
\texttt{english} & 190.48 & 58.75 & 27.49 & \textbf{16.16} \\
\texttt{spanish} & 480.38 & 106.99 & 78.54 & \textbf{13.67} \\
\texttt{icl} & 218.55 & 37.67 & 64.68 & \textbf{20.05} \\
\midrule
Avg. & 455.91 & 225.31 & 67.30 & \textbf{17.30} \\
\bottomrule
\end{tabular}
\end{subtable}
\hfill
\begin{subtable}{0.315\textwidth}
\centering
\caption{2-layer model}
\begin{tabular}{lcccc}
\toprule
Dist. & MHIS & ITGIS & QLD & GA-AMLS \\
\midrule
\texttt{hex} & 136.82 & 1141.36 & \textbf{66.22} & 272.59 \\
\texttt{camel} & 947.94 & 725.50 & 93.04 & \textbf{27.75} \\
\texttt{colon} & 361.36 & 555.99 & 48.02 & \textbf{12.63} \\
\texttt{if} & 270.16 & 146.68 & 48.44 & \textbf{7.31} \\
\texttt{caps} & 463.91 & 392.87 & 56.80 & \textbf{20.55} \\
\texttt{english} & 708.88 & 295.81 & \textbf{17.51} & 21.46 \\
\texttt{spanish} & 637.53 & 259.79 & 61.24 & \textbf{15.04} \\
\texttt{icl} & 97.04 & \textbf{26.26} & 38.06 & 28.94 \\
\midrule
Avg. & 452.96 & 443.03 & 53.67 & \textbf{50.78} \\
\bottomrule
\end{tabular}
\end{subtable}
\hfill
\begin{subtable}{0.315\textwidth}
\centering
\caption{4-layer model}
\begin{tabular}{lcccc}
\toprule
Dist. & MHIS & ITGIS & QLD & GA-AMLS \\
\midrule
\texttt{hex} & 308.79 & 1520.17 & \textbf{45.92} & 133.83 \\
\texttt{camel} & 1167.25 & 561.11 & 90.86 & \textbf{12.81} \\
\texttt{colon} & 151.61 & 767.77 & 34.08 & \textbf{11.06} \\
\texttt{if} & 119.23 & 615.02 & 34.02 & \textbf{6.41} \\
\texttt{caps} & 205.58 & 609.63 & 56.90 & \textbf{17.07} \\
\texttt{english} & 306.29 & 327.72 & \textbf{19.07} & 22.30 \\
\texttt{spanish} & 743.64 & 312.04 & 66.29 & \textbf{7.39} \\
\texttt{icl} & 38.85 & 940.25 & 28.98 & \textbf{10.07} \\
\midrule
Avg. & 380.15 & 706.71 & 47.01 & \textbf{27.62} \\
\bottomrule
\end{tabular}
\end{subtable}
\end{table}

\begin{table}[ht!]
\centering
\fontsize{6.2}{7.0}\selectfont
\setlength{\tabcolsep}{0.8pt}
\caption{SPB Loss comparison across model sizes ($\alpha=2$, $\gamma=0$, $\varepsilon=\varepsilon_{\mathrm{corr}}= 3.3 \times 10^{-13}$). Lower is better.}
\label{tab:spb_alpha_2p0_gamma_0p0}

\begin{subtable}{0.32\textwidth}
\centering
\caption{1-layer model}
\begin{tabular}{@{}lcccc@{}}
\toprule
Dist. & MHIS & ITGIS & QLD & GA-AMLS \\
\midrule
\texttt{hex} & 3.62e+04 & 2.70e+05 & \textbf{26.99} & 2053.05 \\
\texttt{camel} & 4.34e+05 & 3.84e+04 & \textbf{6.15} & 318.50 \\
\texttt{colon} & 1.43e+04 & 6.20e+04 & \textbf{5.32} & 11.35 \\
\texttt{if} & 1.02e+05 & 7.99e+04 & \textbf{4.72} & 935.19 \\
\texttt{caps} & 1.62e+06 & 1.90e+05 & \textbf{5.34} & 15.28 \\
\texttt{english} & 5.57e+04 & 1.63e+04 & \textbf{3.79} & 1183.62 \\
\texttt{spanish} & 1.79e+05 & 9693.76 & \textbf{5.55} & 289.98 \\
\texttt{icl} & 5.00e+04 & 3564.92 & \textbf{5.13} & 556.86 \\
\midrule
Avg. & 3.12e+05 & 8.38e+04 & \textbf{7.87} & 670.48 \\
\bottomrule
\end{tabular}
\end{subtable}
\hspace{0.002\textwidth}%
\begin{subtable}{0.32\textwidth}
\centering
\caption{2-layer model}
\begin{tabular}{@{}lcccc@{}}
\toprule
Dist. & MHIS & ITGIS & QLD & GA-AMLS \\
\midrule
\texttt{hex} & 2.96e+04 & 1.02e+06 & \textbf{1.46e+04} & 1.52e+05 \\
\texttt{camel} & 9.71e+05 & 7.79e+05 & \textbf{644.61} & 5066.22 \\
\texttt{colon} & 1.35e+05 & 2.76e+05 & \textbf{4.88} & 62.13 \\
\texttt{if} & 7.62e+04 & 4.30e+04 & \textbf{4.92} & 27.54 \\
\texttt{caps} & 2.89e+05 & 3.60e+05 & \textbf{42.00} & 1379.45 \\
\texttt{english} & 4.02e+05 & 1.12e+05 & \textbf{3.12} & 2853.15 \\
\texttt{spanish} & 3.06e+05 & 1.14e+05 & \textbf{4.99} & 816.30 \\
\texttt{icl} & 7.31e+04 & 1545.28 & \textbf{798.60} & 1858.51 \\
\midrule
Avg. & 2.85e+05 & 3.38e+05 & \textbf{2017.40} & 2.05e+04 \\
\bottomrule
\end{tabular}
\end{subtable}
\hspace{0.024\textwidth}%
\begin{subtable}{0.32\textwidth}
\centering
\caption{4-layer model}
\begin{tabular}{@{}lcccc@{}}
\toprule
Dist. & MHIS & ITGIS & QLD & GA-AMLS \\
\midrule
\texttt{hex} & 1.64e+05 & 1.38e+06 & \textbf{6749.56} & 2.81e+04 \\
\texttt{camel} & 1.30e+06 & 5.27e+05 & \textbf{4.90} & 287.16 \\
\texttt{colon} & 4.77e+04 & 5.83e+05 & \textbf{4.14} & 11.17 \\
\texttt{if} & 2.03e+04 & 3.53e+05 & \textbf{4.16} & 97.97 \\
\texttt{caps} & 5.92e+04 & 5.17e+05 & \textbf{4.42} & 41.13 \\
\texttt{english} & 7.00e+04 & 8.72e+04 & \textbf{68.21} & 7759.93 \\
\texttt{spanish} & 5.58e+05 & 1.18e+05 & \textbf{5.17} & 33.23 \\
\texttt{icl} & 8595.27 & 7.34e+05 & \textbf{12.28} & 19.18 \\
\midrule
Avg. & 2.79e+05 & 5.38e+05 & \textbf{856.60} & 4547.43 \\
\bottomrule
\end{tabular}
\end{subtable}
\end{table}

\section{Repeated-Run Variability}
\label{sec:repeated-run variability}

Each estimate in the main experiments comes from a single run per (token,
distribution, model). To check that the cross-token dispersion visible in the scatterplots is not an
artifact of run-to-run randomness, we repeated the estimation with five
independent runs per target token, for both GA-AMLS and QLD, on two token sets. The first is a stratified core: 24 targets per
setting, stratified by ground-truth probability, across ten
(distribution, model) settings: all eight distributions on the 2-layer
model, plus hex on the 1- and 4-layer models. The second is a stress set
targeting where GA-AMLS underperforms due to underestimation relative to ground truth: in the three
settings where such underestimates were most frequent
(hex at 2 and 4 layers, english at 2 layers), we added the eight targets
with the largest single-run underestimation, selected before any re-running.
In total: $10 \times 24 + 3 \times 8 = 264$ targets, $264 \times 5 = 1{,}320$
runs for each estimator.

Figure~\ref{fig:multi-seed_run} shows the results. For each token, the five
per-run estimates cluster tightly around their mean, while the means spread
widely across tokens. Run-to-run variation therefore accounts for a small fraction of the
total dispersion, and the cross-token dispersion mentioned in Section \ref{sec:gaamls_real_data_exp} reflects variation in estimator error across targets rather than sampling noise of individual runs. Quantitatively, within each (distribution, model) setting the median run-to-run standard deviation of $\log_{10}\hat{p}$ is between $0.00$ and $0.10$ for GA-AMLS and at most $0.04$ for QLD. By contrast, the standard deviation across tokens of the per-token mean $\log_{10}$ error is $0.7$--$3.1$ for GA-AMLS and $0.6$--$1.8$ for QLD. Run-to-run variation therefore accounts for at most $1.2\%$ of the total variance of the $\log_{10}$ error in any setting (median $0.2\%$).

\begin{figure}
    \centering
    \includegraphics[width=1\linewidth]{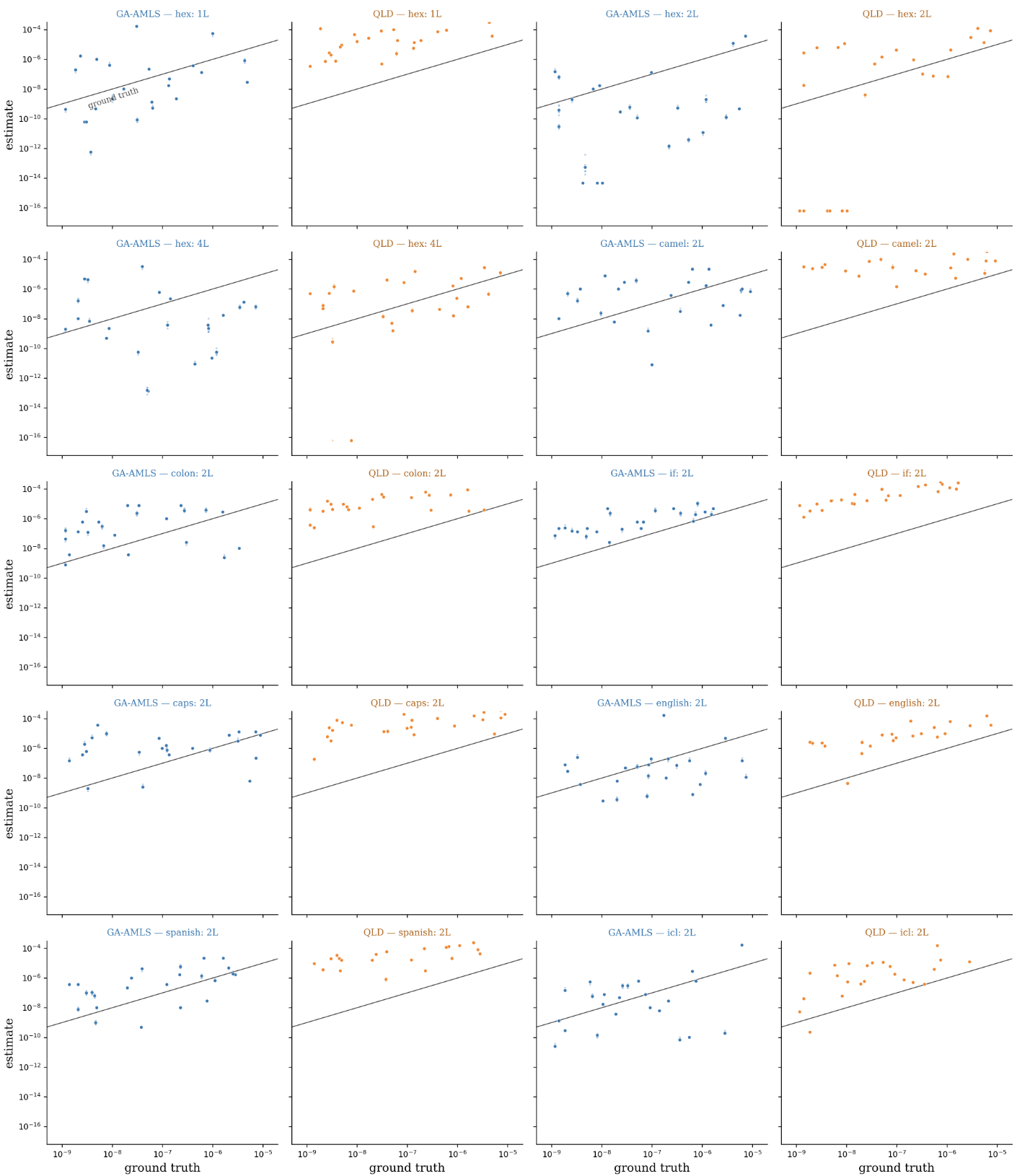}
    \caption{\textbf{Repeated-run variability of GA-AMLS and QLD.} Estimates versus ground truth for 24 stratified target tokens per distribution, five
independent runs per token, with GA-AMLS (blue) and QLD (orange) side by side
for each setting. Faint dots show individual runs, solid dots the mean over
runs, and the solid line marks estimate $=$ ground truth. Run-to-run spread is small relative to the spread across
tokens in every panel.}
    \label{fig:multi-seed_run}
\end{figure}

\section{Computational Cost}
\label{app:compcost}

Table~\ref{tab:compcost} reports the cost of producing one probability estimate for each of
the four estimators, as median wall-clock seconds and floating-point operations per target
token. GA-AMLS and QLD are measured on the repeated-run study of
Appendix~\ref{sec:repeated-run variability}; ITGIS and MHIS are measured over the single
benchmark runs per target. 

GA-AMLS costs roughly 300--400$\times$ more wall-clock time and 800--1000$\times$ more FLOPs than QLD per target: it evolves $N=2000$ particles through sequential MALA rejuvenation across adaptive levels, whereas QLD reduces estimation to closed-form matrix operations once the shared activation pool is loaded. Both methods reuse the same cached calibration activations, whose one-off cost per (model, distribution) pair is shared across all target tokens and excluded from the per-target numbers. This gap is the price of replacing QLD's independence-based recombination with conditional sampling under an explicit prior, and it is the practical reason to reserve GA-AMLS for settings where its lower bias matters more than throughput. The input-space methods sit between these extremes: GA-AMLS is roughly $3$--$7\times$
slower per target than ITGIS and MHIS but avoids their zero-estimate collapse, while QLD remains one to two orders of magnitude cheaper than any sampling-based method.

\begin{table}[t]
\centering
\caption{Compute cost for all four estimators. Times (wall-clock seconds) and FLOPs are
reported per target-token probability estimate. GA-AMLS uses sequential activation-space
scorer evaluations; QLD reduces estimation to closed-form matrix operations after the
shared activation pool is loaded; ITGIS and MHIS run full-model passes over sampled token
sequences. FLOP counts for ITGIS include forward passes only; MHIS counts include forward
and backward passes.}
\label{tab:compcost}
\setlength{\tabcolsep}{4pt}
\small
\begin{tabular}{lcccccccc}
\toprule
& \multicolumn{2}{c}{GA-AMLS} & \multicolumn{2}{c}{QLD} & \multicolumn{2}{c}{ITGIS} & \multicolumn{2}{c}{MHIS} \\
\cmidrule(lr){2-3}\cmidrule(lr){4-5}\cmidrule(lr){6-7}\cmidrule(lr){8-9}
Model & sec & FLOPs & sec & FLOPs & sec & FLOPs & sec & FLOPs \\
\midrule
GELU-1L & $450$ & $2.98{\times}10^{15}$ & $1.39$ & $3.34{\times}10^{12}$ & $72$ & $1.31{\times}10^{13}$ & $120$ & $3.94{\times}10^{13}$ \\
GELU-2L & $408$ & $2.67{\times}10^{15}$ & $1.33$ & $3.34{\times}10^{12}$ & $72$ & $2.30{\times}10^{13}$ & $132$ & $6.91{\times}10^{13}$ \\
GELU-4L & $546$ & $3.38{\times}10^{15}$ & $1.40$ & $3.34{\times}10^{12}$ & $78$ & $4.28{\times}10^{13}$ & $138$ & $1.28{\times}10^{14}$ \\
\bottomrule
\end{tabular}
\end{table}

\section{MCMC Mixing Diagnostics}
\label{app:mcmc_diagnostics}

A natural concern with GA-AMLS is that inaccurate estimates could be caused
by insufficient MCMC mixing within the constrained AMLS levels. We therefore
recorded two MALA diagnostics on the repeated-run study of
Appendix~\ref{sec:repeated-run variability}: acceptance rates and
within-level score autocorrelation. These diagnostics provide a check for
obvious MCMC pathologies.

GA-AMLS uses $T=900$ MALA proposal steps per AMLS level, with the
step size adapted during the first $T_{\mathrm{burn}}=126$ steps toward the
high-dimensional MALA target acceptance rate $r^{\star}=0.57$ and then frozen
for the remainder of the level. Figure~\ref{fig:mcmc_acceptance} reports
acceptance rates and adapted step sizes across representative runs, together
with the distribution of mean post-burn-in acceptance rates over all
repeated-run diagnostics. The observed acceptance rates concentrate near the
target, including for runs whose final estimates substantially underestimate
the ground truth reference probability.

Acceptance rates alone do not establish mixing, so
Figure~\ref{fig:mcmc_autocorr} also reports within-level autocorrelation of
the scalar score $s(u)$ for representative early, middle, and late AMLS
levels. The score autocorrelation decays within tens of MALA steps, with
integrated autocorrelation times of approximately 9--13 steps in the
displayed runs, well below the $T=900$ rejuvenation budget. Together, these
diagnostics suggest that the most severely underestimated estimates are not
primarily explained by poor within-level MCMC mixing. Instead, consistent
with the prior-sensitivity results of Appendix~\ref{app:prior_sensitivity},
they are more plausibly associated with mismatch between the fitted
activation prior and the input-induced activation distribution, especially in
difficult tail regimes.

\begin{figure}
    \centering
    \includegraphics[width=1\linewidth]{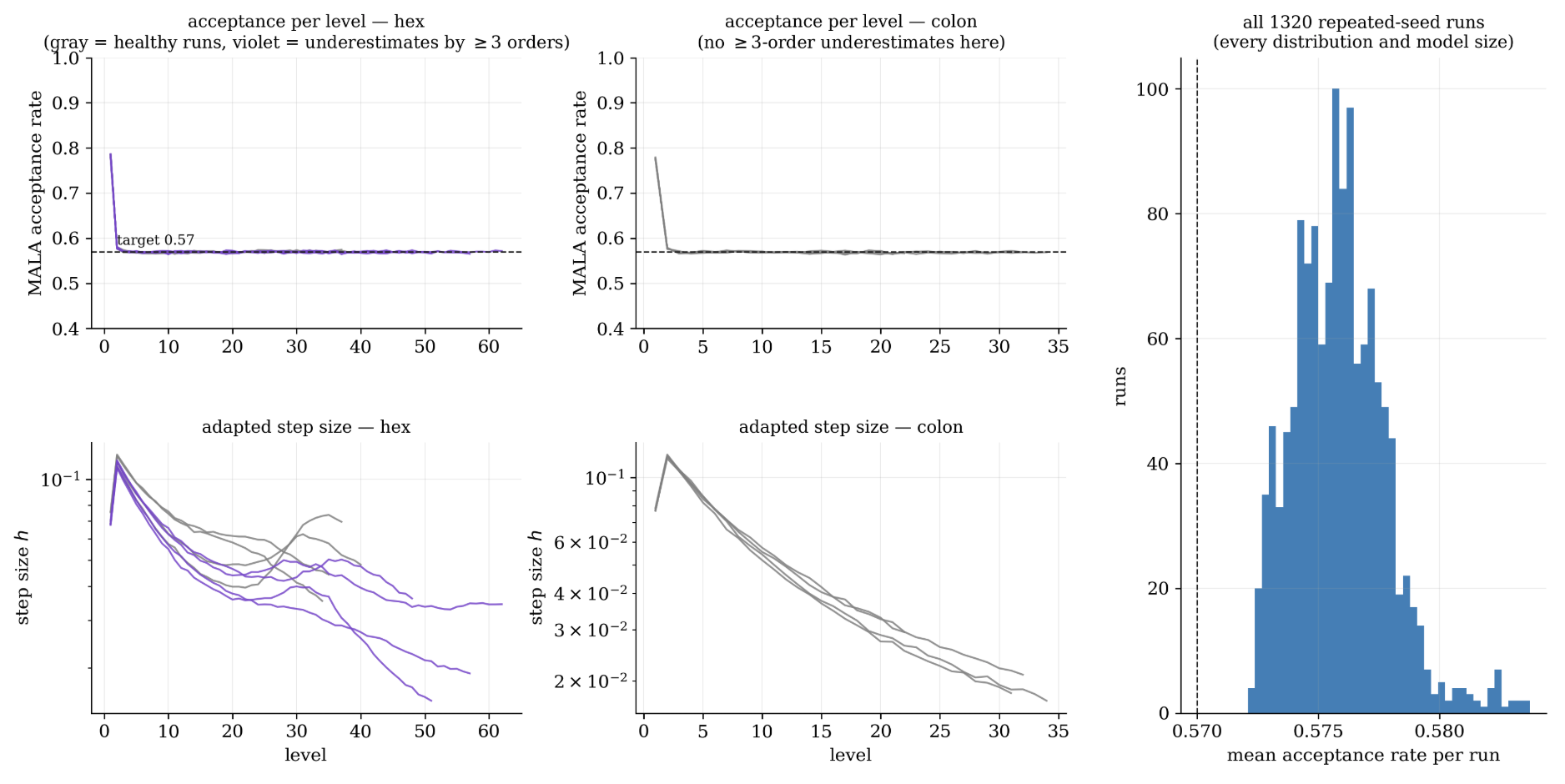}
\caption{
MALA acceptance-rate diagnostics. Left and middle: representative acceptance rates and adapted step sizes across AMLS levels for selected repeated-seed runs. Gray curves denote runs whose estimates are within three orders of magnitude of the ground truth probability, and violet curves denote runs with larger underestimation. Right: distribution of mean post-burn-in acceptance rates over all repeated-seed diagnostic runs. Acceptance rates concentrate near the adaptation target $0.57$, suggesting that the MALA step-size adaptation is behaving as intended.
}\label{fig:mcmc_acceptance}
\end{figure}

\begin{figure}
    \centering
    \includegraphics[width=1\linewidth]{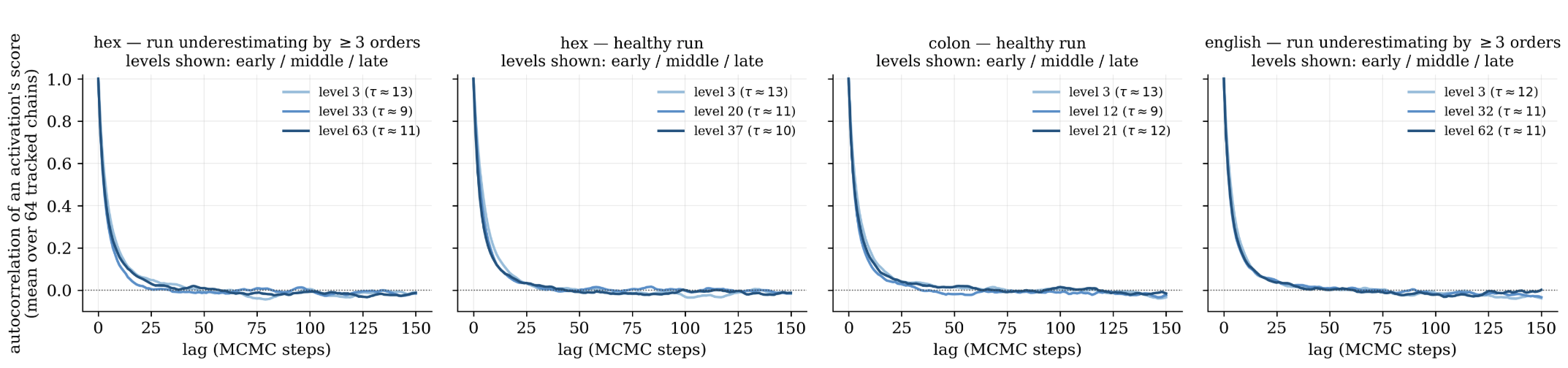}
\caption{
Within-level score autocorrelation diagnostics for representative runs. For each run, we show autocorrelation of the particle score $s(u_t)$ within early, middle, and late AMLS levels, averaged over tracked chains. The integrated autocorrelation times are approximately $9$--$13$ MALA steps, far below the $T=900$ rejuvenation budget used in the main experiments. 
}
    \label{fig:mcmc_autocorr}
\end{figure}

\section{Sensitivity to Student-$t$ Activation-Prior Fit}
\label{app:prior_sensitivity}

GA-AMLS estimates a probability under a fitted activation-space distribution, not directly under the
original discrete input distribution. For a target token $t$, let
\[
p_D(t) = \Pr_{x\sim D}\!\left[s_t(a(x)) \ge 0\right],
\]
where $a(x)$ is the cached activation and
$s_t(u)=\mathrm{logit}_t(u)-\max_{j\ne t}\mathrm{logit}_j(u)$ is the target-token margin. GA-AMLS instead
fits an activation prior $\pi_{\theta,D}$ to calibration activations from $D$ and estimates
\[
p_{\pi}(t) = \Pr_{u\sim \pi_{\theta,D}}\!\left[s_t(u) \ge 0\right].
\]
Idealized AMLS is unbiased for the tail probability of its own sampling distribution, but agreement between
$p_{\pi}(t)$ and $p_D(t)$ depends on the fitted prior matching the input-induced activation distribution,
especially in the joint tails relevant for rare target tokens.  
We therefore add a targeted prior-sensitivity diagnostic on hard cases. We use 12 target tokens per
distribution from the 2-layer model, one GA-AMLS run per token, and evaluate three
distributions: \texttt{hex}, \texttt{colon}, and \texttt{english}. The diagnostic metric is the fraction of target
tokens whose GA-AMLS estimate is at least three orders of magnitude below the ground-truth probability. 

First, we vary the Student-$t$ degrees of freedom
$\nu\in\{3,5,10,\infty\}$, where $\nu=\infty$ denotes a Gaussian prior.  Smaller $\nu$ gives heavier tails.
Figure~\ref{fig:nu_ablation} shows that tail thickness materially affects collapse on the hardest
distribution. On \texttt{hex}, the collapse rate is approximately $0\%$ for $\nu=3$, $8\%$ for the paper
default $\nu=5$, $50\%$ for $\nu=10$, and $92\%$ for the Gaussian prior.  The effect is much smaller on
\texttt{colon} and \texttt{english}. This supports the use of a heavy-tailed activation prior.

Second, we compare the paper's diagonal Student-$t$ prior with a non-factorized multivariate
Student-$t$ prior with the same $\nu=5$ but a shared joint tail. In whitened coordinates, the diagonal prior
makes the coordinates independent, while the multivariate Student-$t$ introduces tail dependence. Figure~\ref{fig:mvt_t} shows that this shared-joint-tail prior reduces
the \texttt{hex} collapse rate from $8\%$ to $0\%$ in this diagnostic, while \texttt{colon} and \texttt{english}
remain at $0\%$.

These results have two implications. First, the choice $\nu=5$ should be interpreted as a tuned compromise,
not as a universal optimum: heavier tails can help on some hard tokens, while thinner-tailed or Gaussian
priors can severely under-sample rare activation directions. Second, the diagonal independence assumption
is a real limitation. A shared-joint-tail prior can reduce collapse on the hardest scoped case, suggesting that
some remaining GA-AMLS failures are better explained by activation-prior tail mismatch than by MCMC
mixing alone. We therefore treat prior mismatch as an important limitation of the current method as mentioned in Section \ref{sec:discussion}.

\begin{figure}
    \centering
    \includegraphics[width=0.75\linewidth]{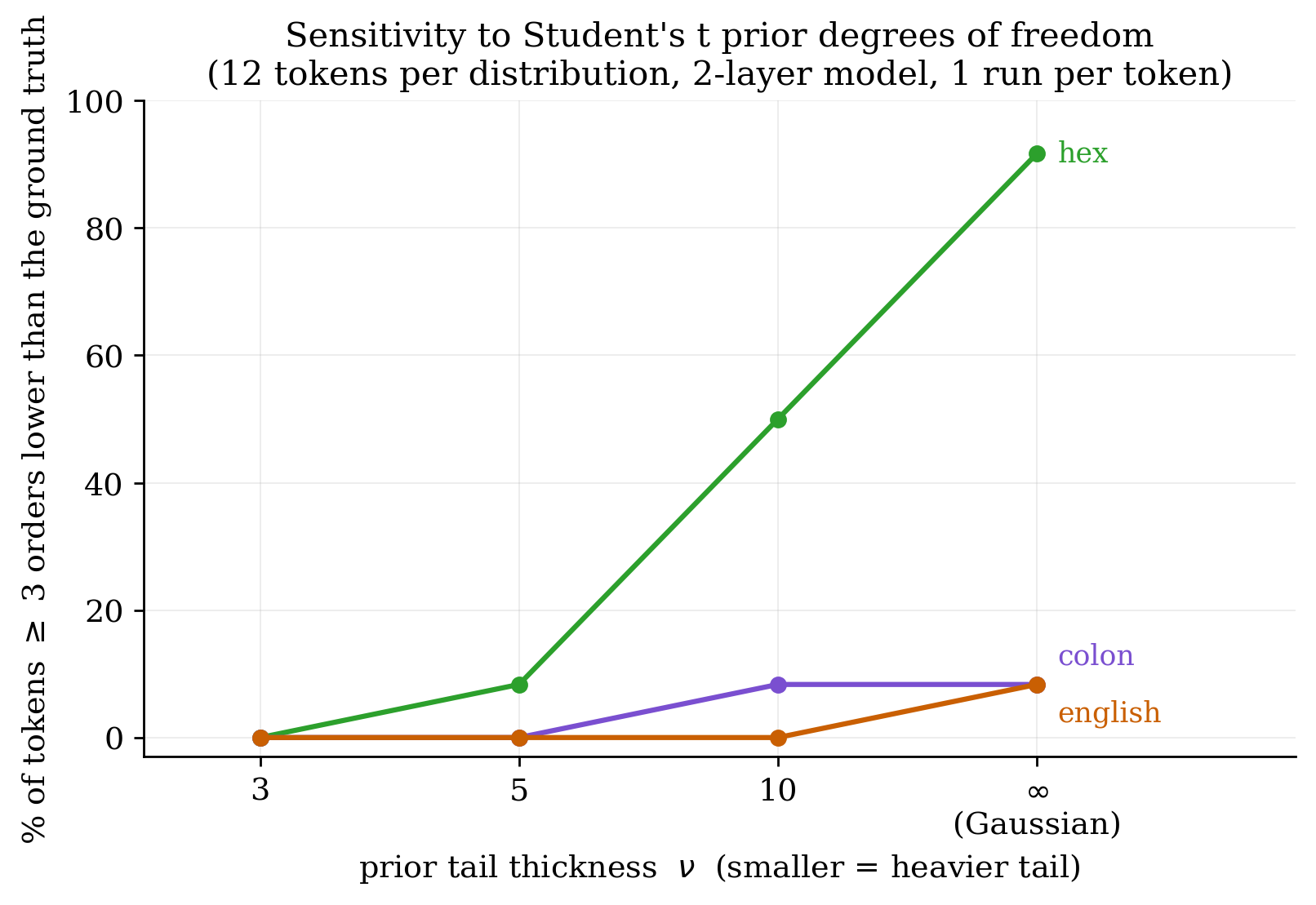}
    \caption{\textbf{Sensitivity to the tail thickness of the Student-$t$ prior.}
Fraction of target tokens whose GA-AMLS estimate falls at least three orders
of magnitude below the ground truth, as a function of the prior degrees of
freedom $\nu$ (12 tokens per distribution, 2-layer model, one run per token;
$\nu=\infty$ denotes a Gaussian prior). On hex, the hardest distribution,
thinner-tailed priors sharply increase the collapse rate: from 0\% at
$\nu=3$ to 92\% for the Gaussian, while colon and english are largely
insensitive. Heavy tails in the activation prior are essential for reaching
the rare target regions.}
    \label{fig:nu_ablation}
\end{figure}

\begin{figure}
    \centering
    \includegraphics[width=0.75\linewidth]{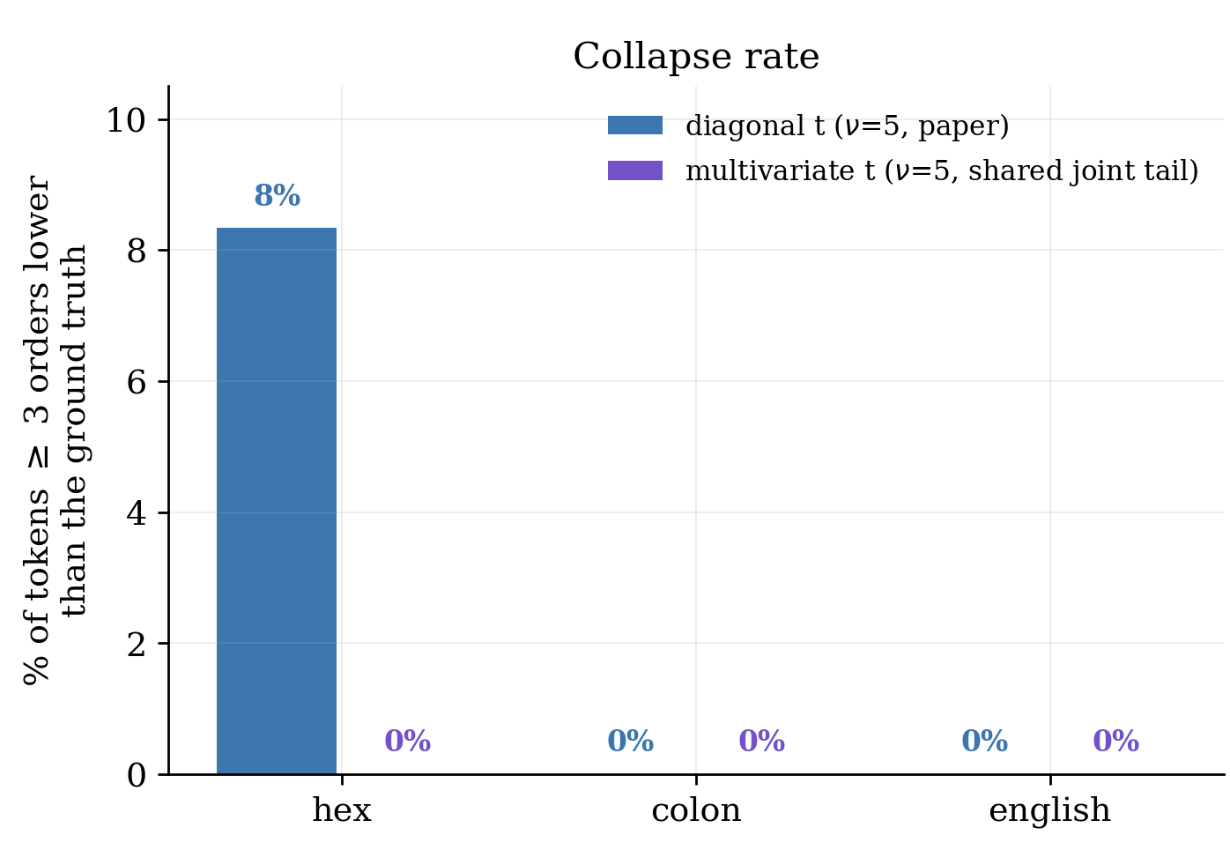}
    \caption{\textbf{Diagonal versus joint-tail Student-$t$ prior.} Fraction of target tokens whose GA-AMLS estimate falls at least three orders
of magnitude below the ground truth for the paper's diagonal
Student-$t$ prior ($\nu=5$) and a multivariate Student-$t$
with the same $\nu$ but a shared joint tail. The joint-tail prior removes the
remaining hex collapses (8\% $\to$ 0\%) in this diagnostic, indicating that
part of GA-AMLS's deep-underestimate failure mode stems from the
coordinate-independence of the diagonal prior rather than from MCMC mixing.}
\label{fig:mvt_t}
\end{figure}

\section{Sensitivity of the Empirical Rankings to the Shift $\varepsilon$}
\label{app:eps_sensitivity}

The shift $\varepsilon$ keeps the loss finite when an estimator returns
exactly zero, and Theorems~\ref{thm:scale} and~\ref{thm:asymmetry} bound how far it may perturb the unshifted
geometry. A separate practical question is whether the empirical comparisons
in Section~\ref{sec:spb_synth_data_exp} depend on the particular value chosen. We therefore recompute
the aggregate real-data evaluation while sweeping $\varepsilon$ over four
orders of magnitude, spanning the corollary-based choice
$\varepsilon_{\mathrm{corr}}$ and the heuristic choice
$\varepsilon_{\mathrm{heur}} = 0.01\,q_{\min}/m_{\max}$.
Figure~\ref{fig:eps_sweep} reports the mean SPB Loss over the 8 distributions
for each method and model size, under the near-symmetric ($\alpha=1.5$) and
FN-averse ($\alpha=2$) profiles with $\gamma=0$.

Two patterns emerge. First, the two headline comparisons are stable across
essentially the entire sweep: GA-AMLS attains the lowest $\alpha=1.5$ loss in
the 1- and 4-layer models, with the 2-layer tie intact, at every
$\varepsilon$ tested, and QLD attains the lowest $\alpha=2$ loss at every
$\varepsilon$ at and above $\varepsilon_{\mathrm{corr}}$ in all three models.
Second, the absolute losses of the input-space methods fall steeply as
$\varepsilon$ grows. This is expected:
$\varepsilon$ sets the finite penalty assigned to an exact-zero estimate, so
the score of any estimator that frequently returns zero must depend on it,
and its position relative to the non-collapsing methods shifts at the largest
$\varepsilon$ values. 

\begin{figure}
    \centering
    \includegraphics[width=1\linewidth]{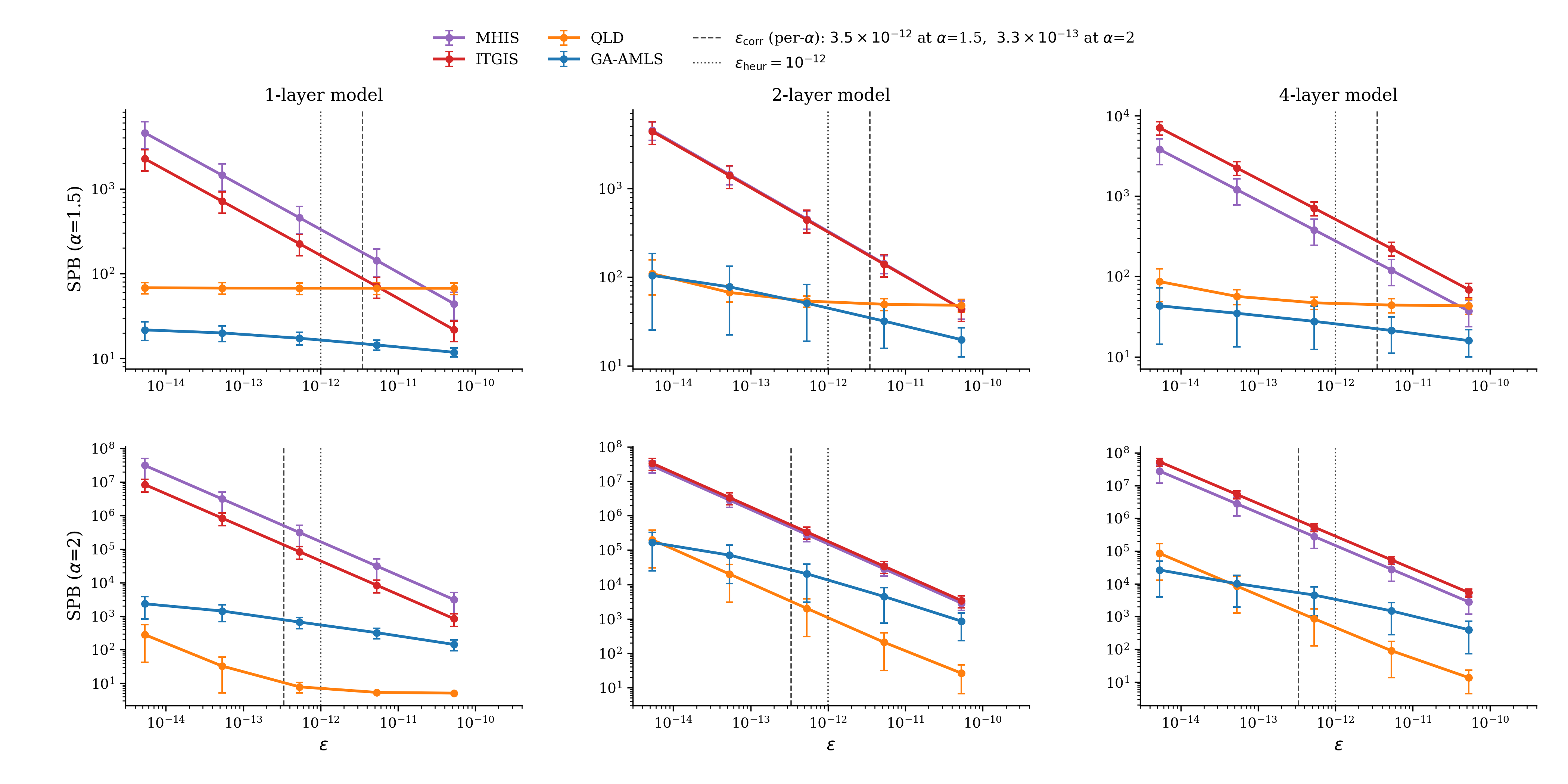}
    \caption{\textbf{Sensitivity of the real-data rankings to
    $\varepsilon$.} Aggregate SPB Loss as $\varepsilon$ is swept over four orders of
    magnitude, for the near-symmetric ($\alpha=1.5$, top) and FN-averse
    ($\alpha=2$, bottom) profiles at $\gamma=0$. Dashed and dotted lines mark
    the corollary-based and heuristic choices.}
    \label{fig:eps_sweep}
\end{figure}

\section{Ground-Truth Reliability in the Rarest Regimes}
\label{app:gt_reliability}

The reference probabilities used in this benchmark are themselves finite-sample Monte Carlo estimates and this finite-sample noise acts as a confounder when comparing estimators in the deepest tail. Following \cite{wu2025estimatingprobabilitiesrareoutputs}, the reference probabilities are estimated from \(M=2^{32}\) forward passes. For an event with true probability \(q\), the reference count \(K\) satisfies
\[
    K \sim \mathrm{Binomial}(M,q) \approx \mathrm{Poisson}(Mq)
\]
in the rare-event regime. Thus the relative standard error of the reference estimate \(\hat q_{\mathrm{ref}}=K/M\) is approximately
\[
    \frac{\mathrm{sd}(\hat q_{\mathrm{ref}})}{q}
    \approx \frac{1}{\sqrt{Mq}},
\]
and, by the delta method,
\[
    \mathrm{sd}\!\left(\log_{10} \hat q_{\mathrm{ref}}\right)
    \approx \frac{1}{\ln(10)\sqrt{Mq}}.
\]
This noise floor is non-negligible for the rarest targets. For example, at \(q=10^{-9}\), the expected count is \(Mq \approx 4.29\), giving a relative standard error of about \(48\%\). At \(q=10^{-8}\), the expected count is \(42.9\), giving a relative standard error of about \(15\%\). At \(q=10^{-7}\), the expected count is \(429\), giving a relative standard error of about \(4.8\%\).

This means that losses on the deepest-tail individual targets should not be interpreted as exact measurements of estimator error. In particular, when the dataset-level rarity premium \(\gamma>0\) upweights the smallest reference probabilities, the resulting aggregate loss is also more sensitive to this finite-reference noise. We therefore perform two robustness checks. First, we recompute the main comparisons after filtering out targets whose reference counts are below \(20\) or \(50\). Second, we stratify the symmetric-loss comparison by reference-count bands to show where the method differences are concentrated.

\begin{figure}
    \centering
    \includegraphics[width=1\linewidth]{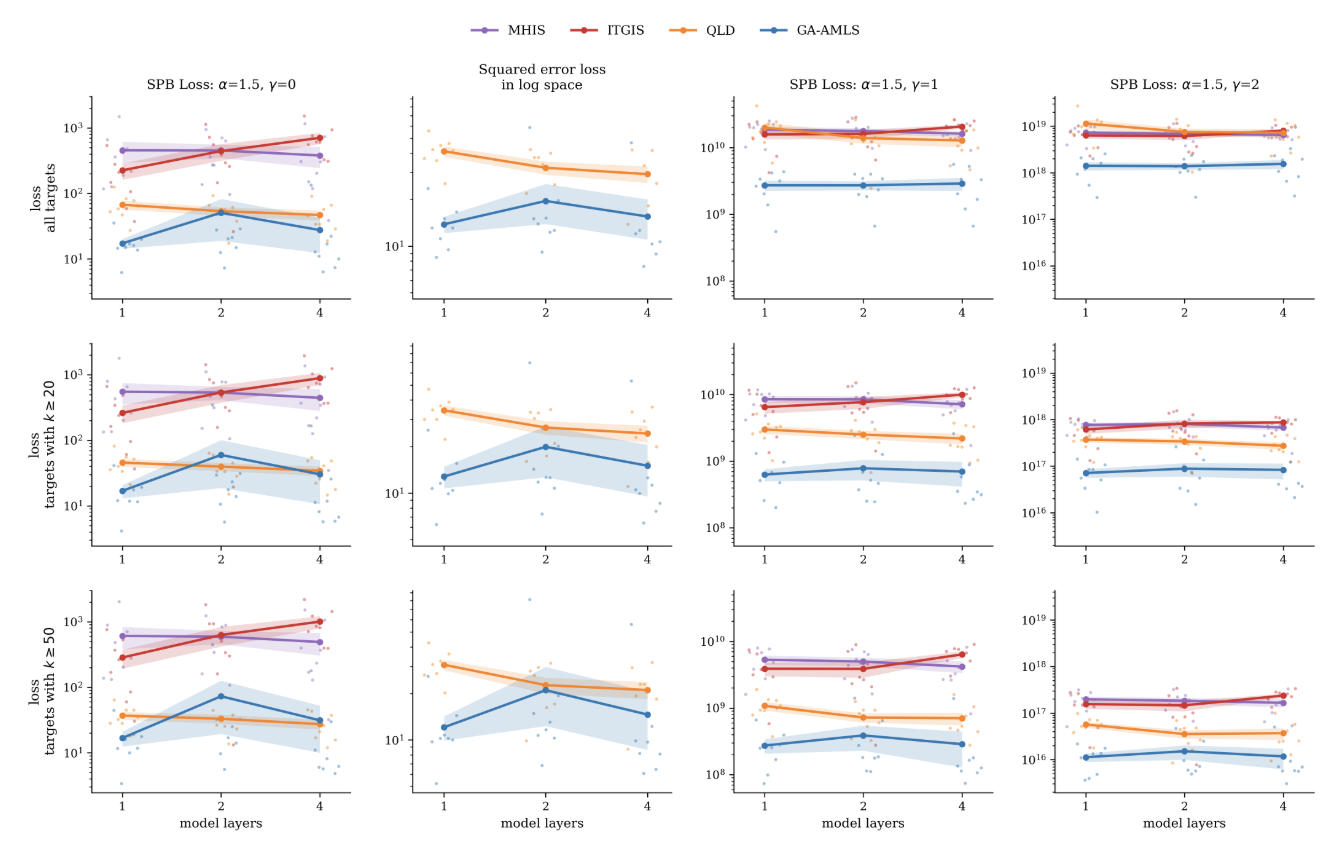}
    \caption{Ground-truth reliability filtering. Rows show the same evaluation computed on all targets, only targets with reference count \(k \ge 20\), and only targets with reference count \(k \ge 50\). The \(k\ge 20\) and \(k\ge 50\) filters remove the noisiest deep-tail reference probabilities. Columns show the scale-neutral symmetric SPB profile \((\alpha=1.5,\gamma=0)\), log-space squared error, and increasingly tail-weighted SPB profiles with \(\gamma=1\) and \(\gamma=2\). Filtering changes some scale-neutral comparisons, especially where QLD and GA-AMLS are already close, so we do not interpret the unfiltered deepest-tail results as exact pointwise rankings. However, the qualitative separation between QLD and GA-AMLS under the tail-weighted \(\gamma>0\) profiles remains visible after applying \(k\ge 20\) and \(k\ge 50\) filters. This suggests that the rarity-premium results are not solely an artifact of the \(k\approx 4\) reference-count regime, although they should still be interpreted with this finite-reference noise floor in mind.}
    \label{fig:gt_sens}
\end{figure}

Figure~\ref{fig:gt_sens} separates two effects that are conflated in the unfiltered benchmark: estimator performance in the rare tail and finite-sample noise in the reference probabilities. We repeat the evaluation on all targets, on targets with \(k\ge20\), and on targets with \(k\ge50\), where \(k=M\hat q_{\mathrm{ref}}\) is the reference count. Under the Poisson approximation, these filters restrict the reference-probability relative standard error to at most about \(22\%\) and \(14\%\), respectively.

The filtering changes the interpretation of the scale-neutral SPB comparison. For \(\alpha=1.5,\gamma=0\), GA-AMLS no longer uniformly beats QLD after removing the lowest-count targets, especially in the 2-layer model. The main text therefore avoids a uniform symmetric-loss dominance claim. However, the tail-weighted comparisons tell a different story: with the same approximately symmetric pointwise penalty \(\alpha=1.5\), GA-AMLS remains below QLD for \(\gamma=1\) and \(\gamma=2\) even after the \(k\ge20\) and \(k\ge50\) filters. Thus the rarity-premium result is not explained solely by the noisiest \(k\approx4\) reference probabilities. Rather, GA-AMLS's advantage is most visible when evaluation emphasizes the rarest targets within a reliability-filtered set.

Table~\ref{tab:rarity_bands} gives a count-band breakdown for the scale-neutral symmetric profile \((\alpha=1.5,\gamma=0)\). The rare-but-more-reliable band \(50\le k<200\), whose reference estimates have at most about \(14\%\) relative standard error, still favors GA-AMLS across all three model sizes. In higher-count, less rare bands, QLD becomes competitive or better in some models. This reinforces the interpretation that GA-AMLS's advantage is concentrated in rare-event regimes, while also showing why we should not claim uniform dominance across all rarity bands.

\begin{table}[t]
\centering
\caption{
Symmetric-loss \((\alpha=1.5,\gamma=0)\) comparison by reference-count band, with targets pooled across distributions. Here \(k=M\hat q_{\mathrm{ref}}\) is the reference count implied by the \(2^{32}\)-sample benchmark estimate. The column \(\max \mathrm{sd}(\log_{10}\hat q_{\mathrm{ref}})\) gives the approximate worst-case log-scale standard deviation within each band, using the lower edge of the band and the Poisson delta-method approximation \(1/(\ln(10)\sqrt{k})\). Lower loss is better.
}
\label{tab:rarity_bands}
\begin{tabular}{llccccc}
\toprule
Layers & \(k\) band & \(n\) & \(\max \mathrm{sd}(\log_{10}\hat q_{\mathrm{ref}})\) & QLD & GA-AMLS & Best \\
\midrule
1l & \([5,20)\)       & 480 & 0.194 & 138.57 & 18.15 & GA-AMLS \\
1l & \([20,50)\)      & 232 & 0.097 & 96.12  & 18.60 & GA-AMLS \\
1l & \([50,200)\)     & 341 & 0.061 & 72.81  & 16.20 & GA-AMLS \\
1l & \([200,2000)\)   & 496 & 0.031 & 34.27  & 20.66 & GA-AMLS \\
1l & \([2000,\infty)\)& 499 & 0.010 & 14.43  & 13.28 & GA-AMLS \\
\midrule
2l & \([5,20)\)       & 413 & 0.194 & 108.00 & 18.49 & GA-AMLS \\
2l & \([20,50)\)      & 246 & 0.097 & 76.03  & 17.39 & GA-AMLS \\
2l & \([50,200)\)     & 311 & 0.061 & 50.55  & 20.72 & GA-AMLS \\
2l & \([200,2000)\)   & 465 & 0.031 & 36.08  & 28.36 & GA-AMLS \\
2l & \([2000,\infty)\)& 516 & 0.010 & 14.88  & 89.19 & QLD \\
\midrule
4l & \([5,20)\)       & 417 & 0.194 & 95.33  & 18.62 & GA-AMLS \\
4l & \([20,50)\)      & 240 & 0.097 & 71.55  & 16.20 & GA-AMLS \\
4l & \([50,200)\)     & 325 & 0.061 & 44.66  & 12.31 & GA-AMLS \\
4l & \([200,2000)\)   & 421 & 0.031 & 33.47  & 20.49 & GA-AMLS \\
4l & \([2000,\infty)\)& 524 & 0.010 & 9.92   & 30.71 & QLD \\
\bottomrule
\end{tabular}
\end{table}

Overall, finite-reference noise is a real limitation of the benchmark, and we avoid strong pointwise claims for the smallest-\(k\) targets; nevertheless, the filtered results support the qualitative conclusion that GA-AMLS is favorable for approximately symmetric pointwise evaluation of rare-tail events, while QLD remains competitive under scale-neutral aggregation.

\end{document}